\documentclass[twoside,11pt]{article}
\usepackage{jmlr2e}

\usepackage{algorithm}
\usepackage{algorithmic}

\usepackage{amsmath}
\newtheorem{assumption}{Assumption}

\newtheorem{observation}{Observation}

\begin{document}
	\title{On the Approximation of Cooperative Heterogeneous Multi-Agent Reinforcement Learning (MARL) using Mean Field Control (MFC)}

		\author{\name Washim Uddin Mondal \email wmondal@purdue.edu \\
		\addr Lyles School of Civil Engineering, \\
		School of Industrial Engineering,\\
		Purdue University, \\
		West Lafayette, IN, 47907, USA
		\AND
		\name Mridul Agarwal \email agarw180@purdue.edu \\
		\addr School of Electrical and Computer Engineering,\\
		 Purdue University, \\
		 West Lafayette, IN, 47907, USA
		\AND
				\name Vaneet Aggarwal \email vaneet@purdue.edu \\
		\addr School of Industrial Engineering,\\
		 School of Electrical and Computer Engineering,\\
		Purdue University, \\
		West Lafayette, IN, 47907, USA		
		\AND
		\name Satish V. Ukkusuri \email sukkusur@purdue.edu \\
		\addr Lyles School of Civil Engineering, \\
		Purdue University, \\
		West Lafayette, IN, 47907, USA
		\thanks{This work was presented in part at the NeurIPS Workshop on Cooperative AI, Dec. 2021. \\
		The current version is published in the Journal of Machine Learning Research 23(129): 1--46, 2022. }
	}

	\if 0
	\author[1,2]{Washim Uddin Mondal}
	\author[3]{Mridul Agarwal}
	\author[2,3]{Vaneet Aggarwal}
	\author[1]{Satish V. Ukkusuri}
	
	\affil[1]{Lyles School of Civil Engineering, Purdue University, West Lafayette, IN, 47907, USA}
	\affil[2]{School of Industrial Engineering, Purdue University, West Lafayette, IN, 47907, USA}
	\affil[3]{School of Electrical and Computer Engineering, Purdue University, West Lafayette, IN, 47907, USA}                 
	\fi
	\maketitle
	
	\begin{abstract}
	Mean field control (MFC) is an effective way to mitigate the curse of dimensionality of cooperative multi-agent reinforcement learning (MARL) problems. This work considers a collection of $N_{\mathrm{pop}}$  heterogeneous agents that can be segregated into $K$ classes such that the $k$-th class contains $N_k$ homogeneous agents. We  aim to prove approximation guarantees of the MARL problem for this heterogeneous system by its corresponding MFC problem.  We consider three  scenarios where the reward and transition dynamics of all agents are respectively taken to be functions of $(1)$ joint state and action distributions across all classes, $(2)$ individual distributions of each class, and $(3)$ marginal distributions of the entire population. We show that, in these cases, the $K$-class MARL problem can be approximated by MFC with errors given  as $e_1=\mathcal{O}(\frac{\sqrt{|\mathcal{X}|}+\sqrt{|\mathcal{U}|}}{N_{\mathrm{pop}}}\sum_{k}\sqrt{N_k})$, $e_2=\mathcal{O}(\left[\sqrt{|\mathcal{X}|}+\sqrt{|\mathcal{U}|}\right]\sum_{k}\frac{1}{\sqrt{N_k}})$ and $e_3=\mathcal{O}\left(\left[\sqrt{|\mathcal{X}|}+\sqrt{|\mathcal{U}|}\right]\left[\frac{A}{N_{\mathrm{pop}}}\sum_{k\in[K]}\sqrt{N_k}+\frac{B}{\sqrt{N_{\mathrm{pop}}}}\right]\right)$, respectively,  where $A, B$ are some constants and $|\mathcal{X}|,|\mathcal{U}|$ are the sizes of state and action spaces of each agent. Finally, we design a Natural Policy Gradient (NPG) based algorithm that, in the three cases stated above, can converge to an optimal MARL policy within $\mathcal{O}(e_j)$ error with a sample complexity of $\mathcal{O}(e_j^{-3})$, $j\in\{1,2,3\}$, respectively.
\end{abstract}

\begin{keywords}
multi-agent learning, heterogeneous systems, mean-field control, approximation guarantees, policy gradient algorithm
\end{keywords}

	\section{Introduction}
\label{section_intro}

The control of a large number of interacting agents is a common problem in  social science and engineering with applications in finance, smart grids, transportation, wireless networks, epidemic control, etc. \citep{schwartz2014multi,zhang2021multi}. A common approach for decision making in such environments is multi-agent reinforcement learning (MARL). In \textit{cooperative} MARL, the target is to design a sequence of \textit{decision rules} or a \textit{policy} that instructs the agents how to select \textit{actions} based on their observed \textit{state} of the environment such that the long-term \textit{collective} reward is maximized. The joint state and action spaces of the agents, however, increase exponentially with the  size of the population. 
This makes the computation of reward maximizing policy an incredibly challenging pursuit, especially  when the number of agents is large. \hspace{0cm}

To overcome the exponential blow-up of joint state and action spaces in collaborative MARL, several computationally efficient approaches have been proposed, including Independent Q-learning (IQL) \citep{tan1993multi}, centralized training with decentralized execution (CTDE) \citep{rashid2020weighted,sunehag2018value,son2019qtran,rashid2018qmix}, and mean-field control (MFC) \citep{angiuli2020unified}. IQL forces the environment to be non-stationary and thus its global convergence cannot be shown in general \citep{zhu2019q}. Global convergence for CTDE-type algorithms is also not known. On the other hand, the core idea of MFC is that, if the population size is infinite and the agents are \textit{homogeneous}, then one can draw accurate inferences about the population by studying only one representative agent \citep{bensoussan2018mean}.  The assumption of homogeneity, however, does not go hand-in-hand with many scenarios of practical interest. For example, ride-hailing services typically offer multiple types of vehicles and drivers, each with different accommodation capacity, driving behavior, searching behavior and preferred travel range. If the profit earned per unit time is considered as reward, then each type of vehicle/driver will possess a distinct reward function and thus the system as a whole cannot be homogeneous.

It is evident from the above discussion that there are no scalable approaches in the literature to solve the problem of heterogeneous MARL with global convergence guarantees. The goal of our paper is to bridge this gap. In particular, we consider a population of $N_{\mathrm{pop}}$ heterogeneous agents that can be partitioned into $K$ classes such that $k$-th class consists of $N_{k}$ homogeneous agents. In other words, the agents in each class are assumed to have identical reward function and state transition dynamics. However, those functions are different in different classes. 
In this framework, we prove that MARL can be approximated 
as a $K$-class MFC problem and obtain the approximation error as a function of different class sizes.
We further develop an algorithm to solve the $K$-class MFC problem and with the help of our approximation result, show that it efficiently converges to a provably near-optimal policy of heterogeneous MARL.

$K$-class MFC can be depicted as a generalization of traditional MFC-based approach which as stated before, assumes all agents to be identical. Homogeneity enforces the impact of the population on any agent to be summarized by the state and action distributions of the entire population. In contrast, $K$-class MFC does not allow such simplification. The agents in such a case, not only influence other agents from the same class but their influence  extends to agents from other classes as well. Due to the inter-class interaction, the influence of the whole population must be summarized either via joint state and action distributions over all classes or via the collection of distributions of each individual classes. The  analysis of a $K$-class MFC, as a result, turns out to be very different from that of  a single class/traditional MFC.

\subsection{Key Contributions:}
We analyse the above heterogeneous system under two generic setups. In the first case, the reward and transition functions of all agents are assumed to be functions of joint state and action distributions across all classes while in the second scenario, those are taken to be functions of state and action distributions of each individual classes. We prove that, in the first case, the $N_{\mathrm{pop}}$-agent RL problem can be approximated by the $K$-class MFC problem within an error of  $e_1=\mathcal{O}(\left[\sqrt{|\mathcal{X}|}+\sqrt{|\mathcal{U}|}\right]\frac{1}{N_{\mathrm{pop}}}\sum_{k\in[K]}\sqrt{N_k})$ where $N_k$ is the population size of $k$-th class, $k\in\{1,\cdots,K\}\triangleq [K]$ and $|\mathcal{X}|, |\mathcal{U}|$  denote the size of state and action spaces of individual agents, respectively. In the second case, the approximation error is proven to be $e_2=\mathcal{O}(\left[\sqrt{|\mathcal{X}|}+\sqrt{|\mathcal{U}|}\right]\sum_{k\in[K]}\frac{1}{\sqrt{N_k}})$.

For single class of agents, the approximation error reduces to $\mathcal{O}\left(\left[\sqrt{|\mathcal{X}|}+\sqrt{|\mathcal{U}|}\right]\frac{1}{\sqrt{N_{\mathrm{pop}}}}\right)$ which matches a recent result of \citep{gu_mean-field_2020}. It is worthwhile to point out that our proof methods are distinct from that used in \citep{gu_mean-field_2020}. In particular, at the heart of our approximation, lies a novel 
inequality on independent random variables bounded in $[0,1]$ with constrained 
parameters (Lemma \ref{simple_lemma} of Appendix \ref{app_th_1}). This, in conjugation with two important observations about state and action evolution of the agents, establishes our preliminary results. In contrast, \citep{gu_mean-field_2020} utilises a well-known property of sub-Gaussian variables. Although for $K=1$, both our  bound and that suggested in \citep{gu_mean-field_2020}   are of the same order, our bounds possess smaller leading constant terms\footnote{We note that the authors of \citep{gu_mean-field_2020} had an incorrect result when we first posted our version on arXiv in Sept 2021 (https://arxiv.org/pdf/2109.04024.pdf), and the error was detailed in our arXiv version. The authors of \citep{gu_mean-field_2020} fixed the error in the final version, acknowledging our manuscript. }.

 % except for a factor of $\sqrt{|\mathcal{X}||\mathcal{U}|}$. We note that there is an error in the proof of \citep{gu_mean-field_2020} as detailed in the Appendix \ref{app_gu_error}.  
We also consider a special case where the reward and transition dynamics are functions of aggregate state and action distributions of the entire population. In this case, the approximation error reduces to $e_3=\mathcal{O}\left(\left[\sqrt{|\mathcal{X}|}+\sqrt{|\mathcal{U}|}\right]\left[\frac{A}{N_{\mathrm{pop}}}\sum_{k\in[K]}\sqrt{N_k}+\frac{B}{\sqrt{N_{\mathrm{pop}}}}\right]\right)$ where $A,B$ are some constants.

Finally, extending the approach in \citep{liu2020improved}, we develop a natural policy-gradient (NPG) based algorithm for MFC, which, combined with the approximation results between MARL and MFC, shows that the proposed NPG algorithm converges to the optimal MARL policy within $\mathcal{O}(e_j)$ error with a sample complexity of $\mathcal{O}(e_j^{-3})$, $j\in\{1,2,3\}$ for the three cases, respectively.

\section{Related Work}

{\bf Approaches for RL: }  
Tabular algorithms such as Q-learning \citep{watkins1992q} and SARSA   \citep{rummery1994line} were the earliest approaches to solve RL problems. However, they are not suitable for large state-action space due to their huge memory requirement. Recently, Deep Q-network (DQN) \citep{mnih2015human} and policy-gradient based  algorithms \citep{mnih2016asynchronous} have shown promising results in terms of scalability. Although these algorithms can handle large state-space due to neural network (NN) based architecture, the approach is not scalable to multiple agents. Further, the guarantees of these algorithms either require the underlying Markov Decision Processes to be linear \citep{jin2020provably}, of low Bellman rank \citep{jiang2017contextual}, or the scaling of parameters of NNs to be increasing with time \citep{wang2019neural} - all of which are restrictive assumptions and may not hold for general MARL.

{\bf Use of MFC for MARL problems: } MFC has found its application in various MARL setups. For example, it has been used in traffic signal control \citep{wang2020large}, management of power grids \citep{chen2016state}, ride-sharing \citep{al2019deeppool},  and epidemic control \citep{watkins2016optimal}, among others.

{\bf Learning Algorithms for MFC:}  To solve homogeneous MFC problems, several learning algorithms have been proposed. For example, model-free Q-learning algorithms have been suggested in \citep{angiuli2020unified,gu_mean-field_2020,carmona2019model} while \citep{carmona2019linear}  designed a policy-gradient based method. Recently, \citep{pasztor_efficient_2021} proposed a model-based algorithm for MFC. All of these works are appropriate only for homogeneous MFC. 

{\bf Theoretical Relation between MARL and MFC:} It is well known that when the number of agents approaches infinity, the limiting behaviour of homogeneous MARL is described by MFC \citep{lacker2017limit}. However, it was proven only recently    \citep{gu_mean-field_2020} that for a finite $N_{\mathrm{pop}}$ number of agents, MARL is approximated by MFC within $\mathcal{O}(1/\sqrt{N_{\mathrm{pop}}})$ error margin. 
Our work is the first to provide such approximation bound for the heterogeneous MARL.

{\bf Mean Field Games: }  Alongside MFC, mean field games (MFG) has garnered attention in the mean-field community. MFG analyses an infinite population of \textit{non-cooperative} homogeneous  agents. The target is to identify the Nash equilibrium (NE) of the game and design learning algorithms that    converge to such an equilibrium \citep{guo2019learning,elie2020convergence,yang2018learning,agarwal2022reinforcement}.

	\section{ Model for Heterogeneous Cooperative MARL}
\label{section1}

We consider $K$ classes of agents where the agents belonging to each class are identical and interchangeable. The population size of  $k$-th class, where $k\in\{1,\cdots,K\}\triangleq [K]$ is $N_k$, while the total population size is $N_\mathrm{pop}\triangleq \sum_{k\in[K]}N_k$. Also, $\mathbf{N}\triangleq\{N_k\}_{k\in[K]}$. Let $\mathcal{X}, \mathcal{U}$ be (finite) state and action spaces of each agent. At time $t\in\{0,1,\cdots\}$, $j$-th agent belonging to $k$-th class possesses a state $x_{j,k}^{t,\mathbf{N}}\in\mathcal{X}$ and takes an action $u_{j,k}^{t,\mathbf{N}}\in\mathcal{U}$.
As a consequence, it receives a reward $r_{j,k}^{t,\mathbf{N}}$ and its state changes to $x_{j,k}^{t+1,\mathbf{N}}$ following some transition probability law. In general $r_{j,k}^{t,\mathbf{N}}$ is a function of $(x_{j,k}^{t,\mathbf{N}}, u_{j,k}^{t,\mathbf{N}})$, i.e, the state and action of the concerned agent at time $t$, as well as the joint states and actions of all the agents at time $t$ which are denoted by $\mathbf{x}_t^{\mathbf{N}}$ and $\mathbf{u}_t^{\mathbf{N}}$,  respectively. Mathematically,
\begin{align}
\label{eq_reward}
r^{t,\mathbf{N}}_{j,k}=\tilde{r}_k(x_{j,k}^{t,\mathbf{N}}, u_{j,k}^{t,\mathbf{N}},\mathbf{x}_t^{\mathbf{N}},\mathbf{u}_t^{\mathbf{N}})
\end{align}

Note that the function $\tilde{r}_k(\cdot,\cdot,\cdot,\cdot)$ is identical for all agents of $k$-th class. This is due to the fact that the agents of a certain class are homogeneous. Recall that the agents belonging to a given class are interchangeable as well. Thus if $\boldsymbol{\mu}_t^{\mathbf{N}}$, $\boldsymbol{\nu}_t^{\mathbf{N}}$ are empirical joint distributions of states and actions of all agents at time $t$, i.e, $\forall x\in\mathcal{X}$, $\forall u\in\mathcal{U}$, $ \forall k\in[K]$,
\begin{align}
\label{def_mu}
\boldsymbol{\mu}_t^{\mathbf{N}}(x,k)\triangleq\dfrac{1}{N_{\mathrm{pop}}}\sum_{j=1}^{N_k}\delta(x_{j,k}^{t,\mathbf{N}}=x),\\
\label{def_nu}
\boldsymbol{\nu}_t^{\mathbf{N}}(u,k)\triangleq\dfrac{1}{N_{\mathrm{pop}}}\sum_{j=1}^{N_k}\delta(u_{j,k}^{t,\mathbf{N}}=u)
\end{align}
where $\delta(\cdot)$ is an indicator function, then, for some function $r_k$, we can rewrite $(\ref{eq_reward})$ as 
\begin{align}
\label{eq_reward_rw}
r^{t,\mathbf{N}}_{j,k}={r}_k(x_{j,k}^{t,\mathbf{N}}, u_{j,k}^{t,\mathbf{N}},\boldsymbol{\mu}_t^{\mathbf{N}},\boldsymbol{\nu}_t^{\mathbf{N}},N_{\mathrm{pop}})
\end{align}

Note that the output of $r_k$, in general, is dependent on the total number of agents, $N_{\mathrm{pop}}$. Moreover, if, for an arbitrary set $\mathcal{A}$, the collection of all distributions over $\mathcal{A}$ is denoted as $\mathcal{P}(\mathcal{A})$, then $\boldsymbol{\mu}_t^{\mathbf{N}}\in\mathcal{P}(\mathcal{X}\times [K])$, and $\boldsymbol{\nu}_t^{\mathbf{N}}\in\mathcal{P}(\mathcal{U}\times [K])$. 

We shall now show that $(\ref{eq_reward})$ can be also written in an alternate form. Let, $\boldsymbol{\bar{\mu}}_t^{\mathbf{N}}$, $\boldsymbol{\bar{\nu}}_t^{\mathbf{N}}$ be such that $\boldsymbol{\bar{\mu}}_t^{\mathbf{N}}(.,k)$ and $\boldsymbol{\bar{\nu}}_t^{\mathbf{N}}(.,k)$  are state and action distributions of the agents of $k$-th class, i.e., $\boldsymbol{\bar{\mu}}_t^{\mathbf{N}}\in\mathcal{P}^K(\mathcal{X})\triangleq \mathcal{P}(\mathcal{X})\times \cdots \times\mathcal{P}(\mathcal{X})$,  $\mathbf{\bar{\nu}}_t^{\mathbf{N}}\in\mathcal{P}^K(\mathcal{U})$, and $\forall x\in\mathcal{X}$, $\forall u\in\mathcal{U}$, $ \forall k\in[K]$
\begin{align}
\label{eq5}
\boldsymbol{\bar{\mu}}_t^{\mathbf{N}}(x,k)\triangleq\dfrac{1}{N_k}\sum_{j=1}^{N_k}\delta(x_{j,k}^{t,\mathbf{N}}=x),\\
\label{eq6}
\boldsymbol{\bar{\nu}}_t^{\mathbf{N}}(u,k)\triangleq\dfrac{1}{N_k}\sum_{j=1}^{N_k}\delta(u_{j,k}^{t,\mathbf{N}}=u)
\end{align}

With this notation, for some $\bar{r}_k$, we can rewrite $(\ref{eq_reward})$ as
\begin{align}
\label{eq_reward_rw_2}
r^{t,\mathbf{N}}_{j,k}=\bar{r}_k(x_{j,k}^{t,\mathbf{N}}, u_{j,k}^{t,\mathbf{N}},{\boldsymbol{\bar{\mu}}}_t^{\mathbf{N}},{\boldsymbol{\bar{\nu}}}_t^{\mathbf{N}}, \mathbf{N})
\end{align}

Note that the output of $\bar{r}_k$ is, in general, dependent on $\mathbf{N}$, i.e., the population size of each of the classes. Similar to $(\ref{eq_reward})$, the state transition law in general can be written as
\begin{align}
\label{eq_transition}
x^{t+1,\mathbf{N}}_{j,k}\sim \tilde{P}_k(x_{j,k}^{t,\mathbf{N}}, u_{j,k}^{t,\mathbf{N}},\mathbf{x}_t^{\mathbf{N}},\mathbf{u}_t^{\mathbf{N}}),
\end{align}
for some function $\tilde{P}_k$. Using the same argument as used in $(\ref{eq_reward_rw})$ and $(\ref{eq_reward_rw_2})$, we can express $(\ref{eq_transition})$ in the following two equivalent forms for some functions $P_k$ and $\bar{P}_k$.
\begin{align}
\label{eq_trans_law}
\begin{split}
x^{t+1,\mathbf{N}}_{j,k}&\sim P_k(x_{j,k}^{t,\mathbf{N}}, u_{j,k}^{t,\mathbf{N}},\boldsymbol{\mu}_t^{\mathbf{N}},\boldsymbol{\nu}_t^{\mathbf{N}}, N_{\mathrm{pop}}),
\end{split}\\
\label{eq_trans_law_2}
\begin{split}
x^{t+1,\mathbf{N}}_{j,k}&\sim \bar{P}_k(x_{j,k}^{t,\mathbf{N}}, u_{j,k}^{t,\mathbf{N}},{\boldsymbol{\bar{\mu}}}_t^{\mathbf{N}},{\boldsymbol{\bar{\nu}}}_t^{\mathbf{N}}, \mathbf{N})
\end{split}
\end{align}

To proceed with the analysis, we need to assume one of the following assumptions to be true.

\begin{assumption}
	\label{ass_1}
	(a)~$\forall k\in[K]$, the outputs of $r_k,P_k$ are independent of the last argument $N_{\mathrm{pop}}$. To simplify notations, $N_{\mathrm{pop}}$ can be dropped as argument from both the functions.
	\begin{align*}
	(b) |{r}_k(x, u, \boldsymbol{\mu}_1,\boldsymbol{\nu}_1)&|\leq M_R\\	
	(c) |{r}_k(x, u,\boldsymbol{\mu}_1,\boldsymbol{\nu}_1)&-{r}_k(x, u, \boldsymbol{\mu}_2,\boldsymbol{\nu}_2)|
	\leq L_R\left[|\boldsymbol{\mu}_1-\boldsymbol{\mu}_2|_{1} + |\boldsymbol{\nu}_1-\boldsymbol{\nu}_2|_{1}\right]\\
	(d)  |{P}_k(x, u, \boldsymbol{\mu}_1,\boldsymbol{\nu}_1&)-{P}_k(x, u, \boldsymbol{\mu}_2,\boldsymbol{\nu}_2)|_{1}\leq L_P\left[|\boldsymbol{\mu}_1-\boldsymbol{\mu}_2|_{1} + |\boldsymbol{\nu}_1-\boldsymbol{\nu}_2|_{1}\right]
	\end{align*}
	$\forall x\in \mathcal{X}, \forall u\in\mathcal{U},\forall \boldsymbol{\mu}_1, \boldsymbol{\mu}_2\in\mathcal{P}(\mathcal{X}\times[K]),\forall \boldsymbol{\nu}_1, \boldsymbol{\nu}_2 \in \mathcal{P}(\mathcal{U}\times[K])$, $\forall k\in[K]$. The terms $M_R,L_R,L_P$ denote some positive constants. The function $|.|_1$ indicates $L_1$-norm.
\end{assumption}

\begin{assumption}
	\label{ass_2}
	(a) $\forall k\in[K]$, the outputs of $\bar{r}_k$, $\bar{P}_k$ are independent of the last argument $\mathbf{N}$. For simplifying notations, $\mathbf{N}$ can be dropped as argument from both the functions.	
	\begin{align*}
	&(b)|\bar{r}_k(x, u, \boldsymbol{\bar{\mu}}_1,\boldsymbol{\bar{\nu}}_1)|\leq \bar{M}_R\\	
	&(c) |\bar{r}_k(x, u, \boldsymbol{\bar{\mu}}_1,\boldsymbol{\bar{\nu}}_1)-\bar{r}_k(x, u, \boldsymbol{\bar{\mu}}_2,\boldsymbol{\bar{\nu}}_2)|\leq \bar{L}_R\left[|\boldsymbol{\bar{\mu}}_1-\boldsymbol{\bar{\mu}}_2|_{1} + |\boldsymbol{\bar{\nu}}_1-\boldsymbol{\bar{\nu}}_2|_{1}\right]\\
	&(d)  |\bar{P}_k(x, u, \boldsymbol{\bar{\mu}}_1,\boldsymbol{\bar{\nu}}_1)-\bar{P}_k(x, u, \boldsymbol{\bar{\mu}}_2,\boldsymbol{\bar{\nu}}_2)|_{1}
	\leq \bar{L}_P\left[|\boldsymbol{\bar{\mu}}_1-\boldsymbol{\bar{\mu}}_2|_{1} + |\boldsymbol{\bar{\nu}}_1-\boldsymbol{\bar{\nu}}_2|_{1}\right]
	\end{align*}	
	$\forall x\in \mathcal{X}, \forall u\in\mathcal{U},\forall \boldsymbol{\bar{\mu}}_1, \boldsymbol{\bar{\mu}}_2\in\mathcal{P}^K(\mathcal{X}),\forall \boldsymbol{\bar{\nu}}_1, \boldsymbol{\bar{\nu}}_2 \in \mathcal{P}^K(\mathcal{U})$, $\forall k\in[K]$. The terms $\bar{M}_R, \bar{L}_R$ and $\bar{L}_P$ are constants.
\end{assumption}

Assumptions $\ref{ass_1}(a)$,  $\ref{ass_2}(a)$  state that the influence of the population on individual agents is summarized by the state and action distributions only and it does not vary with the scale of the population. In particular, Assumption $\ref{ass_1}(a)$ dictates that such influence is conveyed through joint state and action distributions across all classes which makes the reward and transition functions invariant to $N_{\mathrm{pop}}$. In contrast, Assumption $\ref{ass_2}(a)$ presumes that the joint influence of the whole population can be segregated based on the class it originated from. This makes the reward and  transition law invariant to the  population size of each individual class. For single class of agents (i.e., when $K=1$), both assumptions are identical. Scale invariance is one of the fundamental assumptions in the mean-field literature \citep{carmona2018probabilistic,gu_mean-field_2020,angiuli2020unified}.

Assumptions $\ref{ass_1}(b)$, $\ref{ass_2}(b)$ state that the reward functions are bounded while Assumptions $\ref{ass_1}(c)$, $\ref{ass_2}(c)$ and $\ref{ass_1}(d)$, $\ref{ass_2}(d)$ dictate that the reward functions and the transition probabilities are Lipschitz continuous w. r. t. their respective state and action distribution arguments. These assumptions are common in the literature \citep{carmona2018probabilistic,gu_mean-field_2020,angiuli2020unified}.

It is worthwhile to mention that for given $r_k$'s and $P_k$'s satisfying Assumption \ref{ass_1}, one can define equivalent $\bar{r}_k$'s and $\bar{P}_k$'s that satisfy Assumption \ref{ass_2} and vice versa. For example, in appendix \ref{app_loose_bounds}, we exhibit that if $r_k$'s and $P_k$'s satisfy Assumption 1 with Lipschitz constants $L_R, L_P$ respectively, then we can define equivalent $\bar{r}_k$'s and $\bar{P}_k$'s that satisfy Assumption \ref{ass_2} with constants $L_P\boldsymbol{\theta}_M^{-1}, L_Q\boldsymbol{\theta}_M^{-1}$ respectively where $\boldsymbol{\theta}_M^{-1}\triangleq \max_{k\in[K]}\{N_{\mathrm{pop}}/N_k\}$. Note that the modified `constants' are dependent on the population sizes of different classes. Therefore, if we have an approximation bound for Assumption 2, by injecting the values of the modified constants into the expression of that bound, we can obtain a bound for Assumption \ref{ass_1}. In appendix \ref{app_loose_bounds}, however, we demonstrate that such translated bounds are, in general, loose. This is primarily because, in the derivation of the bound for Assumption \ref{ass_2}, the Lipschitz constants are not treated as functions of the population sizes. Therefore, it cannot account for any stringent inequality that might be applicable due to the special structure of the translated functions. We can similarly argue why a translation from Assumption \ref{ass_2} to Assumption \ref{ass_1} may not produce a tight result. In summary, although the approximation result derived for one of the above assumptions can be cast, with slight modifications, as an approximation result for the other assumption, in general, such translated results are loose. To derive tighter bounds, it is therefore necessary to produce analysis for each of these assumptions separately. We shall establish our approximation result first with Assumption \ref{ass_1} and then with Assumption \ref{ass_2}.

\section{Policy, Value Function and Mean Field Limit under Assumption \ref{ass_1}}

\subsection{Policy and Value Function}

Recall that the distributions $\boldsymbol{\mu}_t^{\mathbf{N}}$ and $\boldsymbol{\nu}_t^{\mathbf{N}}$ defined by $(\ref{def_mu}), (\ref{def_nu})$ are elements of $\mathcal{P}(\mathcal{X}\times[K])$ and $\mathcal{P}(\mathcal{U}\times[K])$ respectively. Therefore, presuming Assumption \ref{ass_1} to be true, the reward function $r_k$ for $k$-th class of agents can be described as a map of the following form, $r_k: \mathcal{X}\times\mathcal{U}\times \mathcal{P}(\mathcal{X}\times[K])\times\mathcal{P}(\mathcal{U}\times[K])\rightarrow \mathbb{R}$. Similarly, the transition law $P_k$ can be described as, $P_k:\mathcal{X}\times\mathcal{U}\times \mathcal{P}(\mathcal{X}\times[K])\times\mathcal{P}(\mathcal{U}\times[K])\rightarrow\mathcal{P}(\mathcal{X})$.

A time-dependent decision rule $\pi_k^t$ for $k$-th class of agents is a map, $\pi_k^t:\mathcal{X}\times\mathcal{P}(\mathcal{X}\times[K])\rightarrow \mathcal{P}(\mathcal{U})$. In simple words, a decision rule $\pi_k^t$ states with what probability a certain action $u\in\mathcal{U}$ should be selected by any agent of $k$-th class at time $t$, given its own state and the state distribution across all classes at time $t$. A policy $\boldsymbol{\pi}\triangleq\{(\pi_k^t)_{k\in[K]}\}_{t\in\{0,1,\cdots\}}$ is defined as a sequence of decision rules over all classes of agents. For a policy $\boldsymbol{\pi}$ and given initial states $\mathbf{x}_0^{\mathbf{N}}$, the infinite-horizon $\gamma\in[0,1)$-discounted value of the policy $\boldsymbol{\pi}$ for $j$-th agent of $k$-th class is defined as
\begin{align}
\label{eq_10}
v_{j,k}^{\mathbf{N}}(\mathbf{x}_0^{\mathbf{N}},\boldsymbol{\pi}) =\mathbb{E}\left[ \sum_{t=0}^{\infty}  \gamma^t r_k(x_{j,k}^{t,\mathbf{N}},u_{j,k}^{t,\mathbf{N}}, \boldsymbol{\mu}_t^{\mathbf{N}}, \boldsymbol{\nu}_t^{\mathbf{N}}) \right],
\end{align}
where the expectation is taken over $u_{j,k}^{t,\mathbf{N}}\sim \pi^{t}_{k}(x_{j,k}^{t,\mathbf{N}},\boldsymbol{\mu}_t^{\mathbf{N}})$,  $x_{j,k}^{t+1,\mathbf{N}}\sim P_k(x_{j,k}^{t,\mathbf{N}},u_{j,k}^{t,\mathbf{N}}, \boldsymbol{\mu}_t^{\mathbf{N}}, \boldsymbol{\nu}_t^{\mathbf{N}})$.
Also, $\boldsymbol{\mu}_t^{\mathbf{N}}$, and $\boldsymbol{\nu}_t^{\mathbf{N}}$ are obtained from $\mathbf{x}_t^{\mathbf{N}}$ and $\mathbf{u}_t^{\mathbf{N}}$ respectively. The average infinite-horizon discounted value of policy $\boldsymbol{\pi}$ is defined as
\begin{align}
\label{eq_vN}
\begin{split}
&v^{\mathbf{N}}(\mathbf{x}_0^{\mathbf{N}},\boldsymbol{\pi}) \triangleq\dfrac{1}{N_{\mathrm{pop}}}\sum_{k\in[K]}\sum_{j=1}^{N_k} v_{j,k}^{\mathbf{N}}(\mathbf{x}_0^{\mathbf{N}},\boldsymbol{\pi})
\end{split}
\end{align}

In the next subsection, we discuss how to compute the mean-field limit of the empirical value  function $v^{\mathbf{N}}$. The following two observations will be useful in many of our forthcoming analyses. 
\begin{observation}
	\label{obs1}
	$\{u_{j,k}^{t,\mathbf{N}}\}_{j\in[N_k],k\in[K]}$ are independent conditioned on $\mathbf{x}_t^{\mathbf{N}}$,  $\forall t\in\{0,1,\cdots\}$. Specifically, for a given policy $\boldsymbol{\pi}$, and $\forall j\in[N_k], \forall j'\in[N_{k'}]$, $\forall k,k'\in[K]$, 
	\begin{align*}
	\mathbb{P}(u_{j,k}^{t,\mathbf{N}},u_{j',k'}^{t,\mathbf{N}}|\mathbf{x}_t^{\mathbf{N}})=\mathbb{P}(u_{j,k}^{t,\mathbf{N}}|\mathbf{x}_t^{\mathbf{N}})\mathbb{P}(u_{j',k'}^{t,\mathbf{N}}|\mathbf{x}_t^{\mathbf{N}})
	\end{align*}
\end{observation}

\begin{observation}
	$\{x_{j,k}^{t+1,\mathbf{N}}\}_{j\in[N_k],k\in[K]}$ are independent conditioned on $\mathbf{x}_t^{\mathbf{N}},\mathbf{u}_t^{\mathbf{N}}$, $\forall t\in\{0,1,\cdots\}$.
	\label{obs2}
\end{observation}

\subsection{Mean Field Limit for $K$ Classes}
\label{section_mfc1}

In the mean-field limit, i.e., when $N_k\rightarrow\infty$, $\forall k\in[K]$, it is enough to consider a representative for each of the classes. The state and action of the representative of $k$-th class at time $t$ are indicated as $x_k^t$$\in\mathcal{X}$ and $u_k^t\in\mathcal{U}$ respectively. The joint distribution of states and actions of all classes of agents are symbolized as $\boldsymbol{\mu}_t\in\mathcal{P}(\mathcal{X}\times[K])$ and $\boldsymbol{\nu}_t\in\mathcal{P}(\mathcal{U}\times[K])$.
If Assumption \ref{ass_1} holds, then the reward and the transition probability law of the representative of $k$-th class at time $t$ can be expressed as, $r_k(x_k^t,u_k^t,\boldsymbol{\mu}_t,\boldsymbol{\nu}_t)$ and $P_k(x_k^t,u_k^t,\boldsymbol{\mu}_t,\boldsymbol{\nu}_t)$ respectively. For a given policy, $\boldsymbol{\pi}\triangleq\{\boldsymbol{\pi}_t\}_{t\in\{0,1,\dots\}}$, $\boldsymbol{\pi}_t \triangleq\{(\pi_k^t)_{k\in[K]}\}$, where $\{\pi_k^t\}_{t\in\{0,1,\dots\}}$ is a sequence of decision rules for $k$-th class, the action distribution at time $t$ can be obtained as follows.
\begin{align}
\begin{split}
\label{eq_nu_v}
\boldsymbol{\nu}_t=\nu^{\mathrm{MF}}(\boldsymbol{\mu}_t, \boldsymbol{\pi}_t)&\triangleq\{\nu_k^{\mathrm{MF}}(\boldsymbol{\mu}_t, \boldsymbol{\pi}_t)\}_{k\in[K]},\\
\nu^{\mathrm{MF}}_k(\boldsymbol{\mu}_t, \boldsymbol{\pi}_t) &\triangleq \sum_{x\in\mathcal{X}} \pi_k^t(x,\boldsymbol{\mu}_t)\boldsymbol{\mu}_t(x, k)
\end{split}
\end{align}

Using the definition of $\nu^{\mathrm{MF}}$, the evolution of the state distribution can be written as
\begin{align}
\begin{split}
\boldsymbol{\mu}_{t+1}=&P^{\mathrm{MF}}(\boldsymbol{\mu}_t, \boldsymbol{\pi}_t)\triangleq\{P_k^{\mathrm{MF}}(\boldsymbol{\mu}_t, \boldsymbol{\pi}_t)\}_{k\in [K]},\\
&P_k^{\mathrm{MF}}(\boldsymbol{\mu}_t, \boldsymbol{\pi}_t) \triangleq \sum_{x\in\mathcal{X}}\sum_{u\in\mathcal{U}}\boldsymbol{\mu}_t(x,k)\pi_k^t(x,\boldsymbol{\mu}_t)(u) 
\times P_k\left(x,u,\boldsymbol{\mu}_t,\nu^{\mathrm{MF}}(\boldsymbol{\mu}_t,\boldsymbol{\pi}_t)\right)
\label{eq_phi}
\end{split}
\end{align}

Finally, the average reward of $k$-th class is computed as
\begin{align}
\begin{split}
r_k^{\mathrm{MF}}(\boldsymbol{\mu}_t, \boldsymbol{\pi}_t) \triangleq \sum_{x\in\mathcal{X}}\sum_{u\in\mathcal{U}}\boldsymbol{\mu}_t(x,k)\pi_k^t(x,\boldsymbol{\mu}_t)(u)\times r_k(x,u,\boldsymbol{\mu}_t,\nu^{\mathrm{MF}}(\boldsymbol{\mu}_t,\boldsymbol{\pi}_t))
\label{eqr}
\end{split}
\end{align}

For a given initial state distribution $\boldsymbol{\mu}_0$, and a policy $\boldsymbol{\pi}$, the infinite-horizon $\gamma$-discounted average reward is
\begin{align}
\begin{split}
v^{\mathrm{MF}}(\boldsymbol{\mu}_0,\boldsymbol{\pi})= \sum_{k\in[K]}\sum_{t=0}^{\infty}\gamma^t r_k^{\mathrm{MF}}(\boldsymbol{\mu}_t, \boldsymbol{\pi}_t)
\label{eqvr}
\end{split}
\end{align}

In the following section, we show how well the function $v^{\mathbf{N}}$, given by $(\ref{eq_vN})$  can be approximated by $v^{\mathrm{MF}}$ as the population sizes, $\mathbf{N}$ and the cardinality of state and action spaces, indicated by $|\mathcal{X}|$ and $|\mathcal{U}|$ respectively, become large.

\section{MFC as an Approximation of MARL with Assumption \ref{ass_1}}

To establish the approximation result, we need to restrict the policies to a set $\Pi$ such that the following assumption holds.

\begin{assumption}
	\label{assumptions}
	Every policy $ \boldsymbol{\pi}\triangleq\{(\pi_k^t)_{k\in[K]}\}_{t\in\{0,1,\cdots\}}$ in $\Pi$ is such that, $\forall x\in \mathcal{X}$,$\forall \boldsymbol{\mu}_1, \boldsymbol{\mu}_2\in\mathcal{P}(\mathcal{X}\times[K])$, $\forall k\in[K]$	
	\begin{align*}
	&  |\pi_k^t(x, \boldsymbol{\mu}_1)-\pi_k^t(x, \boldsymbol{\mu}_2)|_{1}\leq L_{Q}|\boldsymbol{\mu}_1-\boldsymbol{\mu}_2|_{1}
	\end{align*}
	for some positive real $L_Q$. 
\end{assumption}

Assumption \ref{assumptions} states that the decision rules $\pi_k^t$, associated with any policy in $\Pi$ are Lipschitz continuous w. r. t. the state distribution argument. Such assumption holds in practice because the decision rules are commonly realised by Neural Networks possessing bounded weights \citep{pasztor_efficient_2021}. Below we state our first result.

\begin{theorem}
	\label{thm_1}
	Let $\mathbf{x}_0^{\mathbf{N}}$ be the initial states and $\boldsymbol{\mu}_0$ be their corresponding distribution. If $v^{\mathbf{N}}$ is the empirical value function given by $(\ref{eq_vN})$ and $v^{\mathrm{MF}}$ is its mean-field limit defined in $(\ref{eqvr})$, then for any policy, $\boldsymbol{\pi}\in\Pi$,  the following inequality holds if $\gamma S_P <1$ and Assumptions \ref{ass_1}, \ref{assumptions} are true. 
	\begin{align}
	\label{eq_28_new}
	\begin{split}
	\Big|v^{\mathbf{N}}(\mathbf{x}_0^{\mathbf{N}},{\boldsymbol{\pi}})&-v^{\mathrm{MF}}(\boldsymbol{\mu}_0,\boldsymbol{\pi})\Big|
	\leq
	\dfrac{C_R}{1-\gamma}\sqrt{|\mathcal{U}|} \dfrac{1}{N_{\mathrm{pop}}}\left(\sum_{k\in[K]}\sqrt{N_k}\right) \\&+ C_P\left(\dfrac{S_R}{S_P-1}\right)\left[\sqrt{|\mathcal{X}|}+\sqrt{|\mathcal{U}|}\right] \dfrac{1}{N_{\mathrm{pop}}}\left(\sum_{k\in[K]}\sqrt{N_k}\right)\times\left[\dfrac{1}{1-\gamma S_P}-\dfrac{1}{1-\gamma}\right]
	\end{split}
	\end{align}
	where $S_R\triangleq M_R(1+L_Q)+L_R(2+L_Q)$, $S_P\triangleq (1+L_Q)+L_P(2+L_Q)$, $C_R\triangleq M_R+L_R$, $C_P\triangleq 2+L_P$.
\end{theorem}

Theorem \ref{thm_1} dictates that the empirical value function, $v^{\mathbf{N}}$, can be approximated by its mean-field limit, $v^{\mathrm{MF}}$,  within an error margin of $\mathcal{O}\left(\frac{1}{N_{\mathrm{pop}}}\sum_{k\in[K]} \sqrt{N_k}\right)$. In a special case, where the number of agents in each classes are equal, the error is $\mathcal{O}(\sqrt{K/N_{\mathrm{pop}}})$. Additionally, Theorem \ref{thm_1} also dictates how the error varies as a function of the   state-action cardinality. For example, given other things as constant, the error is $\mathcal{O}\left(\sqrt{|\mathcal{X}|}+\sqrt{|\mathcal{U}|}\right)$.

The implication of this result is profound. It essentially assures that, if we can come up with an algorithm to compute the optimal MFC policy, then the obtained policy is guaranteed to be close to the optimal MARL policy. In practice, an MFC problem is much easier to solve than a MARL problem, primarily because in MFC, we are needed to keep track of only  one representative agent from each class. Therefore, if the number of agents is large and individual state-action spaces are relatively small, MFC can be utilized as an easier route to obtain an approximate MARL solution. However, Theorem \ref{thm_1} also suggests that the error of approximation increases with the number of classes, $K$. As a consequence, if the level of heterogeneity in the population is too high, then MFC may not be a good approximation of MARL. 

\subsection{Proof Outline} A detailed proof of Theorem \ref{thm_1} is provided in Appendix \ref{app_th_1}. Here we present a brief outline.

\textit{Step 0:} 
$v^{\mathbf{N}}$ and  $v^{\mathrm{MF}}$ respectively are time-discounted average rewards for a finite agent system and that of an infinite agent system. To estimate their difference, we need to evaluate the difference between mean rewards of these systems at a given time $t$.

\textit{Step 1:} To achieve that, we introduce an intermediate system $X$ whose state-action evolutions are identical to the $\mathbf{N}$-agent system upto time $t$, but after that, it follows the update process of an infinite agent system. Our first task is to bound the difference between the average reward of system $X$ and that of the $\mathbf{N}$-agent system at time $t$ (Lemma \ref{lemma_r}).

\textit{Step 2:} The next task is to estimate the difference between the average reward of system $X$ and the mean-field reward at time $t$. Using the continuity of mean-field reward function (Lemma \ref{lemma2}), this difference can be bounded by a multiple of the difference between the empirical state distribution, $\boldsymbol{\mu}_t^{\mathbf{N}}$ of the $\mathbf{N}$-agent system and the mean-field distribution, $\boldsymbol{\mu}_t$.

\textit{Step 3:} The difference between $\boldsymbol{\mu}_t^{\mathbf{N}}$ and $\boldsymbol{\mu}_t$ can be obtained in a recursive manner. To achieve this, we introduce another intermediate system $Y$ whose state, action distributions upto time $t-1$ are same as the $\mathbf{N}$-agent system, but after that, those evolve following mean-field updates. First, we evaluate the difference between $\boldsymbol{\mu}_t^{\mathbf{N}}$ and the state distribution of the system $Y$ at time $t$ (Lemma \ref{lemma_mu}). %{\bf What is System $Y$?}

\textit{Step 4:} Using the continuity of mean-field state-transition function (Lemma \ref{lemma3}), the difference between $\boldsymbol{\mu}_t$ and the state distribution of system $Y$ at $t$ is upper bounded by a multiple of the difference between $\boldsymbol{\mu}_{t-1}^{\mathbf{N}}$ and $\boldsymbol{\mu}_{t-1}$.

\textit{Step 5:} Combining the above results, the difference between the average finite-agent reward and the mean-field reward at time $t$ can be written as a function of $t$. 

\textit{Step 6:}  Taking a $\gamma$-discounted sum of the these estimate errors over $t$, we arrive at the desired result.

\section{MFC as an Approximation of MARL  with Assumption \ref{ass_2}}
\label{sec_ass_2}

We shall now discuss how well the empirical value function is approximated by its mean-field counterpart if Assumption \ref{ass_2} is true. The empirical state and action distributions of $k$-th class at time $t$ are denoted as $\boldsymbol{\bar{\mu}}_t^{\mathbf{N}}(\cdot,k)$, $\boldsymbol{\bar{\nu}}_t^{\mathbf{N}}(\cdot,k)$ and defined by $(\ref{eq5}),(\ref{eq6})$ respectively. Clearly, $\boldsymbol{\bar{\mu}}_t^{\boldsymbol{N}}\in\mathcal{P}^K(\mathcal{X})$ and $\boldsymbol{\bar{\nu}}_t^{\mathbf{N}}\in\mathcal{P}^K(\mathcal{U})$ where $\mathcal{P}^K(\cdot)\triangleq \mathcal{P}(\cdot)\times\cdots\times\mathcal{P}(\cdot)$.

The reward function, $\bar{r}_k$ and the transition probability law, $\bar{P}_k$ of $k$-th class of agents are defined to be functions of the following forms, $\bar{r}_k:\mathcal{X}\times\mathcal{U}\times\mathcal{P}^K(\mathcal{X})\times \mathcal{P}^K(\mathcal{U})\rightarrow \mathbb{R}$ and $\bar{P}_k:\mathcal{X}\times\mathcal{U}\times\mathcal{P}^K(\mathcal{X})\times \mathcal{P}^K(\mathcal{U})\rightarrow \mathcal{P}(\mathcal{X})$. Similarly, a policy $\boldsymbol{\bar{\pi}}\triangleq \{(\bar{\pi}_k^t)_{k\in[K]}\}_{t\in\{0,1,\cdots\}}$ is defined as a sequence of collection of decision rules, $\bar{\pi}_k^t$ where $\bar{\pi}_k^t:\mathcal{X}\times \mathcal{P}^K(\mathcal{X})\rightarrow \mathcal{P}(\mathcal{U})$. Similar to Assumption \ref{assumptions}, we restrict the policies to a set $\bar{\Pi}$ such that the decision rules associated with each elements of $\bar{\Pi}$ are Lipschitz continuous. This is formally expressed as follows.
\begin{assumption}
	\label{assumptions_2}
	Every policy $ \boldsymbol{\bar{\pi}}\triangleq\{(\bar{\pi}_k^t)_{k\in[K]}\}_{t\in\{0,1,\cdots\}}$ in $\bar{\Pi}$ is such that, $\forall x\in \mathcal{X}$,$\forall \boldsymbol{\bar{\mu}}_1, \boldsymbol{\bar{\mu}}_2\in\mathcal{P}^K(\mathcal{X})$, $\forall k\in[K]$	
	\begin{align*}
	&  |\bar{\pi}_k^t(x, \boldsymbol{\bar{\mu}}_1)-\bar{\pi}_k^t(x, \boldsymbol{\bar{\mu}}_2)|_{1}\leq \bar{L}_{Q}|\boldsymbol{\bar{\mu}}_1-\boldsymbol{\bar{\mu}}_2|_{1}
	\end{align*}
	for some positive real $\bar{L}_Q$.
	\label{ass_4}
\end{assumption}

For initial states $\mathbf{x}_0^{\mathbf{N}}$, the empirical value of a given policy $\boldsymbol{\bar{\pi}}$  is defined as follows.
\begin{align}
\label{eq_18}
\begin{split}
\bar{v}^{\mathbf{N}}(\mathbf{x}_0^{\mathbf{N}}, \boldsymbol{\bar{\pi}})=\dfrac{1}{N_{\mathrm{pop}}}\sum_{k\in[K]}\sum_{j=1}^{N_k}\mathbb{E}\left[ \sum_{t=0}^{\infty}  \gamma^t \bar{r}_k(x_{j,k}^{t,\mathbf{N}},u_{j,k}^{t,\mathbf{N}}, \boldsymbol{\bar{\mu}}_t^{\mathbf{N}}, \boldsymbol{\bar{\nu}}_t^{\mathbf{N}}) \right]
\end{split}
\end{align}
where the expectation is taken over $u_{j,k}^{t,\mathbf{N}}\sim \bar{\pi}^{t}_{k}(x_{j,k}^{t,\mathbf{N}},\boldsymbol{\bar{\mu}}_t^{\mathbf{N}})$,  $x_{j,k}^{t+1,\mathbf{N}}\sim \bar{P}_k(x_{j,k}^{t,\mathbf{N}},u_{j,k}^{t,\mathbf{N}}, \boldsymbol{\bar{\mu}}_t^{\mathbf{N}}, \boldsymbol{\bar{\nu}}_t^{\mathbf{N}})$. Also, $\boldsymbol{\bar{\mu}}_t^{\mathbf{N}}$, $\boldsymbol{\bar{\nu}}_t^{\mathbf{N}}$ are obtained from $\mathbf{x}_t^{\mathbf{N}}$, $\mathbf{u}_t^{\mathbf{N}}$. If $\bar{v}^{\mathrm{MF}}$ denotes the mean-field limit of $\bar{v}^{\mathbf{N}}$, then the following approximation result holds.

\begin{theorem}
	\label{thm_2}
	If $\mathbf{x}_0^{\mathbf{N}}$ are initial states and $\boldsymbol{\bar{\mu}}_0\in\mathcal{P}^K(\mathcal{X})$ is the resulting distribution,  then under Assumptions \ref{ass_2}, \ref{ass_4}, $\forall \boldsymbol{\bar{\pi}}\in \bar{\Pi}$,
	\begin{align}
	\label{eq_within_thm2}
	\begin{split}
	\Big|\bar{v}^{\mathbf{N}}(\mathbf{x}_0^{\mathbf{N}},{\boldsymbol{\bar{\pi}}})&-\bar{v}^{\mathrm{MF}}(\boldsymbol{\bar{\mu}}_0,\boldsymbol{\bar{\pi}})\Big|
	\leq
	\dfrac{\bar{C}_R}{1-\gamma}\sqrt{|\mathcal{U}|} \left(\sum_{k\in[K]}\dfrac{1}{\sqrt{N_k}}\right) \\+ &\bar{C}_P\left(\dfrac{\bar{S}_R}{\bar{S}_P-1}\right)\left[\sqrt{|\mathcal{X}|}+\sqrt{|\mathcal{U}|}\right] \left(\sum_{k\in[K]}\dfrac{1}{\sqrt{N_k}}\right)\times\left[\dfrac{1}{1-\gamma \bar{S}_P}-\dfrac{1}{1-\gamma}\right]
	\end{split}
	\end{align}
	whenever $\gamma \bar{S}_P<1$ where $\bar{v}^{\mathbf{N}}(\cdot,\cdot)$ denotes the empirical value function defined in $(\ref{eq_18})$ and $\bar{v}^{\mathrm{MF}}(\cdot,\cdot)$ is its mean-field limit. The other terms are given as follows: $\bar{C}_R\triangleq \bar{M}_R+\bar{L}_R$, $\bar{C}_P\triangleq 2+K\bar{L}_P$, $\bar{S}_R\triangleq \bar{M}_R(1+\bar{L}_Q)+\bar{L}_R(2+K\bar{L}_Q)$, and $\bar{S}_P\triangleq (1+K\bar{L}_Q)+K\bar{L}_P(2+K\bar{L}_Q)$.
\end{theorem}

Therefore, Theorem $\ref{thm_2}$ asserts that the error in approximating the value function $\bar{v}^{\mathbf{N}}$ by its mean-field limit, $\bar{v}^{\mathrm{MF}}$, is $\mathcal{O}(\left[\sqrt{|\mathcal{X}|}+\sqrt{|\mathcal{U}|}\right]\sum_{k\in[K]}\frac{1}{\sqrt{N_k}})$.  Note that, Theorem \ref{thm_1} gives a tighter bound than Theorem \ref{thm_2} (if Lipschitz constants are same in both the cases). This can be attributed to the fact that the difference between two joint distributions $\boldsymbol{\mu},\boldsymbol{\mu}'$ (which is used to bound the approximation error in Theorem \ref{thm_1}) is, in general, less than the difference between the resulting collection of distributions of all classes $\boldsymbol{\bar{\mu}},\boldsymbol{\bar{\mu}}'$ (which is used to bound the approximation error in Theorem \ref{thm_2}).

\section{Improved Results when Transition and Reward Functions Depend on Aggregate Distributions}

In this section, the transition and reward functions (and thus, the decision rules associated with the policies) are assumed to be Lipschitz continuous functions of aggregate/marginal state and action distributions of the entire population.
It is easy to see that, this assumption is stronger  than Assumption \ref{ass_1} and \ref{ass_2} as any Lipschitz continuous function of the marginal distributions is necessarily a Lipschitz continuous function of both the joint distributions and the collection of distributions of each classes, with the same Lipschitz parameter.
We shall demonstrate that such stronger assumption leads to improved approximation result.
Mathematically, if the reward and 
state transition functions are indicated as $r_k$'s, and $P_k$'s,
a generic policy is denoted as $\boldsymbol{{\pi}}\triangleq\{(\pi_k^t)_{k\in[K]}\}_{t\in\{0,1,\cdots\}}$, and the class of policies is defined by $\Pi$, then our assumption can be stated as follows.
\begin{assumption}
	\label{ass_5}
	(a) The reward functions, transition dynamics and the decision rules are of the following form.
	\begin{align*}
	&{r}_k:\mathcal{X}\times \mathcal{U}\times \mathcal{P}(\mathcal{X})\times\mathcal{P}(\mathcal{U})\rightarrow \mathbb{R}\\
	&{P}_k:\mathcal{X}\times \mathcal{U}\times \mathcal{P}(\mathcal{X})\times\mathcal{P}(\mathcal{U})\rightarrow \mathcal{P}(\mathcal{X})\\
	&\pi_k^t:\mathcal{X}\times \mathcal{P}(\mathcal{X})\rightarrow \mathcal{P}(\mathcal{U})
	\end{align*}
	
	$\forall \boldsymbol{\mu}_1, \boldsymbol{\mu}_2\in\mathcal{P}(\mathcal{X}\times[K])$, $\forall \boldsymbol{\nu}_1, \boldsymbol{\nu}_2 \in \mathcal{P}(\mathcal{U}\times[K])$, $\forall x\in \mathcal{X}$, $\forall u\in\mathcal{U}$, $\forall k\in[K]$, $\forall \boldsymbol{\pi}=\{(\pi_k^t)_{k\in[K]}\}_{t\in\{0,1,\cdots\}}\in \Pi$, 
	\begin{align*}
	(b) &|{r}_k(x, u, \boldsymbol{\mu}_1[\mathcal{X}],\boldsymbol{\nu}_1[\mathcal{U}])|\leq M_R\\	
	(c) &|{r}_k(x, u,\boldsymbol{\mu}_1[\mathcal{X}],\boldsymbol{\nu}_1[\mathcal{U}])-{r}_k(x, u, \boldsymbol{\mu}_2[\mathcal{X}],\boldsymbol{\nu}_2[\mathcal{U}])|
	\leq L_R\left[|\boldsymbol{\mu}_1[\mathcal{X}]-\boldsymbol{\mu}_2[\mathcal{X}]|_{1} + |\boldsymbol{\nu}_1[\mathcal{U}]-\boldsymbol{\nu}_2[\mathcal{U}]|_{1}\right]\\
	(d)  &|{P}_k(x, u, \boldsymbol{\mu}_1[\mathcal{X}],\boldsymbol{\nu}_1[\mathcal{U}])-{P}_k(x, u, \boldsymbol{\mu}_2[\mathcal{X}],\boldsymbol{\nu}_2[\mathcal{U}])|_{1}
	\leq L_P\left[|\boldsymbol{\mu}_1[\mathcal{X}]-\boldsymbol{\mu}_2[\mathcal{X}]|_{1} + |\boldsymbol{\nu}_1[\mathcal{U}]-\boldsymbol{\nu}_2[\mathcal{U}]|_{1}\right]\\
	(e)&
	|\pi_k^t(x, \boldsymbol{\mu}_1[\mathcal{X}])-\pi_k^t(x, \boldsymbol{\mu}_2[\mathcal{X}])|_{1}\leq L_{Q}|\boldsymbol{\mu}_1[\mathcal{X}]-\boldsymbol{\mu}_2[\mathcal{X}]|_{1}
	\end{align*}
	where $\boldsymbol{\mu}[\mathcal{X}]$, $\boldsymbol{\nu}[\mathcal{U}]$ are  marginal distributions on $\mathcal{X}$, $\mathcal{U}$ resulting from $\boldsymbol{\mu}$, $\boldsymbol{\nu}$  and $M_R,L_R,L_P,L_Q$ are some constants.
\end{assumption}

\begin{theorem}
	\label{thm_3}
	Assume $\mathbf{x}_0^{\mathbf{N}}$ to be the initial states and $\boldsymbol{\mu}_0$ their corresponding joint distribution. If $\gamma S_P <1$, and Assumption \ref{ass_5} holds, then for any arbitrary policy, $\boldsymbol{\pi}\in\Pi$, 
	\begin{align}
	\begin{split}
	&\Big|v^{\mathbf{N}}(\mathbf{x}_0^{\mathbf{N}},{\boldsymbol{\pi}})-v^{\mathrm{MF}}(\boldsymbol{\mu}_0,\boldsymbol{\pi})\Big|\leq
	\dfrac{C_R}{1-\gamma} \sqrt{|\mathcal{U}|}\dfrac{1}{\sqrt{N_{\mathrm{pop}}}} \\
	& + \left[\sqrt{|\mathcal{X}|}+\sqrt{|\mathcal{U}|}\right]\left(\dfrac{\gamma C_P}{1-\gamma}\right) \left[\dfrac{S_R'}{N_{\mathrm{pop}}}\left(\sum_{k\in[K]}\sqrt{N_k}\right)+\dfrac{S_R''}{\sqrt{N_{\mathrm{pop}}}}\right]\\
	&+ C_P\left(\dfrac{S_R}{S_P-1}\right)\left[\sqrt{|\mathcal{X}|}+\sqrt{|\mathcal{U}|}\right] \left(\dfrac{\gamma}{1-\gamma S_P}-\dfrac{\gamma}{1-\gamma}\right)
	\times \left[\dfrac{S_P'}{N_{\mathrm{pop}}}\left(\sum_{k\in[K]}\sqrt{N_k}\right)+\dfrac{S_P''}{\sqrt{N_{\mathrm{pop}}}}\right] 
		\end{split}
	 \label{eq_36}
	\end{align}
	where $v^{\mathbf{N}}$ denotes the empirical value function  and $v^{\mathrm{MF}}$ is its mean-field limit. Also, $S_R'\triangleq M_R+L_R$, $S_R''\triangleq M_RL_Q+L_R(1+L_Q)$, $S_P'\triangleq 1+L_P$, and $S_P''\triangleq L_Q+L_P(1+L_Q)$. The terms $S_R, S_P, C_R, C_P$ are defined in Theorem \ref{thm_1}.
\end{theorem}

Theorem \ref{thm_3} states that the error in approximating the empirical value function, $v^{\mathbf{N}}$, by its mean-field limit, $v^{\mathrm{MF}}$, can be written as $\mathcal{O}\left(\left[\sqrt{|\mathcal{X}|}+\sqrt{|\mathcal{U}|}\right]\left[\frac{A}{N_{\mathrm{pop}}}\sum_{k=1}^K\sqrt{N_k}+\frac{B}{\sqrt{N_{\mathrm{pop}}}}\right]\right)$ where $A, B$ are some constants. It is easy to show that the approximation error suggested by Theorem \ref{thm_3} is strictly better than the errors given by Theorem \ref{thm_1}, \ref{thm_2}. Intuitively, if the reward and transition functions only depend on the marginal distributions, not on the joint distributions, then those functions overlook the heterogeneity of the agents and treat the whole population holistically. This leads to the $\frac{1}{\sqrt{N_{\mathrm{pop}}}}$ component of the error which matches the error of a single class system. However, the reward and transition functions (and hence, the decision rules) themselves are different for different classes. This variation enforces the other part of the error to align towards a general heterogeneous system.

\section{Global Convergence of MARL using Natural Policy Gradient Algorithm}

The previous sections showed that a $K$-class heterogeneous MARL can be approximated as a $K$-class MFC. This section develops a Natural Policy Gradient (NPG) based algorithm for $K$-class MFC that can obtain policies with guaranteed optimality gaps for heterogeneous MARL. We limit our discussion to the category of systems that satisfy the same set of assumptions as used in Theorem \ref{thm_1}. For other assumptions, one can
replicate similar result and the processes have been briefly described in sections \ref{sub_NPG24} and \ref{sub_NPG5}.  % following Appendix \ref{app_NPG}. %{\bf Revise Appendix number. }

Let the policies in the set $\Pi$ be parametrized by  $\Phi$. Without loss of generality, we can restrict the set $\Pi$ to comprise of only stationary policies \citep{puterman2014markov}. To simplify notations, we denote a stationary policy with the parameter $\Phi$ as $\boldsymbol{\pi}_{\Phi}\triangleq \{\pi_{\Phi}^k\}_{k\in[K]}$, where $\pi_{\Phi}^k$'s are stationary decision rules for each class. In a $K$-class MFC, we need to track only one representative agent from each class. The $k$-th representative takes its action $u_k$ by observing its own state $x_k$ and the joint distribution $\boldsymbol{\mu}$. If $\mathbf{x}\triangleq \{x_k\}_{k\in[K]}$ and $\mathbf{u}\triangleq \{u_k\}_{k\in[K]}$, then $K$-class MFC can effectively be described as a single agent RL problem with  $(\mathbf{x},\boldsymbol{\mu})\in\mathcal{X}^K\times \mathcal{P}(\mathcal{X}\times [K])$ and $\mathbf{u}\in\mathcal{U}^K$ as its state and action respectively. However, such a system comes with the additional advantage that the actions $u_k$'s are conditionally independent given $\mathbf{x}$. It will be clear from our later result (Theorem \ref{corr_1}) that, this prevents the complexity of the problem from being an exponential function of $K$.

For arbitrary $\boldsymbol{\mu}\in \mathcal{P}(\mathcal{X}\times [K])$, $\mathbf{x}\in\mathcal{X}^K$, and $\mathbf{u}\in\mathcal{U}^K$, denote the $Q$-value and the advantage value associated with the policy $\boldsymbol{\pi}_{\Phi}$ as $Q_{\Phi}(\mathbf{x},\boldsymbol{\mu},\mathbf{u})$ and $A_{\Phi}(\mathbf{x},\boldsymbol{\mu},\mathbf{u})$ respectively. The precise definition of Q-function is as follows. 
\begin{align}
	Q_{\Phi}(\mathbf{x},\boldsymbol{\mu},\mathbf{u})=\mathbb{E}_{\Phi}\Bigg[\sum_{k\in[K]}\sum_{t=0}^{\infty}\gamma^tr_k\left(x_k^t,u_k^t,\boldsymbol{\mu}_t,\boldsymbol{\nu}_t\right)\Bigg| \mathbf{x}_0=\mathbf{x},\boldsymbol{\mu}_0=\boldsymbol{\mu},\mathbf{u}_0=\mathbf{u}\Bigg]
	\label{Q}
\end{align}
	
	where the expectation is  computed over $u_k^t\sim \pi_{\Phi}^k(x_k^t,\boldsymbol{\mu}_t)$, $x_k^{t}\sim P_k(x_k^{t-1},u_k^{t-1},\boldsymbol{\mu}_{t-1},\boldsymbol{\nu}_{t-1})$, $\forall t\in\{1,2,\cdots\}$, $\forall k\in [K]$. Moreover, $\forall t\in\{1,2,\cdots\}$, $\boldsymbol{\mu}_{t}=P^{\mathrm{MF}}(\boldsymbol{\mu}_{t-1},\boldsymbol{\pi}_{\Phi})$, $\boldsymbol{\nu}_t=\nu^{\mathrm{MF}}(\boldsymbol{\mu}_t,\boldsymbol{\pi}_{\Phi})$ where $P^{\mathrm{MF}}(\cdot,\cdot)$, $\nu^{\mathrm{MF}}(\cdot,\cdot)$ are given by $(\ref{eq_phi})$, $(\ref{eq_nu_v})$ respectively. Incorporating $(\ref{Q})$, we now define the advantage function as follows
	\begin{align*}
	A_{\Phi}(\mathbf{x},\boldsymbol{\mu},\mathbf{u}) = Q_{\Phi}(\mathbf{x},\boldsymbol{\mu},\mathbf{u}) - \mathbb{E}\left[Q_{\Phi}(\mathbf{x},\boldsymbol{\mu},\mathbf{u})\right],
	\end{align*}
	where the expectation is  over $u_k\sim\pi_{\Phi}^k(x_k,\boldsymbol{\mu})$, $\forall k\in [K]$.

Define $v_{\mathrm{MF}}^*({\boldsymbol{\mu}_0})\triangleq\sup_{\Phi\in \mathbb{R}^{\mathrm{d}}}v^{\mathrm{MF}}(\boldsymbol{\mu}_0,\boldsymbol{\pi}_{\Phi})$, $ \boldsymbol{\mu}_0\in\mathcal{P}(\mathcal{X}\times [K])$, where $v^{\mathrm{MF}}(\cdot,\cdot)$ is the mean-field value function given by $(\ref{eqvr})$ and $\mathbb{R}^{\mathrm{d}}$ denotes the space of $\Phi$. Consider a sequence of parameters $\{\Phi_j\}_{j=1}^J$ that is recursively calculated by following the natural policy gradient (NPG) \citep{kakade2001natural,liu2020improved,agarwal2021theory} update as described below.
\begin{equation}
\label{eq_37}
%	\begin{split}
\Phi_{j+1} = \Phi_j + \eta \mathbf{w}_j,  \mathbf{w}_j\triangleq{\arg\min}_{\mathbf{w}\in\mathbb{R}^{\mathrm{d}}} ~L_{ \zeta_{\boldsymbol{\mu}_0}^{\Phi_j}}(\mathbf{w},\Phi_j)
%	\end{split}
\end{equation}
where $\eta$ is the learning rate. The definitions of the function $L_{ \zeta_{\boldsymbol{\mu}_0}^{\Phi_j}}$ and the distribution $\zeta_{\boldsymbol{\mu}_0}^{\Phi_j}$ are provided below. Define the following function
	$\forall \Phi, \Phi'\in \mathbb{R}^{\mathrm{d}}$, 
	\begin{align}
	\begin{split}
	L_{ \zeta_{\boldsymbol{\mu}_0}^{\Phi'}}(\mathbf{w},\Phi)\triangleq\mathbb{E}_{(\mathbf{x},\boldsymbol{\mu},\mathbf{u})\sim \zeta_{\boldsymbol{\mu}_0}^{\Phi'}}\Big[\Big(A_{\Phi}(\mathbf{x},\boldsymbol{\mu},\mathbf{u})
	-(1-\gamma)\mathbf{w}^{\mathrm{T}}\nabla_{\Phi}\log \prod_{k\in[K]}\pi_{\Phi}^k(x_k,\boldsymbol{\mu})(u_k) \Big)^2\Big|
	\end{split}\\
	\label{eq_31}
	\begin{split}
	\text{and~}&\zeta_{\boldsymbol{\mu}_0}^{\Phi'}(\mathbf{x},\boldsymbol{\mu},\mathbf{u})\triangleq (1-\gamma)\sum_{\tau=0}^{\infty}\gamma^{\tau} \mathbb{P}(\mathbf{x}_\tau=\mathbf{x},\boldsymbol{\mu}_{\tau}=\boldsymbol{\mu},\mathbf{u}_\tau=\mathbf{u}
	|\mathbf{x}_0=\mathbf{x},\boldsymbol{\mu}_0=\boldsymbol{\mu},\mathbf{u}_0=\mathbf{u},\boldsymbol{\pi}_{\Phi'})
	\end{split}
	\end{align}

It is evident from $(\ref{eq_37})$ that at each NPG update, one needs to solve a stochastic minimization problem to find the update direction. This sub-problem can be solved by another stochastic gradient descent algorithm with the  update equation $\mathbf{w}_{j,l+1}=\mathbf{w}_{j,l}-\alpha\mathbf{h}_{j,l}$ \citep{liu2020improved}, where  $\alpha$ is the learning rate and the update direction $\mathbf{h}_{j,l}$ is defined as:
\begin{eqnarray}
\mathbf{h}_{j,l}\triangleq \Bigg(\mathbf{w}_{j,l}^{\mathrm{T}}\nabla_{\Phi_j}\log \prod_{k\in[K]}\pi_{\Phi_j}^k(x_k,\boldsymbol{\mu})(u_k)-\dfrac{1}{1-\gamma}\nonumber\hat{A}_{\Phi_j}(\mathbf{x},\boldsymbol{\mu},\mathbf{u})\Bigg)
\nabla_{\Phi_j}\log \prod_{k\in[K]}\pi_{\Phi_j}^k(x_k,\boldsymbol{\mu})(u_k)
\label{eq_33}
\end{eqnarray}
where $\mathbf{x}=\{x_k\}_{k\in[K]},\mathbf{u}=\{u_k\}_{k\in[K]}$, $(\mathbf{x},\boldsymbol{\mu},\mathbf{u})$ are sampled from $\zeta_{\boldsymbol{\mu}_0}^{\Phi_j}$ and $\hat{A}_{\Phi_j}$ is an unbiased estimator of $A_{\Phi}$. The details of procuring the samples and the unbiased estimate is provided in Algorithm \ref{algo_2} which is based on Algorithm 3 of \citep{agarwal2021theory}. In Algorithm \ref{algo_1}, we summarize the NPG-based procedure to obtain the optimal MFC policy.
\begin{algorithm}
	\caption{Natural Policy Gradient for $K$-class MFC}
	\label{algo_1}
	\textbf{Input:} $\eta,\alpha$: Learning rates, $J,L$: Number of execution steps\\
	\hspace{1.3cm}$\mathbf{w}_0,\Phi_0$: Initial parameters, $\boldsymbol{\mu}_0$: Initial state distribution\\
	\textbf{Initialization:} $\Phi\gets \Phi_0$ 
	\begin{algorithmic}[1]
		\FOR{$j\in\{0,1,\cdots,J-1\}$}
		{
			\STATE $\mathbf{w}_{j,0}\gets \mathbf{w}_0$\\
			\FOR {$l\in\{0,1,\cdots,L-1\}$}
			{
				\STATE Sample $(\mathbf{x},\boldsymbol{\mu},\mathbf{u})\sim\zeta_{\boldsymbol{\mu}_0}^{\Phi_j}$ and $\hat{A}_{\Phi_j}(\mathbf{x},\boldsymbol{\mu},\mathbf{u})$ using Algorithm \ref{algo_2}\\
				\STATE Compute $\mathbf{h}_{j,l}$ using $(\ref{eq_33})$\\
				$\mathbf{w}_{j,l+1}\gets\mathbf{w}_{j,l}-\alpha\mathbf{h}_{j,l}$
			}
			\ENDFOR
			\STATE	$\mathbf{w}_j\gets\dfrac{1}{L}\sum_{l=1}^{L}\mathbf{w}_{j,l}$\\
			\STATE	$\Phi_{j+1}\gets \Phi_j +\eta \mathbf{w}_j$
		}
		\ENDFOR
	\end{algorithmic}
	\textbf{Output:} $\{\Phi_1,\cdots,\Phi_J\}$: Policy parameters
\end{algorithm}

\begin{algorithm}
	\caption{Algorithm to sample $(\mathbf{x},\boldsymbol{\mu},\mathbf{u})\sim\zeta_{\boldsymbol{\mu}_0}^{\Phi_j}$ and $\hat{A}_{\Phi_j}(\mathbf{x},\boldsymbol{\mu},\mathbf{u})$}
	\label{algo_2}
	\textbf{Input:} $\boldsymbol{\mu}_0$: Initial joint state distribution, $\boldsymbol{\pi}_{\Phi_j}\triangleq\{\pi_{\Phi_j}^k\}_{k\in[K]}$: Policy,\\
	$\{P_k(.,.,.,.)\}_{k\in[K]}$: Transition laws, $\{r_k(.,.,.,.)\}_{k\in[K]}$: Reward functions,\\ $\boldsymbol{\theta}\triangleq\{\theta_k\}_{k\in[K]}$: Prior probabilities of different classes.\\
	\begin{algorithmic}[1]
		\STATE Sample $\mathbf{x}_0\triangleq\{x_k^0\}_{k\in[K]}\sim \boldsymbol{\mu}_0$. 
		\STATE Sample $\mathbf{u}_0\triangleq\{u_k^0\}_{k\in[K]}\sim \boldsymbol{{\pi}}_{\Phi_j}(\mathbf{x}_0,\boldsymbol{\mu}_0)$ i.e., sample $u_k^0\sim \pi_{\Phi_j}^k(x_k^0,\boldsymbol{\mu}_0)$, $\forall k\in [K]$.
		\STATE  $\boldsymbol{\nu}_0\gets\nu^{\mathrm{MF}}(\boldsymbol{\mu}_0,\boldsymbol{\pi}_{\Phi_j})$ where $\nu^{\mathrm{MF}}$ is defined in $(\ref{eq_nu_v})$.
		\vspace{0.2cm}
		\STATE $t\gets 0$ 
		\STATE $\mathrm{FLAG}\gets \mathrm{FALSE}$
		\WHILE{$\mathrm{FLAG~is~} \mathrm{FALSE}$}
		{
			\STATE $\mathrm{FLAG}\gets \mathrm{TRUE}$ with probability $1-\gamma$.
			\STATE Execute $\mathrm{SystemUpdate}$
			
		}
		\ENDWHILE
		
		\STATE $T\gets t$
		
		\STATE Accept   $(\mathbf{x}_T,\boldsymbol{\mu}_T,\mathbf{u}_T)$ as a sample.

		\vspace{0.2cm}
		\STATE $\hat{V}_{\Phi_j}\gets 0$, $\hat{Q}_{\Phi_j}\gets 0$
		
		\STATE $\mathrm{FLAG}\gets \mathrm{FALSE}$
		\STATE $\mathrm{SumRewards}\gets 0$
		\WHILE{$\mathrm{FLAG~is~} \mathrm{FALSE}$}
		{
			\STATE $\mathrm{FLAG}\gets \mathrm{TRUE}$ with probability $1-\gamma$.
			\STATE Execute $\mathrm{SystemUpdate}$
			\STATE $\mathrm{SumRewards}\gets \mathrm{SumRewards} + \sum_{k\in[K]}\theta_kr_k(x_k^t,u_k^t,\boldsymbol{\mu}_t,\boldsymbol{\nu}_t)$
			
		}
		\ENDWHILE
		
		\vspace{0.2cm}
		\STATE With probability $\frac{1}{2}$, $\hat{V}_{\Phi_j}\gets \mathrm{SumRewards}$. Otherwise $\hat{Q}_{\Phi_j}\gets \mathrm{SumRewards}$.  
		
		\STATE $\hat{A}_{\Phi_j}(\mathbf{x}_T,\boldsymbol{\mu}_T,\mathbf{u}_T)\gets 2(\hat{Q}_{\Phi_j}-\hat{V}_{\Phi_j})$.
		
	\end{algorithmic} 
	\textbf{Output}: $(\mathbf{x}_T,\boldsymbol{\mu}_T,\mathbf{u}_T)$ and $\hat{A}_{\Phi_j}(\mathbf{x}_T,\boldsymbol{\mu}_T,\mathbf{u}_T)$

	\vspace{0.3cm}
	
	\textbf{Procedure} $\mathrm{SystemUpdate}$:
	
	\begin{algorithmic}[1]
		\STATE Execute the actions $\mathbf{u}_t\triangleq \{u_k^t\}_{k\in[K]}$.
		\STATE Transition to $\mathbf{x}_{t+1}\triangleq \{x_k^{t+1}\}_{k\in [K]}$ following $x_k^{t+1}\sim P_k(x_k^t,u_k^t,\boldsymbol{\mu}_t,\boldsymbol{\nu}_t)$, $\forall k\in [K]$.
		\STATE  $\boldsymbol{\mu}_{t+1}\gets P^{\mathrm{MF}}(\boldsymbol{\mu}_t,\boldsymbol{\pi}_{\Phi_j})$ where $P^{\mathrm{MF}}$ is defined in $(\ref{eq_phi})$.
		\STATE Sample $\mathbf{u}_{t+1}\triangleq\{u_k^{t+1}\}_{k\in[K]}\sim \boldsymbol{{\pi}}_{\Phi_j}(\mathbf{x}_{t+1},\boldsymbol{\mu}_{t+1})$ i.e., sample $u_k^{t+1}\sim \pi_{\Phi_j}^k(x_k^{t+1},\boldsymbol{\mu}_{t+1})$, $\forall k\in [K]$.
		\STATE $\boldsymbol{\nu}_{t+1}\gets\nu^{\mathrm{MF}}(\boldsymbol{\mu}_{t+1},\boldsymbol{\pi}_{\Phi_j})$ where $\nu^{\mathrm{MF}}$ is defined in $(\ref{eq_nu_v})$.
		\STATE $t\gets t+1$
	\end{algorithmic}
	\textbf{EndProcedure}
	\vspace{0.2cm}
\end{algorithm}

Following Theorem 4.9 of \citep{liu2020improved}, we can now state the global convergence result of NPG as given below. For the result to hold, the following Assumptions need to be satisfied. These assumptions are similar to Assumptions 2.1, 4.2, 4.4 respectively in \citep{liu2020improved}.

\begin{assumption}
	\label{ass_6}
	$\forall \Phi\in\mathbb{R}^{\mathrm{d}}$, $\forall \boldsymbol{\mu}_0\in\mathcal{P}(\mathcal{X}\times [K])$, the matrix $F_{\boldsymbol{\mu}_0}(\Phi)-\chi I_{\mathrm{d}}$ is positive semi-definite for some $\chi >0$ where $F_{\boldsymbol{\mu}_0}(\Phi)$ is defined as,
	\begin{align*}
	F_{\boldsymbol{\mu}_0}(\Phi)\triangleq \mathbb{E}_{(\mathbf{x},\boldsymbol{\mu},\mathbf{u})\sim \zeta_{\boldsymbol{\mu}_0}^{\Phi}}\left[\left\lbrace\nabla_{\Phi}\log\prod_{k\in[K]}\pi_{\Phi}^k(x_k,\boldsymbol{\mu})(u_k)\right\rbrace\left\lbrace\nabla_{\Phi}\log\prod_{k\in[K]}\pi_{\Phi}^k(x_k,\boldsymbol{\mu})(u_k)\right\rbrace^{\mathrm{T}}\right]
	\end{align*} 
\end{assumption}

\begin{assumption}
	\label{ass_7}
	$\forall \Phi\in\mathbb{R}^{\mathrm{d}}$, $\forall \boldsymbol{\mu}\in\mathcal{P}(\mathcal{X}\times [K])$, $\forall x_k\in\mathcal{X}$, $\forall u_k\in\mathcal{U}$, $\forall k\in [K]$,
	\begin{align*}
	\left|\nabla_{\Phi}\log\prod_{k\in[K]}\pi_{\Phi}^k(x_k,\boldsymbol{\mu})(u_k)\right|_1\leq G
	\end{align*}
	for some positive constant $G$.
\end{assumption}

\begin{assumption}
	\label{ass_8}
	$\forall \Phi_1,\Phi_2\in\mathbb{R}^{\mathrm{d}}$, $\forall \boldsymbol{\mu}\in\mathcal{P}(\mathcal{X}\times [K])$,  $\forall x_k\in\mathcal{X}$, $\forall u_k\in\mathcal{U}$, $\forall k\in [K]$,
	\begin{align*}
	\left|\nabla_{\Phi_1}\log\prod_{k\in[K]}\pi_{\Phi_1}^k(x_k,\boldsymbol{\mu})(u_k)-\nabla_{\Phi_2}\log\prod_{k\in[K]}\pi_{\Phi_2}^k(x_k,\boldsymbol{\mu})(u_k)\right|_1\leq M|\Phi_1-\Phi_2|_1
	\end{align*}
	
	for some positive constant $M$.
\end{assumption}

\begin{assumption}
	\label{ass_9}
	$\forall \Phi\in\mathbb{R}^{\mathrm{d}}$, $\forall \boldsymbol{\mu}_0\in\mathcal{P}(\mathcal{X}\times [K])$, the following holds true
	\begin{align*}
	L_{\zeta_{\boldsymbol{\mu}_0}^{\Phi^*}}(\mathbf{w}^{*}_{\Phi},\Phi)\leq \epsilon_{\mathrm{bias}}, ~~\mathbf{w}^*_{\Phi}\triangleq{\arg\min}_{\mathbf{w}\in\mathbb{R}^{\mathrm{d}}} L_{\zeta_{\boldsymbol{\mu}_0}^{\Phi}}(\mathbf{w},\Phi) 
	\end{align*}
	where $\Phi^*$ is the parameter associated with an optimal policy.
\end{assumption}

\begin{lemma}
	\label{lemma_10}
	If $\{\Phi_j\}_{j=1}^J$ are computed following Algorithm \ref{algo_1}, and Assumptions \ref{ass_6}$-$\ref{ass_9} are satisfied, then for appropriate choices of $\eta, \alpha, J,L$, 
	\begin{align*}
	v_{\mathrm{MF}}^*(\boldsymbol{\mu}_0)-\dfrac{1}{J}\sum_{j=1}^J v^{\mathrm{MF}}({\boldsymbol{\mu}_0},\boldsymbol{\pi}_{\Phi_j}) \leq \dfrac{\sqrt{\epsilon_{\mathrm{bias}}}}{1-\gamma}+\epsilon,
	\end{align*}  
	for arbitrary initial state distribution $\boldsymbol{\mu}_0\in\mathcal{P}(\mathcal{X}\times [K])$ and initial parameter $\Phi_0$. The sample complexity of Algorithm \ref{algo_1} is $\mathcal{O}(\epsilon^{-3})$. The parameter $\epsilon_{\mathrm{bias}}$ is a constant. 
\end{lemma}

The parameter $\epsilon_{\mathrm{bias}}$ measures the capacity of  parametrization. For rich neural network based policies, we can assume $\epsilon_{\mathrm{bias}}$ to be small \citep{liu2020improved}. 

Lemma \ref{lemma_10} states that, with a sample complexity of $\mathcal{O}(\epsilon^{-3})$, Algorithm \ref{algo_1} can approximate the optimal mean-field value function with an error margin of $\epsilon$.  Combining this with Theorem \ref{thm_1}, we obtain the following result.
\begin{theorem}
	\label{corr_1}
	Let $\mathbf{x}_0^{\mathbf{N}}$ be the initial states and $\boldsymbol{\mu}_0$ their associated distribution. If the parameters $\{\Phi_j\}_{j=1}^J$ are obtained by following Algorithm \ref{algo_1}, then under Assumptions \ref{ass_1}, \ref{assumptions}, and the set of assumptions used in Lemma \ref{lemma_10}, the following inequality holds for appropriate choices of $\eta, \alpha,J,L$ if $\gamma S_P<1$ 
	\begin{align}
	\label{eq_thm4}
	\begin{split}
	&	\left|\sup_{\Phi\in\mathbb{R}^{\mathrm{d}}}v^{\mathbf{N}}(\boldsymbol{\mu}_0,\pi_{\Phi})-\dfrac{1}{J}\sum_{j=1}^J v^{\mathrm{MF}}({\boldsymbol{\mu}_0},\boldsymbol{\pi}_{\Phi_j})\right|
	 \leq \dfrac{\sqrt{\epsilon_{\mathrm{bias}}}}{1-\gamma}+Ce_1\\
	 \text{where}& ~e_1\triangleq\left[\sqrt{|\mathcal{X}|}+\sqrt{|\mathcal{U}|}\right]\dfrac{1}{N_{\mathrm{pop}}}\sum_{k\in[K]}\sqrt{N_k}
	\end{split}
	\end{align}  
	where $S_P$ is defined in Theorem \ref{thm_1}, $C$ is a constant and the parameter $\epsilon_{\mathrm{bias}}$ is defined in Lemma \ref{lemma_10}. The sample complexity of the process is $\mathcal{O}(e_1^{-3})$.
\end{theorem}

Theorem \ref{corr_1} states that, with a sample complexity of $\mathcal{O}(e_1^{-3})$, Algorithm \ref{algo_1} generates a policy which is within $\mathcal{O}(e_1)$ error of the optimal heterogeneous MARL policy.

Note that both time and space complexity of the sampling step in Algorithm \ref{algo_1} is $\mathcal{O}(K)$. In contrast, if NPG is directly applied to MARL, those complexities increase to $\mathcal{O}(N_{\mathrm{pop}})$. Therefore, MFC based NPG provides an advantage of the order of $N_{\mathrm{pop}}/K$ in comparison to MARL based NPG.

In the following subsections, we shall establish results similar to Theorem \ref{corr_1} for the set of assumptions used in Theorem \ref{thm_2} and \ref{thm_3}.

\subsection{NPG with Assumption \ref{ass_2} and \ref{ass_4}}
\label{sub_NPG24}

If a multi-agent system satisfies Assumption \ref{ass_2},  \ref{ass_4}, and the set of stationary policies, $\bar{\Pi}$ is parametrized by $\Phi\in \mathbb{R}^{\mathrm{d}}$, then similar to Algorithm \ref{algo_1}, an NPG-based algorithm can be made to obtain its global optimal policy within $\bar{\Pi}$. Let this algorithm be denoted as $\mathrm{NPG}_{2,4}$. Algorithm $\mathrm{NPG}_{2,4}$ is identical to  Algorithm \ref{algo_1} except the joint distribution $\boldsymbol{\mu}\in \mathcal{P}(\mathcal{X}\times [K])$ in Algorithm \ref{algo_1} is replaced by $\bar{\boldsymbol{\mu}}\in\mathcal{P}^K(\mathcal{X})$,  in $\mathrm{NPG}_{2,4}$.
To show its global convergence, we need to assume a set of assumptions that are identical to those used in Lemma \ref{lemma_10}, except the joint distributions in all those assumptions must be replaced by the collection of distributions over all classes. Let this set of assumptions be denoted as $\mathrm{ASMP}_{2,4}$.

Following the same line of argument as is used in Theorem \ref{corr_1}, we can derive the result stated below.

\begin{theorem}
	\label{corr_2}
	Let $\mathbf{x}_0^{\mathbf{N}}$ be the initial states and $\boldsymbol{\bar{\mu}}_0\in\mathcal{P}^K(\mathcal{X})$ their associated distribution. If the parameters $\{\Phi_j\}_{j=1}^J$ are obtained by following $\mathrm{NPG}_{2,4}$, then under Assumptions \ref{ass_2}, \ref{ass_4}, and $\mathrm{ASMP}_{2,4}$, the following inequality holds for appropriate choices of the Algorithm parameters, $\eta, \alpha,J,L$ if $\gamma \bar{S}_P<1$. 
	\begin{align}
	\label{eq_thm4_app}
	\begin{split}
	&	\left|\sup_{\Phi\in\mathbb{R}^{\mathrm{d}}}\bar{v}^{\mathbf{N}}(\mathbf{x}_0^{\mathbf{N}},\bar{\pi}_{\Phi})-\dfrac{1}{J}\sum_{j=1}^J \bar{v}^{\mathrm{MF}}({\boldsymbol{\bar{\mu}}_0},\boldsymbol{\bar{\pi}}_{\Phi_j})\right| \leq \dfrac{\sqrt{\epsilon_{\mathrm{bias}}}}{1-\gamma}+\bar{C}e_2\\
	\text{where}& ~e_2\triangleq \left[\sqrt{|\mathcal{X}|}+\sqrt{|\mathcal{U}|}\right]\sum_{k\in[K]}\dfrac{1}{\sqrt{N_k}}
	\end{split}
	\end{align}  
	where $\bar{v}^{\mathbf{N}}$ is the empirical value function of the $\mathbf{N}$-agent system, $\bar{v}^{\mathrm{MF}}$ is its mean-field limit, $\bar{\pi}_{\Phi}$ is the stationary decision rules associated with the policy $\boldsymbol{\bar{\pi}}_{\Phi}$, $\bar{S}_P$ is defined in Theorem \ref{thm_2}, $\bar{C}$ is a constant and the parameter $\epsilon_{\mathrm{bias}}$ is a measure of the capacity of  parametrization. The sample complexity of the process is $\mathcal{O}(e_2^{-3})$.
\end{theorem} 

Theorem \ref{corr_2} states that, with a sample complexity of $\mathcal{O}(e_2^{-3})$, Algorithm $\mathrm{NPG}_{2,4}$ can approximate the empirical value function of MARL within an error margin of  $\mathcal{O}(e_2)$.

\subsection{NPG with Assumption \ref{ass_5}}
\label{sub_NPG5}

If a multi-agent system satisfies Assumption \ref{ass_5}, and the set of stationary policies, ${\Pi}$ is parametrized by $\Phi\in \mathbb{R}^{\mathrm{d}}$, then similar to Algorithm \ref{algo_1}, an NPG-based algorithm can be made to obtain its global optimal policy within ${\Pi}$. Let this algorithm be denoted as $\mathrm{NPG}_{5}$. Algorithm $\mathrm{NPG}_{5}$ is identical to  Algorithm \ref{algo_1} except the joint distribution $\boldsymbol{\mu}\in \mathcal{P}(\mathcal{X}\times [K])$ in Algorithm \ref{algo_1} must be replaced by $\boldsymbol{\mu}[\mathcal{X}]\in\mathcal{P}(\mathcal{X})$,  in $\mathrm{NPG}_{5}$.
To show its global convergence, we need to assume a set of assumptions that are same as those used in Lemma \ref{lemma_10}, except the joint distributions in those assumptions must be replaced by marginal distributions. Let this set of assumptions be denoted as $\mathrm{ASMP}_{5}$.
Following the same line of argument as is used in Theorem \ref{corr_1}, we can derive the result stated below.

\begin{theorem}
	\label{corr_3}
	Let $\mathbf{x}_0^{\mathbf{N}}$ be the initial states and $\boldsymbol{\mu}_0\in\mathcal{P}(\mathcal{X}\times [K])$ their associated joint distribution. If the parameters $\{\Phi_j\}_{j=1}^J$ are obtained by following $\mathrm{NPG}_{5}$, then under Assumptions \ref{ass_5}, and $\mathrm{ASMP}_{5}$, the following inequality holds for appropriate choices of the Algorithm parameters, $\eta, \alpha,J,L$ if $\gamma {S}_P<1$. 
	\begin{align}
	\label{eq_thm5_app}
	\begin{split}
	&	\left|\sup_{\Phi\in\mathbb{R}^{\mathrm{d}}}v^{\mathbf{N}}(\mathbf{x}_0^{\mathbf{N}},\pi_{\Phi})-\dfrac{1}{J}\sum_{j=1}^J v^{\mathrm{MF}}({\boldsymbol{\mu}_0},\boldsymbol{\pi}_{\Phi_j})\right| \leq \dfrac{\sqrt{\epsilon_{\mathrm{bias}}}}{1-\gamma}+e_3\\
	\text{where~}& e_3 \triangleq \left[\sqrt{|\mathcal{X}|}+\sqrt{|\mathcal{U}|}\right]\left[\frac{A}{N_{\mathrm{pop}}}\sum_{k=1}^K\sqrt{N_k}+\frac{B}{\sqrt{N_{\mathrm{pop}}}}\right]
	\end{split}
	\end{align}  
	where $v^{\mathbf{N}}$ is the empirical value function of the $\mathbf{N}$-agent system, $v^{\mathrm{MF}}$ is its mean-field limit, ${\pi}_{\Phi}$ is the stationary decision rules associated with policy $\boldsymbol{\pi}_{\Phi}$, ${S}_P$ is defined in Theorem \ref{thm_1}, $A,B$ are constants and the parameter $\epsilon_{\mathrm{bias}}$ is a measure of the capacity of  parametrization. The sample complexity of the process is $\mathcal{O}(e_3^{-3})$.
\end{theorem} 

Theorem \ref{corr_3} states that, with  $\mathcal{O}(e_3^{-3})$ sample complexity, Algorithm $\mathrm{NPG}_{5}$ can approximate the empirical value function of MARL within an error margin of  $\mathcal{O}(e_3)$.

\section{Conclusions}
In this paper, we prove that a $K$-class heterogeneous cooperative MARL problem can be approximated by its associated MFC problem. We also  provide estimates of the approximation error as a function of class sizes for various set of assumptions. Finally, we propose a natural policy gradient based algorithm that approximates the optimal MARL  policy in a  sample efficient manner. Exchangeability among  agents is one of the most important assumptions in MFC-type analyses. It allows the influence of the whole population to be summarized by the state-action distribution. In many  scenarios of practical interest, however,  agents interact only with certain number of neighbouring agents. As a result, the presumption of exchangeability may only hold  locally. Establishing MFC-type approximation  for system with limited agent exchangeability is an important direction to pursue in the future. 

%\section{Acknowledgement}
%This work was supported in part by the National Science Foundation under grant CNS-1618335. 

\newpage
\appendix
%\onecolumn

\section{Proof of Theorem \ref{thm_1}}
\label{app_th_1}

The following results are needed to prove the theorem. The proofs of Lemma \ref{lemma1}-\ref{lemma_mu} are relegated to Appendix \ref{app_proof_lemma_2}-\ref{section_proof_lemma_8} respectively.

\subsection{Continuity Lemmas}

\begin{lemma}
	\label{lemma1}
	If $\nu^{\mathrm{MF}}(\cdot,\cdot)$ is defined by (\ref{eq_nu_v}), then $\forall \boldsymbol{\mu},\boldsymbol{\mu}'\in\mathcal{P}(\mathcal{X}\times[K])$ and $\forall \boldsymbol{\pi}=\{\pi_k\}_{k\in[K]}$ where $\pi_k$'s are decision rules satisfying Assumption \ref{assumptions}, the following inequality holds.
	\begin{align}
	|\nu^{\mathrm{MF}}(\boldsymbol{\mu},\boldsymbol{\pi})-\nu^{\mathrm{MF}}(\boldsymbol{\mu}',\boldsymbol{\pi})|_{1}\leq (1+L_Q)|\boldsymbol{\mu}-\boldsymbol{\mu}'|_{1} 
	\end{align}
\end{lemma}

\begin{lemma}
	\label{lemma2}
	If $r_k^{\mathrm{MF}}(\cdot,\cdot)$ satisfies (\ref{eqr}), then  $\forall \boldsymbol{\mu},\boldsymbol{\mu}'\in\mathcal{P}(\mathcal{X}\times[K])$ and $\forall \boldsymbol{\pi}=\{\pi_k\}_{k\in[K]}$ where $\pi_k$'s are decision rules satisfying Assumption \ref{assumptions}, the following inequality holds.
	\begin{align}
	\label{eq_SR}
	\begin{split}
	&\sum_{k\in[K]}	|	r_k^{\mathrm{MF}}(\boldsymbol{\mu},\boldsymbol{\pi})-	r_k^{\mathrm{MF}}(\boldsymbol{\mu}',\boldsymbol{\pi})|
	\leq S_R|\boldsymbol{\mu}-\boldsymbol{\mu}'|_{1} \\
	\text{where~}&S_R\triangleq M_R(1+L_Q)+L_R\left[2+L_Q\right]
	\end{split}
	\end{align}
\end{lemma}

\begin{lemma}
	\label{lemma3}
	If $P^{\mathrm{MF}}(\cdot,\cdot)$ is defined by (\ref{eq_phi}), then $\forall \boldsymbol{\mu},\boldsymbol{\mu}'\in\mathcal{P}(\mathcal{X}\times[K])$ and $\forall \boldsymbol{\pi}=\{\pi_k\}_{k\in[K]}$ where $\pi_k$'s denote decision rules satisfying Assumption \ref{assumptions}, the following inequality holds.
	\begin{align}
	\label{eq_SP}
	\begin{split}
	&|	P^{\mathrm{MF}}(\boldsymbol{\mu},\boldsymbol{\pi})-	P^{\mathrm{MF}}(\boldsymbol{\mu}',\boldsymbol{\pi})|_{1}\leq S_P|\boldsymbol{\mu}-\boldsymbol{\mu}'|_{1} \\
	\text{where~}&S_P\triangleq  (1+L_Q)+L_P\left[2+L_Q\right]
	\end{split}
	\end{align}
\end{lemma}

Lemma \ref{lemma1}-\ref{lemma3} essentially state that the average reward function, $r^{\mathrm{MF}}_k(\cdot,\cdot)$ defined by $(\ref{eqr})$ and the state and action evolution operators $P^{\mathrm{MF}}(\cdot,\cdot)$, $\nu^{\mathrm{MF}}(\cdot,\cdot)$ defined by $(\ref{eq_phi}), (\ref{eq_nu_v})$ respectively are Lipschitz continuous. These lemmas will be important in deriving the main result.

\subsection{Approximation Lemmas}
Recall that our primary goal is to prove that the value functions generated by a certain policy in a finite agent system can be well approximated by those generated by the same policy in the mean-field limit. As a precursor to this grand target, in this section, we discuss how various components of the value functions themselves behave when the population sizes become large. Lemma \ref{simple_lemma} serves as a key ingredient in many of the forthcoming lemmas.

\begin{lemma}
	If $\forall m\in[M]$, $\{X_{m,n}\}_{n\in[N]}$ are independent  random variables bounded within $[0,1]$ with $\sum_{m\in[M]}\mathbb{E}[X_{m,n}]=1$, $\forall n\in[N]$ and $\{C_{m,n}\}_{m\in[M],n\in[N]}\in\mathbb{R}$ are constants obeying $|C_{m,n}|\leq C$, $\forall m\in [M], \forall n\in [N]$, then the following holds.
	\begin{align}
	\label{eq_27}
	\sum_{m=1}^{M}\mathbb{E}\left|\sum_{n=1}^{N} C_{m,n}\Big( X_{m,n}-\mathbb{E}[X_{m,n}]\Big)\right|\leq C\sqrt{MN}
	\end{align} 
	\label{simple_lemma}
\end{lemma}

Below we state our first approximation result. Essentially, Lemma \ref{lemma_nu} provides an estimate of the difference between the empirical action distributions, $\boldsymbol{\nu}_t^{\mathbf{N}}$ and the action distribution that would have been obtained by following the mean-field action evolution operator $\nu(\cdot,\cdot)$, defined by $(\ref{eq_nu_v})$, in a finite agent system.
\begin{lemma} 
	\label{lemma_nu}
	If $\{\boldsymbol{\mu}_t^{\mathbf{N}},\boldsymbol{\nu}_t^{\mathbf{N}}\}_{t\in\{0,1,\cdots\}}$ are empirical joint state and action distributions induced by the policy $\boldsymbol{\pi}=\{\boldsymbol{\pi}_t\}_{t\in\{0,1,\cdots\}}$, then the following inequality holds $\forall t\in\{0,1,\cdots\}$.
	\begin{align}
	\mathbb{E}|\boldsymbol{\nu}_t^{\mathbf{N}}-\nu^{\mathrm{MF}}(\boldsymbol{\mu}_t^{\mathbf{N}},\boldsymbol{\pi}_t)|_{1}\leq \dfrac{1}{N_{\mathrm{pop}}}\left(\sum_{k\in[K]}\sqrt{N_k}\right)\sqrt{|\mathcal{U}|}
	\end{align}
\end{lemma}

Lemma \ref{lemma_r} (stated below) bounds the error between the empirical average reward and the reward obtained by following the mean-field averaging process quantified by  $(\ref{eqr})$.

\begin{lemma}
	\label{lemma_r}
	If $\{\boldsymbol{\mu}_t^{\mathbf{N}},\boldsymbol{\nu}_t^{\mathbf{N}}\}_{t\in\{0,1,\cdots\}}$ are  empirical joint state and action distributions induced by the policy $\boldsymbol{\pi}=\{\boldsymbol{\pi}_t\}_{t\in\{0,1,\cdots\}}$, then the following  holds  $\forall t\in\{0,1,\cdots\}$.
	\begin{align}
	\label{eq_34_app}
	\mathbb{E}\left|\dfrac{1}{N_{\mathrm{pop}}}\sum_{k\in[K]}\sum_{j=1}^{N_k}r_k(x_{j,k}^{t,\mathbf{N}},u_{j,k}^{t,\mathbf{N}},\boldsymbol{\mu}_t^{\mathbf{N}},\boldsymbol{\nu}_t^{\mathbf{N}})-\sum_{k\in[K]}r_k^{\mathrm{MF}}(\boldsymbol{\mu}_t^{\mathbf{N}},\boldsymbol{\pi}_t)\right|\leq C_R\sqrt{|\mathcal{U}|} \dfrac{1}{N_{\mathrm{pop}}}\left(\sum_{k\in[K]}\sqrt{N_k}\right)
	\end{align}
	where  $C_R=M_R+L_R$. 
\end{lemma}

Finally, Lemma \ref{lemma_mu} computes an upper bound on the error between the empirical state distribution, $\boldsymbol{\mu}_{t+1}^{\mathbf{N}}$ and the distribution that would have been obtained by following the mean-field state distribution evolution operator $P(\cdot,\cdot)$, defined by $(\ref{eq_phi})$ in a finite agent system.
\begin{lemma}
	\label{lemma_mu}
	If $\{\boldsymbol{\mu}_t^{\mathbf{N}}\}_{t\in\{0,1,\cdots\}}$ are empirical joint state distributions induced by the policy $\boldsymbol{\pi}=\{\boldsymbol{\pi}_t\}_{t\in\{0,1,\cdots\}}$, then the following inequality holds $\forall t\in\{0,1,\cdots\}$.
	\begin{align}
	\label{eq_35_app}
	\mathbb{E}\left|\boldsymbol{\mu}_{t+1}^{\mathbf{N}}-P^{\mathrm{MF}}(\boldsymbol{\mu}_t^{\mathbf{N}},\boldsymbol{\pi}_t)\right|_{1}\leq  C_P\left[\sqrt{|\mathcal{X}|}+\sqrt{|\mathcal{U}|}\right] \dfrac{1}{N_{\mathrm{pop}}}\left(\sum_{k\in[K]}\sqrt{N_k}\right)
	\end{align}
	where $C_P=2+L_P$. 
\end{lemma}

\subsection{Proof of the Theorem}
We are now ready to prove the theorem. Using $(\ref{eq_10}), (\ref{eq_vN})$, and $(\ref{eqvr})$, we can write,
\begin{align}
\left|v^{\mathbf{N}}(\mathbf{x}_0^{\mathbf{N}},{\boldsymbol{\pi}})-v^{\mathrm{MF}}(\boldsymbol{\mu}_0,\boldsymbol{\pi})\right|\leq J_1 + J_2	
\end{align}
where the first term $J_1$ is defined as follows:
\begin{align*}
\begin{split}
J_1&\triangleq \sum_{t=0}^{\infty}\gamma^t\mathbb{E}\left|\dfrac{1}{N_{\mathrm{pop}}}\sum_{k\in[K]}\sum_{j=1}^{N_k}r_k(x_{j,k}^{t,\mathbf{N}},u_{j,k}^{t,\mathbf{N}},\boldsymbol{\mu}_t^{\mathbf{N}},\boldsymbol{\nu}_t^{\mathbf{N}})-\sum_{k\in[K]}r_k^{\mathrm{MF}}(\boldsymbol{\mu}_t^{\mathbf{N}},\boldsymbol{\pi}_t)\right| \\
&\overset{(a)}{\leq} \dfrac{C_R}{1-\gamma} \sqrt{|\mathcal{U}|}\dfrac{1}{N_{\mathrm{pop}}}\left(\sum_{k\in[K]}\sqrt{N_k}\right)
\end{split}
\end{align*}

The inequality (a) follows from Lemma \ref{lemma_r}. The second term, $J_2$ is given as follows:
\begin{align}
\begin{split}
J_2&\triangleq\sum_{t=0}^{\infty}\gamma^t\left| \sum_{k\in[K]}r_k^{\mathrm{MF}}(\boldsymbol{\mu}_t,\boldsymbol{\pi}_t)-\sum_{k\in[K]}\mathbb{E}\left[r_k^{\mathrm{MF}}(\boldsymbol{\mu}_t^{\mathbf{N}},{\boldsymbol{\pi}}_t)\right]\right|\\
&\leq\sum_{t=0}^{\infty}\gamma^t \sum_{k\in[K]}\left|r_k^{\mathrm{MF}}(\boldsymbol{\mu}_t,\boldsymbol{\pi}_t)-\mathbb{E}\left[r_k^{\mathrm{MF}}(\boldsymbol{\mu}_t^{\mathbf{N}},{\boldsymbol{\pi}}_t)\right]\right|\\
&\overset{(a)}{=}\sum_{t=0}^{\infty}\gamma^t \sum_{k\in[K]}\left|\mathbb{E}\left[r_k^{\mathrm{MF}}(\boldsymbol{\mu}_t,\boldsymbol{\pi}_t)-r_k^{\mathrm{MF}}(\boldsymbol{\mu}_t^{\mathbf{N}},{\boldsymbol{\pi}}_t)\right]\right|\\
&\leq\sum_{t=0}^{\infty}\gamma^t \sum_{k\in[K]}\mathbb{E}\left|r_k^{\mathrm{MF}}(\boldsymbol{\mu}_t,\boldsymbol{\pi}_t)-r_k^{\mathrm{MF}}(\boldsymbol{\mu}_t^{\mathbf{N}},{\boldsymbol{\pi}}_t)\right|\overset{(b)}{\leq} S_R\left( \sum_{t=0}^{\infty}\gamma^t\mathbb{E}\left|\boldsymbol{\mu}_t^{\mathbf{N}}-\boldsymbol{\mu}_t\right|_{1}\right)
\end{split}
\label{eq14}
\end{align}

Equation (a) holds because the sequence $\{\boldsymbol{\mu}_t\}_{t\in\{0,1,\cdots\}}$ is deterministic. Inequality (b) is due to Lemma \ref{lemma2}. Observe that, $\forall t\geq 0$ the following holds,
\begin{align}
\mathbb{E}\left|\boldsymbol{\mu}_{t+1}^{\mathbf{N}}-\boldsymbol{\mu}_{t+1}\right|_{1}\leq \mathbb{E}\left|\boldsymbol{\mu}_{t+1}^{\mathbf{N}}-P^{\mathrm{MF}}\left(\boldsymbol{\mu}_t^{\mathbf{N}},{\boldsymbol{\pi}}_t\right)\right|_{1}+ \mathbb{E}\left|P^{\mathrm{MF}}\left(\boldsymbol{\mu}_t^{\mathbf{N}},{\boldsymbol{\pi}}_t\right)-\boldsymbol{\mu}_{t+1}\right|_{1}
\end{align} 

The first term can be upper bounded by invoking  Lemma \ref{lemma_mu}. Using Lemma \ref{lemma3}, the second term can be upper bounded as follows:
\begin{align}
\begin{split}
\mathbb{E}&\left|P^{\mathrm{MF}}\left(\boldsymbol{\mu}_t^{\mathbf{N}},\boldsymbol{\pi}_t\right)-\boldsymbol{\mu}_{t+1}\right|_{1}=\mathbb{E}\left|P^{\mathrm{MF}}\left(\boldsymbol{\mu}_t^{\mathbf{N}},\boldsymbol{\pi}_t\right)-P^{\mathrm{MF}}\left(\boldsymbol{\mu}_{t},\boldsymbol{\pi}_t\right)\right|_{1}\leq S_P (\mathbb{E} |\boldsymbol{\mu}_t^{\mathbf{N}}-\boldsymbol{\mu}_{t}|_{1})	
\end{split}
\end{align}

Recall that, $\boldsymbol{\mu}^{0,\mathbf{N}}=\boldsymbol{\mu}^0$. Therefore, 
\begin{align}	
\label{eq_35}
\begin{split}
\mathbb{E}\left|\boldsymbol{\mu}_{t+1}^{\mathbf{N}}-\boldsymbol{\mu}_{t+1}\right|_{1}&\leq C_P\left[\sqrt{|\mathcal{X}|}+\sqrt{|\mathcal{U}|}\right] \dfrac{1}{N_{\mathrm{pop}}}\left(\sum_{k\in[K]}\sqrt{N_k}\right)+S_P	\left(\mathbb{E}\left|\boldsymbol{\mu}_t^{\mathbf{N}}-\boldsymbol{\mu}_t\right|_{1}\right)\\
&\leq C_P\left[\sqrt{|\mathcal{X}|}+\sqrt{|\mathcal{U}|}\right] \dfrac{1}{N_{\mathrm{pop}}}\left(\sum_{k\in[K]}\sqrt{N_k}\right)\left(\dfrac{S_P^{t+1}-1}{S_P-1}\right)
\end{split}
\end{align}

Clearly, $J_2$ is upper bounded as follows,
\begin{align*}
J_2\leq C_P\left(\dfrac{S_R}{S_P-1}\right)\left[\sqrt{|\mathcal{X}|}+\sqrt{|\mathcal{U}|}\right] \dfrac{1}{N_{\mathrm{pop}}}\left(\sum_{k\in[K]}\sqrt{N_k}\right)\left(\dfrac{1}{1-\gamma S_P}-\dfrac{1}{1-\gamma}\right)
\end{align*}

This completes the proof of $(\ref{eq_28_new})$.

\section{Proof of Theorem \ref{thm_2}}

The collection of empirical state and action distributions of all classes at time $t$ are denoted as $\boldsymbol{\bar{\mu}}_t^{\mathbf{N}}\in\mathcal{P}^K(\mathcal{X})$ and $\boldsymbol{\bar{\nu}}_t^{\mathbf{N}}\in\mathcal{P}^K(\mathcal{U})$ respectively and their mean-field counterparts are $\boldsymbol{\bar{\mu}}_t,\boldsymbol{\bar{\nu}}_t$. The prior probability of $k$-th class, $k\in[K]$ will be denoted as $\theta_k=N_k/N_{\mathrm{pop}}$ and $\boldsymbol{\theta}\triangleq\{\theta_k\}_{k\in[K]}$.

\subsection{Mean-field equations}

For a policy $\boldsymbol{\bar{\pi}}\triangleq\{\boldsymbol{\bar{\pi}}_t\}_{t\in\{0,1,\cdots\}}\triangleq\{(\bar{\pi}_k^t)_{k\in[K]}\}_{t\in\{0,1,\cdots\}}$,  the mean-field action distribution is updated as,
\begin{align}
\begin{split}
\label{eq_nu_v_2}
\boldsymbol{\bar{\nu}}_t=\bar{\nu}^{\mathrm{MF}}(\boldsymbol{\bar{\mu}}_t, \boldsymbol{\bar{\pi}}_t)&\triangleq\{\bar{\nu}_k^{\mathrm{MF}}(\boldsymbol{\bar{\mu}}_t, \boldsymbol{\bar{\pi}}_t)\}_{k\in[K]}\\
\bar{\nu}^{\mathrm{MF}}_k(\boldsymbol{\bar{\mu}}_t, \boldsymbol{\bar{\pi}}_t) &\triangleq \sum_{x\in\mathcal{X}} \bar{\pi}_k^t(x,\boldsymbol{\bar{\mu}}_t)\boldsymbol{\bar{\mu}}_t(x, k)
\end{split}
\end{align}

Similarly, the state distribution is updated as,
\begin{align}
\begin{split}
\boldsymbol{\bar{\mu}}_{t+1}=\bar{P}^{\mathrm{MF}}(\boldsymbol{\bar{\mu}}_t, \boldsymbol{\bar{\pi}}_t)&\triangleq\{\bar{P}_k^{\mathrm{MF}}(\boldsymbol{\bar{\mu}}_t, \boldsymbol{\bar{\pi}}_t)\}_{k\in [K]}\\
\bar{P}_k^{\mathrm{MF}}(\boldsymbol{\bar{\mu}}_t, \boldsymbol{\bar{\pi}}_t) &\triangleq \sum_{x\in\mathcal{X}}\sum_{u\in\mathcal{U}} \bar{P}_k\left(x,u,\boldsymbol{\bar{\mu}}_t,\bar{\nu}^{\mathrm{MF}}(\boldsymbol{\bar{\mu}}_t,\boldsymbol{\bar{\pi}}_t)\right)\times\boldsymbol{\bar{\mu}}_t(x,k)\bar{\pi}_k^t(x,\boldsymbol{\bar{\mu}}_t)(u)
\label{eq_phi_2}
\end{split}
\end{align}

Finally, the average reward of $k$-th class are computed as,
\begin{align}
\begin{split}
\bar{r}_k^{\mathrm{MF}}(\boldsymbol{\bar{\mu}}_t, \boldsymbol{\bar{\pi}}_t) &\triangleq \sum_{x\in\mathcal{X}}\sum_{u\in\mathcal{U}} \bar{r}_k(x,u,\boldsymbol{\bar{\mu}}_t,\bar{\nu}^{\mathrm{MF}}(\boldsymbol{\bar{\mu}}_t,\boldsymbol{\bar{\pi}}_t))\times\boldsymbol{\bar{\mu}}_t(x,k)\bar{\pi}_k^t(x,\boldsymbol{\bar{\mu}}_t)(u)
\label{eqr_2}
\end{split}
\end{align}

For an initial state distribution $\boldsymbol{\bar{\mu}}_0$, and a policy $\boldsymbol{\bar{\pi}}$, the infinite-horizon $\gamma$-discounted average reward in the mean-field limit is,
\begin{align}
\begin{split}
\bar{v}^{\mathrm{MF}}(\boldsymbol{\bar{\mu}}_0,\boldsymbol{\bar{\pi}})= \sum_{k\in[K]}\theta_k\sum_{t=0}^{\infty}\gamma^t  \bar{r}_k^{\mathrm{MF}}(\boldsymbol{\bar{\mu}}_t, \boldsymbol{\bar{\pi}}_t)
\label{eqvr_2}
\end{split}
\end{align}

\subsection{Helper Lemmas}

The following results are necessary to prove the theorem. The proofs of Lemma \ref{lemma_continuity_2}, and $\ref{lemma_approx_2}$ have been relegated to Appendix \ref{app_lemma_9}, and \ref{app_lemma_10} respectively.

\begin{lemma}
	\label{lemma_continuity_2}
	The following inequalities hold $\forall \boldsymbol{\bar{\mu}},\boldsymbol{\bar{\mu}}'\in\mathcal{P}^K(\mathcal{X})$ and $\forall \boldsymbol{\bar{\pi}}=\{\bar{\pi}_k\}_{k\in[K]}$ where $\bar{\pi}_k$'s are decision rules satisfying Assumption \ref{ass_4}.
	\begin{align*}
	(a)&~ |\bar{\nu}^{\mathrm{MF}}(\boldsymbol{\bar{\mu}},\boldsymbol{\bar{\pi}})-\bar{\nu}^{\mathrm{MF}}(\boldsymbol{\bar{\mu}}',\boldsymbol{\bar{\pi}})|_{1}\leq (1+K\bar{L}_Q)|\boldsymbol{\bar{\mu}}-\boldsymbol{\bar{\mu}}'|_1\\
	(b)&~ \sum_{k\in[K]}	\theta_k|	\bar{r}_k^{\mathrm{MF}}(\boldsymbol{\bar{\mu}},\boldsymbol{\bar{\pi}})-	\bar{r}_k^{\mathrm{MF}}(\boldsymbol{\bar{\mu}}',\boldsymbol{\bar{\pi}})|
	\leq \bar{S}_R|\boldsymbol{\bar{\mu}}-\boldsymbol{\bar{\mu}}'|_{1}\\
	(c)&~ |\bar{P}^{\mathrm{MF}}(\boldsymbol{\bar{\mu}},\boldsymbol{\bar{\pi}})-	 \bar{P}^{\mathrm{MF}}(\boldsymbol{\bar{\mu}}',\boldsymbol{\bar{\pi}})|_{1}\leq \bar{S}_P|\boldsymbol{\bar{\mu}}-\boldsymbol{\bar{\mu}}'|_{1}
	\end{align*}
	where $\bar{S}_R\triangleq \bar{M}_R(1+\bar{L}_Q)+\bar{L}_R(2+K\bar{L}_Q)$ and $\bar{S}_P\triangleq (1+K\bar{L}_Q)+K\bar{L}_P(2+K\bar{L}_Q)$.
\end{lemma}

\begin{lemma} 
	\label{lemma_approx_2}
	If $\{\boldsymbol{\bar{\mu}}_t^{\mathbf{N}},\boldsymbol{\bar{\nu}}_t^{\mathbf{N}}\}_{t\in\{0,1,\cdots\}}$ are the collections of empirical state and action distributions of each classes induced by policy $\boldsymbol{\bar{\pi}}=\{\boldsymbol{\bar{\pi}}_t\}_{t\in\{0,1,\cdots\}}$, then the following inequalities hold true $\forall t\in\{0,1,\cdots\}$.
	\begin{align}
	&(a)~\mathbb{E}|\boldsymbol{\bar{\nu}}_t^{\mathbf{N}}-\bar{\nu}^{\mathrm{MF}}(\boldsymbol{\bar{\mu}}_t^{\mathbf{N}},\boldsymbol{\bar{\pi}}_t)|_{1}\leq \left(\sum_{k\in[K]}\dfrac{1}{\sqrt{N_k}}\right)\sqrt{|\mathcal{U}|}\\
	\label{eq_34_2}
	\begin{split}
	&(b)~	\mathbb{E}\left|\dfrac{1}{N_{\mathrm{pop}}}\sum_{k\in[K]}\sum_{j=1}^{N_k}\bar{r}_k(x_{j,k}^{t,\mathbf{N}},u_{j,k}^{t,\mathbf{N}},\boldsymbol{\bar{\mu}}_t^{\mathbf{N}},\boldsymbol{\bar{\nu}}_t^{\mathbf{N}})-\sum_{k\in[K]}\theta_k\bar{r}_k^{\mathrm{MF}}(\boldsymbol{\bar{\mu}}_t^{\mathbf{N}},\boldsymbol{\bar{\pi}}_t)\right|\\
	&\hspace{2cm}\leq \bar{C}_R\left(\sum_{k\in[K]}\dfrac{1}{\sqrt{N_k}}\right)\sqrt{|\mathcal{U}|} 
	\end{split}
	\\
	\label{eq35_2}
	&(c)~	\mathbb{E}\left|\boldsymbol{\bar{\mu}}_{t+1}^{\mathbf{N}}-\bar{P}^{\mathrm{MF}}(\boldsymbol{\bar{\mu}}_t^{\mathbf{N}},\boldsymbol{\bar{\pi}}_t)\right|_{1}\leq  \bar{C}_P\left(\sum_{k\in[K]}\dfrac{1}{\sqrt{N_k}}\right)\left[\sqrt{|\mathcal{X}|}+\sqrt{|\mathcal{U}|}\right] 
	\end{align}
	where $\bar{C}_R\triangleq \bar{M}_R+\bar{L}_R$, and $\bar{C}_P\triangleq 2+K\bar{L}_P$.
\end{lemma}

\subsection{Proof of the Theorem}

We are now ready to prove the theorem. Using $(\ref{eq_18})$ and $(\ref{eqvr_2})$, we can write,
\begin{align}
\left|\bar{v}^{\mathbf{N}}(\mathbf{x}_0^{\mathbf{N}},{\boldsymbol{\bar{\pi}}})-\bar{v}^{\mathrm{MF}}(\boldsymbol{\bar{\mu}}_0,\boldsymbol{\bar{\pi}})\right|\leq J_1 + J_2	
\end{align}
where the first term $J_1$ is defined as follows:
\begin{align*}
\begin{split}
J_1&\triangleq \sum_{t=0}^{\infty}\gamma^t\mathbb{E}\left|\dfrac{1}{N_{\mathrm{pop}}}\sum_{k\in[K]}\sum_{j=1}^{N_k}\bar{r}_k(x_{j,k}^{t,\mathbf{N}},u_{j,k}^{t,\mathbf{N}},\boldsymbol{\bar{\mu}}_t^{\mathbf{N}},\boldsymbol{\bar{\nu}}_t^{\mathbf{N}})-\sum_{k\in[K]}\theta_k\bar{r}_k^{\mathrm{MF}}(\boldsymbol{\bar{\mu}}_t^{\mathbf{N}},\boldsymbol{\bar{\pi}}_t)\right| \\
&\overset{(a)}{\leq} \dfrac{\bar{C}_R}{1-\gamma} \left(\sum_{k\in[K]}\dfrac{1}{\sqrt{N_k}}\right)\sqrt{|\mathcal{U}|}
\end{split}
\end{align*}

The inequality (a) follows from Lemma \ref{lemma_approx_2}. The second term, $J_2$ is given as follows:
\begin{align}
\begin{split}
J_2&\triangleq\sum_{t=0}^{\infty}\gamma^t\left| \sum_{k\in[K]}\theta_k\bar{r}_k^{\mathrm{MF}}(\boldsymbol{\bar{\mu}}_t,\boldsymbol{\bar{\pi}}_t)-\sum_{k\in[K]}\theta_k\mathbb{E}\left[\bar{r}_k^{\mathrm{MF}}(\boldsymbol{\bar{\mu}}_t^{\mathbf{N}},{\boldsymbol{\bar{\pi}}}_t)\right]\right|\\
&\leq\sum_{t=0}^{\infty}\gamma^t \sum_{k\in[K]}\theta_k\mathbb{E}\left|\bar{r}_k^{\mathrm{MF}}(\boldsymbol{\bar{\mu}}_t,\boldsymbol{\bar{\pi}}_t)-\bar{r}_k^{\mathrm{MF}}(\boldsymbol{\bar{\mu}}_t^{\mathbf{N}},{\boldsymbol{\bar{\pi}}}_t)\right|\\
&\overset{(a)}{\leq} \bar{S}_R\left( \sum_{t=0}^{\infty}\gamma^t\mathbb{E}\left|\boldsymbol{\bar{\mu}}_t^{\mathbf{N}}-\boldsymbol{\bar{\mu}}_t\right|_{1}\right)
\end{split}
\label{eq14_2}
\end{align}

Inequality (a) is due to Lemma \ref{lemma_continuity_2}. Observe that, $\forall t\geq 0$ the following holds,
\begin{align}
\mathbb{E}\left|\boldsymbol{\bar{\mu}}_{t+1}^{\mathbf{N}}-\boldsymbol{\bar{\mu}}_{t+1}\right|_{1}\leq \mathbb{E}\left|\boldsymbol{\bar{\mu}}_{t+1}^{\mathbf{N}}-\bar{P}^{\mathrm{MF}}\left(\boldsymbol{\bar{\mu}}_t^{\mathbf{N}},{\boldsymbol{\bar{\pi}}}_t\right)\right|_{1}+ \mathbb{E}\left|\bar{P}^{\mathrm{MF}}\left(\boldsymbol{\bar{\mu}}_t^{\mathbf{N}},{\boldsymbol{\bar{\pi}}}_t\right)-\boldsymbol{\bar{\mu}}_{t+1}\right|_{1}
\end{align} 

The first term can be upper bounded by invoking  Lemma \ref{lemma_approx_2}. Using Lemma \ref{lemma_continuity_2}, the second term can be upper bounded as follows:
\begin{align}
\begin{split}
\mathbb{E}&\left|\bar{P}^{\mathrm{MF}}\left(\boldsymbol{\bar{\mu}}_t^{\mathbf{N}},\boldsymbol{\bar{\pi}}_t\right)-\boldsymbol{\bar{\mu}}_{t+1}\right|_{1}=\mathbb{E}\left|\bar{P}^{\mathrm{MF}}\left(\boldsymbol{\bar{\mu}}_t^{\mathbf{N}},\boldsymbol{\bar{\pi}}_t\right)-\bar{P}^{\mathrm{MF}}\left(\boldsymbol{\bar{\mu}}_{t},\boldsymbol{\bar{\pi}}_t\right)\right|_{1}\leq \bar{S}_P (\mathbb{E} |\boldsymbol{\bar{\mu}}_t^{\mathbf{N}}-\boldsymbol{\bar{\mu}}_{t}|_{1})
\end{split}
\end{align}

Recall that, $\boldsymbol{\bar{\mu}}^{0,\mathbf{N}}=\boldsymbol{\bar{\mu}}^0$. Therefore, 
\begin{align}	
\label{eq_35_2}
\begin{split}
\mathbb{E}\left|\boldsymbol{\bar{\mu}}_{t+1}^{\mathbf{N}}-\boldsymbol{\bar{\mu}}_{t+1}\right|_{1}&\leq \bar{C}_P\left[\sqrt{|\mathcal{X}|}+\sqrt{|\mathcal{U}|}\right] \left(\sum_{k\in[K]}\dfrac{1}{\sqrt{N_k}}\right)+\bar{S}_P	\left(\mathbb{E}\left|\boldsymbol{\bar{\mu}}_t^{\mathbf{N}}-\boldsymbol{\bar{\mu}}_t\right|_{1}\right)\\
&\leq \bar{C}_P\left[\sqrt{|\mathcal{X}|}+\sqrt{|\mathcal{U}|}\right] \left(\sum_{k\in[K]}\dfrac{1}{\sqrt{N_k}}\right)\left(\dfrac{\bar{S}_P^{t+1}-1}{\bar{S}_P-1}\right)
\end{split}
\end{align}
Clearly, $J_2$ is upper bounded as follows,
\begin{align*}
J_2\leq \bar{C}_P\left(\dfrac{\bar{S}_R}{\bar{S}_P-1}\right)\left[\sqrt{|\mathcal{X}|}+\sqrt{|\mathcal{U}|}\right] \left(\sum_{k\in[K]}\dfrac{1}{\sqrt{N_k}}\right)\left(\dfrac{1}{1-\gamma \bar{S}_P}-\dfrac{1}{1-\gamma}\right)
\end{align*}

\section{Proof of Theorem \ref{thm_3}}

The following results are required to prove the theorem. The proofs of Lemma \ref{lemma_continuity}, \ref{lemma_approx} are given in Appendix \ref{app_lemma_11}, \ref{app_lemma_12} respectively. We define mean-field state, action distribution evolution functions $P^{\mathrm{MF}}(\cdot,\cdot)$, $\nu^{\mathrm{MF}}(\cdot,\cdot)$ and the class-average reward functions $r_k^{\mathrm{MF}}(\cdot,\cdot)$'s  by $(\ref{eq_phi}), (\ref{eq_nu_v})$, $(\ref{eqr})$, respectively.

\subsection{Helper Lemmas}

\begin{lemma}
	\label{lemma_continuity}
	The following inequalities hold $\forall \boldsymbol{\mu},\boldsymbol{\mu}'\in\mathcal{P}(\mathcal{X}\times[K])$ and $\forall \boldsymbol{\pi}=\{\pi_k\}_{k\in[K]}$ where $\pi_k$'s denote Lipschitz continuous decision rules with parameter $L_Q$.
	\begin{align*}
	(a)&~|\nu^{\mathrm{MF}}(\boldsymbol{\mu},\boldsymbol{\pi})[\mathcal{U}]-\nu^{\mathrm{MF}}(\boldsymbol{\mu}',\boldsymbol{\pi})[\mathcal{U}]|_{1}\leq |\nu^{\mathrm{MF}}(\boldsymbol{\mu},\boldsymbol{\pi})-\nu^{\mathrm{MF}}(\boldsymbol{\mu}',\boldsymbol{\pi})|_{1}\\
	&\hspace{5.5cm}\leq |\boldsymbol{\mu}-\boldsymbol{\mu}'|_1 + L_Q|\boldsymbol{\mu}[\mathcal{X}]-\boldsymbol{\mu}'[\mathcal{X}]|_1 \\
	(b)&~ \sum_{k\in[K]}	|	r_k^{\mathrm{MF}}(\boldsymbol{\mu},\boldsymbol{\pi})-	r_k^{\mathrm{MF}}(\boldsymbol{\mu}',\boldsymbol{\pi})|
	\leq S_R'|\boldsymbol{\mu}-\boldsymbol{\mu}'|_{1}+ S_R''|\boldsymbol{\mu}[\mathcal{X}]-\boldsymbol{\mu}'[\mathcal{X}]|_{1}\\
	(c)&~| P^{\mathrm{MF}}(\boldsymbol{\mu},\boldsymbol{\pi})[\mathcal{X}]-	 P^{\mathrm{MF}}(\boldsymbol{\mu}',\boldsymbol{\pi})[\mathcal{X}]|_{1}\leq| P^{\mathrm{MF}}(\boldsymbol{\mu},\boldsymbol{\pi})-	 P^{\mathrm{MF}}(\boldsymbol{\mu}',\boldsymbol{\pi})|_{1}\\
	&\hspace{5.85cm}\leq S_P'|\boldsymbol{\mu}-\boldsymbol{\mu}'|_{1}+ S_P''|\boldsymbol{\mu}[\mathcal{X}]-\boldsymbol{\mu}'[\mathcal{X}]|_{1}
	\end{align*}
	where $S_R'\triangleq M_R+L_R$, $S_R''\triangleq M_RL_Q+L_R(1+L_Q)$, $S_P'\triangleq 1+L_P$, and $S_P''\triangleq L_Q+L_P(1+L_Q)$. Note that, $S_R'+S_R''=S_R$ and $S_P'+S_P''=S_P$ where $S_R,S_P$ are defined in $(\ref{eq_SR}), (\ref{eq_SP})$ respectively.
\end{lemma}

Similar to Lemma \ref{lemma_nu}, \ref{lemma_r}, and \ref{lemma_mu}, we can derive the approximation results as follows.
\begin{lemma} 
	\label{lemma_approx}
	If $\{\boldsymbol{\mu}_t^{\mathbf{N}},\boldsymbol{\nu}_t^{\mathbf{N}}\}_{t\in\{0,1,\cdots\}}$ are the empirical joint state and action distributions induced by the policy $\boldsymbol{\pi}=\{\boldsymbol{\pi}_t\}_{t\in\{0,1,\cdots\}}$, then the following inequalities hold true $\forall t\in\{0,1,\cdots\}$.
	\begin{align}
	&(a)~\mathbb{E}|\boldsymbol{\nu}_t^{\mathbf{N}}[\mathcal{U}]-\nu^{\mathrm{MF}}(\boldsymbol{\mu}_t^{\mathbf{N}},\boldsymbol{\pi}_t)[\mathcal{U}]|_{1}\leq \dfrac{1}{\sqrt{N_{\mathrm{pop}}}}\sqrt{|\mathcal{U}|}\\
	\label{eq_34}
	&(b)~	\mathbb{E}\left|\dfrac{1}{N_{\mathrm{pop}}}\sum_{k\in[K]}\sum_{j=1}^{N_k}{r}_k(x_{j,k}^{t,\mathbf{N}},u_{j,k}^{t,\mathbf{N}},\boldsymbol{\mu}_t^{\mathbf{N}}[\mathcal{X}],\boldsymbol{\nu}_t^{\mathbf{N}}[\mathcal{U}])-\sum_{k\in[K]}r_k^{\mathrm{MF}}(\boldsymbol{\mu}_t^{\mathbf{N}},\boldsymbol{\pi}_t)\right|\leq \dfrac{C_R}{\sqrt{N_{\mathrm{pop}}}}\sqrt{|\mathcal{U}|} \\
	\label{eq35}
	&(c)~	\mathbb{E}\left|\boldsymbol{\mu}_{t+1}^{\mathbf{N}}[\mathcal{X}]-P^{\mathrm{MF}}(\boldsymbol{\mu}_t^{\mathbf{N}},\boldsymbol{\pi}_t)[\mathcal{X}]\right|_{1}\leq  \dfrac{C_P}{\sqrt{N_{\mathrm{pop}}}}\left[\sqrt{|\mathcal{X}|}+\sqrt{|\mathcal{U}|}\right] 
	\end{align}
	where $C_R$, $C_P$ are same as defined in Lemma \ref{lemma_r}, \ref{lemma_mu} respectively.
\end{lemma}

\subsection{Proof of the Theorem}
Following the proof of Theorem \ref{thm_1}, we can write,
\begin{align}
\left|v^{\mathbf{N}}(\mathbf{x}_0^{\mathbf{N}},{\boldsymbol{\pi}})-v^{\mathrm{MF}}(\boldsymbol{\mu}_0,\boldsymbol{\pi})\right|\leq J_1 + J_2	
\end{align}
where the first term $J_1$ is defined as follows:
\begin{align}
\begin{split}
J_1&\triangleq \sum_{t=0}^{\infty}\gamma^t\mathbb{E}\left|\dfrac{1}{N_{\mathrm{pop}}}\sum_{k\in[K]}\sum_{j=1}^{N_k}{r}_k(x_{j,k}^{t,\mathbf{N}},u_{j,k}^{t,\mathbf{N}},\boldsymbol{\mu}_t^{\mathbf{N}}[\mathcal{X}],\boldsymbol{\nu}_t^{\mathbf{N}}[\mathcal{U}])-\sum_{k\in[K]}r_k^{\mathrm{MF}}(\boldsymbol{\mu}_t^{\mathbf{N}},\boldsymbol{\pi}_t)\right|\\
& \overset{(a)}{\leq} \dfrac{C_R}{1-\gamma}\sqrt{|\mathcal{U}|} \dfrac{1}{\sqrt{N_{\mathrm{pop}}}}
\end{split}
\end{align}

The inequality (a) follows from Lemma \ref{lemma_approx}. The second term, $J_2$ is given as follows:
\begin{align}
\begin{split}
J_2&\triangleq\sum_{t=0}^{\infty}\gamma^t\left| \sum_{k\in[K]}r_k^{\mathrm{MF}}(\boldsymbol{\mu}_t,\boldsymbol{\pi}_t)-\sum_{k\in[K]}\mathbb{E}\left[r_k^{\mathrm{MF}}(\boldsymbol{\mu}_t^{\mathbf{N}},{\boldsymbol{\pi}}_t)\right]\right|\\
&\leq\sum_{t=0}^{\infty}\gamma^t \sum_{k\in[K]}\mathbb{E}\left|r_k^{\mathrm{MF}}(\boldsymbol{\mu}_t,\boldsymbol{\pi}_t)-r_k^{\mathrm{MF}}(\boldsymbol{\mu}_t^{\mathbf{N}},{\boldsymbol{\pi}}_t)\right|\\
&\overset{(a)}{\leq} \ \sum_{t=0}^{\infty}\gamma^t\left\lbrace S_R'\mathbb{E}\left|\boldsymbol{\mu}_t^{\mathbf{N}}-\boldsymbol{\mu}_t\right|_{1}+S_R''\mathbb{E}\left|\boldsymbol{\mu}_t^{\mathbf{N}}[\mathcal{X}]-\boldsymbol{\mu}_t[\mathcal{X}]\right|_{1}\right\rbrace
\end{split}
\label{eq44}
\end{align}

Inequality (a) is due to Lemma \ref{lemma_continuity}. Observe that, $\forall t\geq 0$ the following holds,
\begin{align}
\label{eq46}
\begin{split}&S_R'\mathbb{E}\left|\boldsymbol{\mu}_{t+1}^{\mathbf{N}}-\boldsymbol{\mu}_{t+1}\right|_{1}+S_R''\mathbb{E}\left|\boldsymbol{\mu}_{t+1}^{\mathbf{N}}[\mathcal{X}]-\boldsymbol{\mu}_{t+1}[\mathcal{X}]\right|_{1}\\
&\leq  S_R'\mathbb{E}\left|\boldsymbol{\mu}_{t+1}^{\mathbf{N}}-P^{\mathrm{MF}}\left(\boldsymbol{\mu}_t^{\mathbf{N}},{\boldsymbol{\pi}}_t\right)\right|_{1}+S_R''\mathbb{E}\left|\boldsymbol{\mu}_{t+1}^{\mathbf{N}}[\mathcal{X}]-P^{\mathrm{MF}}\left(\boldsymbol{\mu}_t^{\mathbf{N}},{\boldsymbol{\pi}}_t\right)[\mathcal{X}]\right|_{1}\\
&+ S_R'\mathbb{E}\left|P^{\mathrm{MF}}\left(\boldsymbol{\mu}_t^{\mathbf{N}},{\boldsymbol{\pi}}_t\right)-\boldsymbol{\mu}_{t+1}\right|_{1} + S_R''\mathbb{E}\left|P^{\mathrm{MF}}\left(\boldsymbol{\mu}_t^{\mathbf{N}},{\boldsymbol{\pi}}_t\right)[\mathcal{X}]-\boldsymbol{\mu}_{t+1}[\mathcal{X}]\right|_{1}
\end{split}
\end{align} 

The first two terms can be upper bounded by invoking  Lemma \ref{lemma_mu} and \ref{lemma_approx} respectively. Utilising Lemma \ref{lemma_continuity}, the last two term can be upper bounded as follows:
\begin{align}
\begin{split}
\mathbb{E}\left|P^{\mathrm{MF}}\left(\boldsymbol{\mu}_t^{\mathbf{N}},\boldsymbol{\pi}_t\right)[\mathcal{X}]-\boldsymbol{\mu}_{t+1}[\mathcal{X}]\right|_{1}&\leq \mathbb{E}\left|P^{\mathrm{MF}}\left(\boldsymbol{\mu}_t^{\mathbf{N}},\boldsymbol{\pi}_t\right)-\boldsymbol{\mu}_{t+1}\right|_{1}\\
&=\mathbb{E}\left|P^{\mathrm{MF}}\left(\boldsymbol{\mu}_t^{\mathbf{N}},\boldsymbol{\pi}_t\right)-P^{\mathrm{MF}}\left(\boldsymbol{\mu}_{t},\boldsymbol{\pi}_t\right)\right|_{1}\\
&\leq S_P' (\mathbb{E} |\boldsymbol{\mu}_t^{\mathbf{N}}-\boldsymbol{\mu}_{t}|_{1}) + S_P'' (\mathbb{E} |\boldsymbol{\mu}_t^{\mathbf{N}}[\mathcal{X}]-\boldsymbol{\mu}_{t}[\mathcal{X}]|_{1})	
\end{split}
\end{align}

Therefore, $(\ref{eq46})$ can be rewritten as,
\begin{align}	
\label{eq_48}
\begin{split}
S_R'\mathbb{E}&\left|\boldsymbol{\mu}_{t+1}^{\mathbf{N}}-\boldsymbol{\mu}_{t+1}\right|_{1}+S_R''\mathbb{E}\left|\boldsymbol{\mu}_{t+1}^{\mathbf{N}}[\mathcal{X}]-\boldsymbol{\mu}_{t+1}[\mathcal{X}]\right|_{1}\\&\leq C_P\left[\sqrt{|\mathcal{X}|}+\sqrt{|\mathcal{U}|}\right] \left[\dfrac{S_R'}{N_{\mathrm{pop}}}\left(\sum_{k\in[K]}\sqrt{N_k}\right)+\dfrac{S_R''}{\sqrt{N_{\mathrm{pop}}}}\right]\\
&\hspace{1cm}+S_R	\Big(S_P'\mathbb{E}\left|\boldsymbol{\mu}_t^{\mathbf{N}}-\boldsymbol{\mu}_t\right|_{1}+S_P'\mathbb{E}\left|\boldsymbol{\mu}_t^{\mathbf{N}}[\mathcal{X}]-\boldsymbol{\mu}_t[\mathcal{X}]\right|_{1}\Big)
\end{split}
\end{align}

where $S_R=S_R'+S_R''$. Similarly, one can show that,
\begin{align}
\label{eq_49}
\begin{split}
S_P'&\mathbb{E}\left|\boldsymbol{\mu}_{t+1}^{\mathbf{N}}-\boldsymbol{\mu}_{t+1}\right|_{1}+S_P''\mathbb{E}\left|\boldsymbol{\mu}_{t+1}^{\mathbf{N}}[\mathcal{X}]-\boldsymbol{\mu}_{t+1}[\mathcal{X}]\right|_{1}\\&\leq C_P\left[\sqrt{|\mathcal{X}|}+\sqrt{|\mathcal{U}|}\right] \left[\dfrac{S_P'}{N_{\mathrm{pop}}}\left(\sum_{k\in[K]}\sqrt{N_k}\right)+\dfrac{S_P''}{\sqrt{N_{\mathrm{pop}}}}\right]\\
&\hspace{1cm}+S_P	\Big(S_P'\mathbb{E}\left|\boldsymbol{\mu}_t^{\mathbf{N}}-\boldsymbol{\mu}_t\right|_{1}+S_P'\mathbb{E}\left|\boldsymbol{\mu}_t^{\mathbf{N}}[\mathcal{X}]-\boldsymbol{\mu}_t[\mathcal{X}]\right|_{1}\Big)
\end{split}
\end{align}

Recall that $\boldsymbol{\mu}_0^{\mathbf{N}}=\boldsymbol{\mu}_0$. Combining the above results, we therefore obtain,
\begin{align}
\label{eq_50}
\begin{split}
&S_R'\mathbb{E}\left|\boldsymbol{\mu}_{t+1}^{\mathbf{N}}-\boldsymbol{\mu}_{t+1}\right|_{1}+S_R''\mathbb{E}\left|\boldsymbol{\mu}_{t+1}^{\mathbf{N}}[\mathcal{X}]-\boldsymbol{\mu}_{t+1}[\mathcal{X}]\right|_{1}\\&\leq C_P\left[\sqrt{|\mathcal{X}|}+\sqrt{|\mathcal{U}|}\right]\Bigg\lbrace \left[\dfrac{S_R'}{N_{\mathrm{pop}}}\left(\sum_{k\in[K]}\sqrt{N_k}\right)+\dfrac{S_R''}{\sqrt{N_{\mathrm{pop}}}}\right]\\
&\hspace{3cm}+S_R\left[\dfrac{S_P'}{N_{\mathrm{pop}}}\left(\sum_{k\in[K]}\sqrt{N_k}\right)+\dfrac{S_P''}{\sqrt{N_{\mathrm{pop}}}}\right]\left(\dfrac{S_P^t-1}{S_P-1}\right)\Bigg\rbrace
\end{split}
\end{align}

Clearly, $J_2$ is upper bounded as follows,
\begin{align*}
J_2&\leq C_P\left[\sqrt{|\mathcal{X}|}+\sqrt{|\mathcal{U}|}\right]\left(\dfrac{\gamma}{1-\gamma}\right) \left[\dfrac{S_R'}{N_{\mathrm{pop}}}\left(\sum_{k\in[K]}\sqrt{N_k}\right)+\dfrac{S_R''}{\sqrt{N_{\mathrm{pop}}}}\right]\\
&+ C_P\left(\dfrac{S_R}{S_P-1}\right)\left[\sqrt{|\mathcal{X}|}+\sqrt{|\mathcal{U}|}\right] \left[\dfrac{S_P'}{N_{\mathrm{pop}}}\left(\sum_{k\in[K]}\sqrt{N_k}\right)+\dfrac{S_P''}{\sqrt{N_{\mathrm{pop}}}}\right]\left(\dfrac{\gamma}{1-\gamma S_P}-\dfrac{\gamma}{1-\gamma}\right)
\end{align*}

This completes the proof of the Theorem. 

\section{Proof of Lemma \ref{lemma1}}
\label{app_proof_lemma_2}

The following chain of inequalities hold true.
\begin{align*}
\begin{split}
|\nu^{\mathrm{MF}}(\boldsymbol{\mu}&,\boldsymbol{\pi})-\nu^{\mathrm{MF}}(\boldsymbol{\mu}',\boldsymbol{\pi})|_{1}\\
&=\sum_{k\in[K]}|\nu_k^{\mathrm{MF}}(\boldsymbol{\mu},\boldsymbol{\pi})-\nu_k^{\mathrm{MF}}(\boldsymbol{\mu}',\boldsymbol{\pi})|_{1}\\
&= \sum_{k\in[K]}\left|\sum_{x\in\mathcal{X}}\boldsymbol{\mu}(x,k)\pi_k(x,\boldsymbol{\mu})-\sum_{x\in\mathcal{X}}\boldsymbol{\mu}'(x,k)\pi_k(x,\boldsymbol{\mu}')\right|_{1}\\
&= \sum_{k\in[K]}\sum_{u\in\mathcal{U}}\left|\sum_{x\in\mathcal{X}}\boldsymbol{\mu}(x,k)\pi_k(x,\boldsymbol{\mu})(u)-\sum_{x\in\mathcal{X}}\boldsymbol{\mu}'(x,k)\pi_k(x,\boldsymbol{\mu}')(u)\right|\\
&\overset{}{\leq} \sum_{k\in[K]}\sum_{u\in\mathcal{U}}\sum_{x\in\mathcal{X}}\left|\boldsymbol{\mu}(x,k)\pi_k(x,\boldsymbol{\mu})(u)-\boldsymbol{\mu}'(x,k)\pi_k(x,\boldsymbol{\mu}')(u)\right|\\
&\overset{}{\leq} \sum_{k\in[K]}\sum_{x\in\mathcal{X}}|\boldsymbol{\mu}(x,k)-\boldsymbol{\mu}'(x,k)|\sum_{u\in\mathcal{U}}\pi_k(x,\boldsymbol{\mu})(u)\\
&\hspace{2cm}+\sum_{k\in[K]}\sum_{x\in\mathcal{X}}\boldsymbol{\mu}'(x,k)\sum_{u\in\mathcal{U}}|\pi_k(x,\boldsymbol{\mu})(u)-\pi_k(x,\boldsymbol{\mu}')(u)|\\
&\overset{(a)}{\leq} \sum_{k\in[K]}\sum_{x\in\mathcal{X}}|\boldsymbol{\mu}(x,k)-\boldsymbol{\mu}'(x,k)|+L_Q|\boldsymbol{\mu}-\boldsymbol{\mu}'|_{1}\sum_{k\in[K]}\sum_{x\in\mathcal{X}}\boldsymbol{\mu}'(x,k)\\
&\overset{(b)}{=} \left(1+L_Q\right)|\boldsymbol{\mu}-\boldsymbol{\mu}'|_1
\end{split}
\end{align*}

Inequality (a) follows from Assumption \ref{assumptions} and the fact that $\pi_k(x,\boldsymbol{\mu})$ is a  distribution. Finally, equality (b) uses  the fact that $\boldsymbol{\mu}'$ is a distribution. This concludes the result.

\section{Proof of Lemma \ref{lemma2}}

Note that,
\begin{align*}
\sum_{k\in[K]}&|	r_k^{\mathrm{MF}}(\boldsymbol{\mu},\boldsymbol{\pi})-	r_k^{\mathrm{MF}}(\boldsymbol{\mu}',\boldsymbol{\pi})|\\
&\leq  \sum_{k\in[K]}\sum_{x\in\mathcal{X}} \sum_{u\in\mathcal{U}} |r_k(x,u,\boldsymbol{\mu},\nu^{\mathrm{MF}}(\boldsymbol{\mu},\boldsymbol{\pi}))-r_k(x,u,\boldsymbol{\mu}',\nu^{\mathrm{MF}}(\boldsymbol{\mu}',\boldsymbol{\pi}))|\times\boldsymbol{\mu}(x,k)\pi_k(x,\boldsymbol{\mu})(u)\\
&+\sum_{k\in[K]} \sum_{x\in\mathcal{X}} \sum_{u\in\mathcal{U}} |r_k(x,u,\boldsymbol{\mu}',\nu^{\mathrm{MF}}(\boldsymbol{\mu}',\boldsymbol{\pi}))|\times|\boldsymbol{\mu}(x,k)\pi_k(x,\boldsymbol{\mu})(u) - \boldsymbol{\mu}'(x,k)\pi_k(x,\boldsymbol{\mu}')(u)|
\end{align*}

Utilising Assumption \ref{ass_1}(c), and the facts that $\boldsymbol{\mu}$, $\pi_k(x,\boldsymbol{\mu})$ are probability distributions, the first term can be upper bounded by the following expression,
\begin{align*} &L_R\left(|\boldsymbol{\mu}-\boldsymbol{\mu}'|_{1} + |\nu^{\mathrm{MF}}(\boldsymbol{\mu},\boldsymbol{\pi})-\nu^{\mathrm{MF}}(\boldsymbol{\mu}',\boldsymbol{\pi})|_{1}\right)
\leq L_R \left[1+(1+L_Q)\right]|\boldsymbol{\mu}-\boldsymbol{\mu}'|_{1}
\end{align*}

Lemma \ref{lemma1} is applied to derive the above inequality. Utilising Assumption \ref{ass_1}(b), the second term can be upper bounded by the following quantity:
\begin{align*}
M_R\sum_{k\in[K]} \sum_{x\in\mathcal{X}} \sum_{u\in\mathcal{U}} |\boldsymbol{\mu}(x,k)\pi_k(x,\boldsymbol{\mu})(u) - \boldsymbol{\mu}'(x,k)\pi_k(x,\boldsymbol{\mu}')(u)|\overset{(a)}{\leq} M_R(1+L_Q)|\boldsymbol{\mu}-\boldsymbol{\mu}'|_{1}
\end{align*}

Inequality (a) can be proved using identical arguments as used in Lemma \ref{lemma1}. 
This concludes the result.

\section{Proof of Lemma \ref{lemma3}}

Note that,
\begin{align*}
|	P^{\mathrm{MF}}&(\boldsymbol{\mu},\boldsymbol{\pi})-	P^{\mathrm{MF}}(\boldsymbol{\mu}',\boldsymbol{\pi})|_1=\sum_{k\in[K]}|	P_k^{\mathrm{MF}}(\boldsymbol{\mu},\boldsymbol{\pi})-	P_k^{\mathrm{MF}}(\boldsymbol{\mu}',\boldsymbol{\pi})|_1\\
&\leq  \sum_{k\in[K]}\sum_{x\in\mathcal{X}} \sum_{u\in\mathcal{U}} |P_k(x,u,\boldsymbol{\mu},\nu^{\mathrm{MF}}(\boldsymbol{\mu},\boldsymbol{\pi}))-P_k(x,u,\boldsymbol{\mu}',\nu^{\mathrm{MF}}(\boldsymbol{\mu}',\boldsymbol{\pi}))|_1\times\boldsymbol{\mu}(x,k)\pi_k(x,\boldsymbol{\mu})(u)\\
&+\sum_{k\in[K]} \sum_{x\in\mathcal{X}} \sum_{u\in\mathcal{U}} |P_k(x,u,\boldsymbol{\mu}',\nu^{\mathrm{MF}}(\boldsymbol{\mu}',\boldsymbol{\pi}))|_1\times|\boldsymbol{\mu}(x,k)\pi_k(x,\boldsymbol{\mu})(u) - \boldsymbol{\mu}'(x,k)\pi_k(x,\boldsymbol{\mu}')(u)|
\end{align*}

Utilising Assumption \ref{ass_1}(d), and the facts that $\boldsymbol{\mu}$, $\pi_k(x,\boldsymbol{\mu})$ are probability distributions, the first term can be upper bounded by the following expression,
\begin{align*} &L_P\left(|\boldsymbol{\mu}-\boldsymbol{\mu}'|_{1} + |\nu^{\mathrm{MF}}(\boldsymbol{\mu},\boldsymbol{\pi})-\nu^{\mathrm{MF}}(\boldsymbol{\mu}',\boldsymbol{\pi})|_{1}\right)
\leq L_P \left[1+(1+L_Q)\right]|\boldsymbol{\mu}-\boldsymbol{\mu}'|_{1}
\end{align*}

Lemma \ref{lemma1} is applied to derive the above inequality. Note that, $|P_k(x,u,\boldsymbol{\mu}',\nu(\boldsymbol{\mu}',\boldsymbol{\pi}))|_1=1$. Therefore, the second term can be bounded by the following quantity.
\begin{align*}
\sum_{k\in[K]} \sum_{x\in\mathcal{X}} \sum_{u\in\mathcal{U}} |\boldsymbol{\mu}(x,k)\pi_k(x,\boldsymbol{\mu})(u) - \boldsymbol{\mu}'(x,k)\pi_k(x,\boldsymbol{\mu}')(u)|\overset{(a)}{\leq} (1+L_Q)|\boldsymbol{\mu}-\boldsymbol{\mu}'|_{1}
\end{align*}

Inequality (a) can be proved using identical arguments as used in Lemma \ref{lemma1}. 
This concludes the result.

\section{Proof of Lemma \ref{simple_lemma}}

Let, $Y_{m,n}\triangleq X_{m,n}-\mathbb{E}[X_{m,n}]$, $\forall m\in [M]$, $\forall m\in[N]$. We need the following results to prove Lemma \ref{simple_lemma}. 

\begin{proposition}
$\forall m\in[M]$, $\forall n\in[N]$, $\mathbb{E}[Y_{m,n}^2]\leq \mathbb{E}[X_{m,n}]$.
	\label{prop_1}
\end{proposition}
\begin{proof}
For random variables $X_{m,n}\in[0,1]$, note that, 
\begin{align*}
\mathbb{E}[Y_{m,n}^2]&=E[X_{m,n}^2]-(\mathbb{E}[X_{m,n}])^2\\
&\leq E[X_{m,n}]-(\mathbb{E}[X_{m,n}])^2\leq \mathbb{E}[X_{m,n}]
\end{align*}
\end{proof}

\begin{proposition}
	\label{prop_2}
$\forall m\in[M]$, $\mathbb{E}[\sum_{n=1}^NC_{m,n}Y_{m,n}]^2\leq C^2\sum_{n=1}^N\mathbb{E}[Y_{m,n}^2]$.
\end{proposition}
\begin{proof}
	Using the independence of $Y_{m,n}$'s, we deduce, $\forall m\in [M]$,
	\begin{align*}
    &\mathbb{E}\left[\sum_{n=1}^NC_{m,n}Y_{m,n}\right]^2 \\
    &=   \mathbb{E}\left[\sum_{n_1=1}^N \sum_{n_2=1}^NC_{m,n_1}C_{m,n_2}Y_{m,n_1}Y_{m,n_2}\right]\\
    &= \sum_{n=1}^NC_{m,n}^2\mathbb{E}\left[Y_{m,n}^2\right] +2 \sum_{n_1=1}^{N}\sum_{n_2>n_1}^N C_{m,n_1}C_{m,n_2}\mathbb{E}[Y_{m,n_1}]\mathbb{E}[Y_{m,n_2}]\\
    &\overset{(a)}{=}\sum_{n=1}^NC_{m,n}^2\mathbb{E}\left[Y_{m,n}^2\right]\\
%    &\leq\sum_{n=1}^N\left(\sum_{s=1}^S\Big|C_{m,n}(s)\Big|\right)^2\mathbb{E}\left[Y_{m,n}^2\right]\\
%    &=\sum_{n=1}^N\Big|C_{m,n}\Big|_1^2\mathbb{E}\left[Y_{m,n}^2\right]\\
    &\leq C^2\sum_{n=1}^N \mathbb{E}[Y^2_{m,n}]
	\end{align*}
	Equality (a) uses the fact that $\mathbb{E}[Y_{m,n}]=0$, $\forall m\in [M],\forall n\in [N]$. 
\end{proof}

We are now ready to prove Lemma \ref{simple_lemma}. Note that,
\begin{align*}
   &\sum_{m=1}^M\mathbb{E}\left|\sum_{n=1}^N C_{m,n}Y_{m,n}\right|\\
   &\overset{(a)}{\leq}\sqrt{M} \left\lbrace\sum_{m=1}^M\mathbb{E}\left[\sum_{n=1}^N C_{m,n}Y_{m,n}\right]^2\right\rbrace^{\frac{1}{2}}\\
   &\overset{(b)}{\leq}C\sqrt{M} \left\lbrace\sum_{m=1}^M\sum_{n=1}^N\mathbb{E}\left[Y_{m,n}^2\right]\right\rbrace^{\frac{1}{2}}\\
   &
      \overset{(c)}{\leq}C\sqrt{M} \left\lbrace\sum_{n=1}^N\sum_{m=1}^M\mathbb{E}\left[X_{m,n}\right]\right\rbrace^{\frac{1}{2}}\\ &=C\sqrt{MN}
\end{align*}

Result (a) is a consequence of Cauchy-Schwarz inequality, and (b), (c) follow from Proposition \ref{prop_2}, and {\ref{prop_1}} respectively. This concludes the result.

\section{Proof of Lemma \ref{lemma_nu}}

Using the definition of $L_{1}$-norm, we get:
\begin{align*}\mathbb{E}\left|\boldsymbol{\nu}_t^{\mathbf{N}}-\nu^{\mathrm{MF}}(\boldsymbol{\mu}_t^{\mathbf{N}},\boldsymbol{\pi}_t)\right|_1 &=  \sum_{k\in[K]}\sum_{u\in\mathcal{U}}\mathbb{E}\left|\boldsymbol{\nu}_t^{\mathbf{N}}(u,k)-\nu^{\mathrm{MF}}(\boldsymbol{\mu}_t^{\mathbf{N}},\boldsymbol{\pi}_t)(u,k)\right|\\
&=\dfrac{1}{N_{\mathrm{pop}}}\sum_{k\in[K]} \sum_{u\in\mathcal{U}}\mathbb{E}\left|\sum_{j=1}^{N_k}\delta(u_{j,k}^{t,\mathbf{N}}=u)-\sum_{j=1}^{N_k}\pi_k^t(x_{j,k}^{t,\mathbf{N}},\boldsymbol{\mu}_t^{\mathbf{N}})(u)\right|
\end{align*}

Recall from Observation \ref{obs1} that, the random variables $u_{j,k}^{t,\mathbf{N}}$'s are independent conditioned on $\mathbf{x}_t^{\mathbf{N}}$. Also, it is easy to check the following relations, 
\begin{align*}
&	\mathbb{E}\left[\delta(u_{j,k}^{t,\mathbf{N}}=u)\Big|\mathbf{x}_t^{\mathbf{N}}\right] = \pi_k^t(x_{j,k}^{t,\mathbf{N}},\boldsymbol{\mu}_t^{\mathbf{N}})(u)~\text{and}~\sum_{u\in\mathcal{U}}\mathbb{E}\left[\delta(u_{j,k}^{t,\mathbf{N}}=u)\Big|\mathbf{x}_t^{\mathbf{N}}\right] = 1
\end{align*}

Using Lemma \ref{simple_lemma}, we therefore conclude:
\begin{align*}
\mathbb{E}\left|\boldsymbol{\nu}_t^{\mathbf{N}}-\nu^{\mathrm{MF}}(\boldsymbol{\mu}_t^{\mathbf{N}},\boldsymbol{\pi}_t)\right|_1\leq \dfrac{1}{N_{\mathrm{pop}}}\left(\sum_{k\in[K]}\sqrt{N_k}\right)\sqrt{|\mathcal{U}|}
\end{align*}

\section{Proof of Lemma \ref{lemma_r}}

Note that,
\begin{align*}
r_k^{\mathrm{MF}}(\boldsymbol{\mu}_t^{\mathbf{N}},\boldsymbol{\pi}_t)&=\dfrac{1}{N_{\mathrm{pop}}}\sum_{j=1}^{N_k}\sum_{x\in\mathcal{X}}\sum_{u\in\mathcal{U}}r_k(x,u,\boldsymbol{\mu}_t^{\mathbf{N}},\nu^{\mathrm{MF}}(\boldsymbol{\mu}_t^{\mathbf{N}},\boldsymbol{\pi}_{t}))\times\pi_k^t(x,\boldsymbol{\mu}_t^{\mathbf{N}})(u)\delta(x_{j,k}^{t,\mathbf{N}}=x)\\
&=\dfrac{1}{N_{\mathrm{pop}}}\sum_{j=1}^{N_k}\sum_{u\in\mathcal{U}}r_k(x_{j,k}^{t,\mathbf{N}},u,\boldsymbol{\mu}_t^{\mathbf{N}},\nu^{\mathrm{MF}}(\boldsymbol{\mu}_t^{\mathbf{N}},\boldsymbol{\pi}_{t}))\times\pi_k^t(x_{j,k}^{t,\mathbf{N}},\boldsymbol{\mu}_t^{\mathbf{N}})(u)
\end{align*}

We can upper bound the $\mathrm{LHS}$ of $(\ref{eq_34_app})$ by $ J_1+J_2$ where $J_1$ is defined as follows.
\begin{align*}
J_1&\triangleq\dfrac{1}{N_{\mathrm{pop}}}\mathbb{E}\left|\sum_{k\in[K]}\sum_{j=1}^{N_k}
r_k(x_{j,k}^{t,\mathbf{N}},u_{j,k}^{t,\mathbf{N}},\boldsymbol{\mu}_t^{\mathbf{N}},\boldsymbol{\nu}_t^{\mathbf{N}})-r_k(x_{j,k}^{t,\mathbf{N}},u_{j,k}^{t,\mathbf{N}},\boldsymbol{\mu}_t^{\mathbf{N}},\nu^{\mathrm{MF}}(\boldsymbol{\mu}_t^{\mathbf{N}},\boldsymbol{\pi}_{t}))
\right|\\
&\leq\dfrac{1}{N_{\mathrm{pop}}}\mathbb{E}\sum_{k\in[K]}\sum_{j=1}^{N_k}\left|
r_k(x_{j,k}^{t,\mathbf{N}},u_{j,k}^{t,\mathbf{N}},\boldsymbol{\mu}_t^{\mathbf{N}},\boldsymbol{\nu}_t^{\mathbf{N}})-r_k(x_{j,k}^{t,\mathbf{N}},u_{j,k}^{t,\mathbf{N}},\boldsymbol{\mu}_t^{\mathbf{N}},\nu^{\mathrm{MF}}(\boldsymbol{\mu}_t^{\mathbf{N}},\boldsymbol{\pi}_{t}))
\right|\\
&\overset{(a)}{\leq}\dfrac{1}{N_{\mathrm{pop}}}\sum_{k\in[K]}\sum_{j=1}^{N_k}L_R \mathbb{E}\left|\boldsymbol{\nu}_t^{\mathbf{N}}-\nu^{\mathrm{MF}}(\boldsymbol{\mu}_t^{\mathbf{N}},\boldsymbol{\pi}_t)\right|_{1}\overset{(b)}{\leq} \dfrac{L_R}{N_{\mathrm{pop}}}\left(\sum_{k\in[K]}\sqrt{N_k}\right)\sqrt{|\mathcal{U}|}
\end{align*}

Inequality (a) follows from Assumption \ref{ass_1}(c) whereas inequality (b) follows from Lemma \ref{lemma_nu}. The term $J_2$ is defined below.
\begin{align*}
J_2&\triangleq\mathbb{E}\left|\dfrac{1}{N_{\mathrm{pop}}}\sum_{k\in[K]} \sum_{j=1}^{N_k}\sum_{u\in\mathcal{U}}r_k(x_{j,k}^{t,\mathbf{N}},u,\boldsymbol{\mu}_t^{\mathbf{N}},\nu^{\mathrm{MF}}(\boldsymbol{\mu}_t^{\mathbf{N}},\boldsymbol{\pi}_{t}))\times\left[\delta(u_{j,k}^{t,\mathbf{N}}=u)-\pi_k^t(x_{j,k}^{t,\mathbf{N}},\boldsymbol{\mu}_t^{\mathbf{N}})(u)\right]\right|\\
&\leq\dfrac{1}{N_{\mathrm{pop}}}\sum_{u\in\mathcal{U}}\mathbb{E}\left|\sum_{k\in[K]} \sum_{j=1}^{N_k}r_k(x_{j,k}^{t,\mathbf{N}},u,\boldsymbol{\mu}_t^{\mathbf{N}},\nu^{\mathrm{MF}}(\boldsymbol{\mu}_t^{\mathbf{N}},\boldsymbol{\pi}_{t}))\times\left[\delta(u_{j,k}^{t,\mathbf{N}}=u)-\pi_k^t(x_{j,k}^{t,\mathbf{N}},\boldsymbol{\mu}_t^{\mathbf{N}})(u)\right]\right| 
\end{align*}

Recall from Observation \ref{obs1} that $u_{j,k}^{t,\mathbf{N}}$'s are independent conditioned on $\mathbf{x}_t^{\mathbf{N}}$. Therefore, $\forall u\in \mathcal{U}$, $\delta(u_{j,k}^{t,\mathbf{N}}=u)$'s are independent, conditioned on $\mathbf{x}_t^{\mathbf{N}}$. Moreover,
\begin{align*}
&\mathbb{E}\left[\delta(u_{j,k}^{t,\mathbf{N}}=u)\Big|\mathbf{x}_t^{\mathbf{N}}\right]=\pi_k^t(x_{j,k}^{t,\mathbf{N}},\boldsymbol{\mu}_t^{\mathbf{N}})(u),~\forall u\in\mathcal{U},~~\\
&
\sum_{u\in\mathcal{U}}\mathbb{E}\left[\delta(u_{j,k}^{t,\mathbf{N}}=u)\Big|\mathbf{x}_t^{\mathbf{N}}\right]=1\\
\text{and~}& |r_k(x,u,\boldsymbol{\mu}_t^{\mathbf{N}},\nu^{\mathrm{MF}}(\boldsymbol{\mu}_t^{\mathbf{N}},\boldsymbol{\pi}_{t}))| \leq M_R, ~\forall x\in\mathcal{X}, \forall u\in\mathcal{U}
\end{align*}

Using Lemma \ref{simple_lemma}, we therefore get,
\begin{align*}
J_2&\leq \dfrac{M_R}{\sqrt{N}}\sqrt{|\mathcal{U}|}\leq\dfrac{M_R}{N}\left(\sum_{k\in[K]}\sqrt{N_k}\right)\sqrt{|\mathcal{U}|}
\end{align*}

This concludes the result.

\section{Proof of Lemma \ref{lemma_mu}}
\label{section_proof_lemma_8}
Note that the $\mathrm{LHS}$ of $(\ref{eq_35_app})$ can be upper bounded as follows.
\begin{align*}
\mathrm{LHS}&=\sum_{k\in[K]}\sum_{x\in\mathcal{X}}\mathbb{E}\left|\boldsymbol{\mu}_{t+1}^{\mathbf{N}}(x,k)-P_k^{\mathrm{MF}}\left(\boldsymbol{\mu}_t^{\mathbf{N}},\boldsymbol{\pi}_t\right)(x)\right|\\
&= \sum_{k\in[K]}  \sum_{x\in\mathcal{X}}\dfrac{1}{N_{\mathrm{pop}}}\mathbb{E}\Bigg|\sum_{j=1}^{N_k}\delta(x_{j,k}^{t+1,\mathbf{N}}=x)\\
&\hspace{3.5cm}-\sum_{j=1}^{N_k}\sum_{u\in\mathcal{U}}\pi_k^t(x_{j,k}^{t,\mathbf{N}},\boldsymbol{\mu}_t^{\mathbf{N}})(u) P_k(x_{j,k}^{t,\mathbf{N}},u,\boldsymbol{\mu}_t^{\mathbf{N}},\nu^{\mathrm{MF}}(\boldsymbol{\mu}_t^{\mathbf{N}},\boldsymbol{\pi}_t))(x)\Bigg|\\
&\leq  J_1 + J_2 + J_3
\end{align*}

 The first term, $J_1$ is defined as follows:
\begin{align*}
J_1&\triangleq \sum_{k\in[K]} \sum_{x\in\mathcal{X}}\dfrac{1}{N_{\mathrm{pop}}} \mathbb{E}\left|\sum_{j=1}^{N_k}\delta(x_{j,k}^{t+1,\mathbf{N}}=x)-\sum_{j=1}^{N_k}P_k(x_{j,k}^{t,\mathbf{N}},u_{j,k}^{t,\mathbf{N}},\boldsymbol{\mu}_t^{\mathbf{N}},\boldsymbol{\nu}_t^{\mathbf{N}})(x)\right|
\end{align*}

Recall from observation \ref{obs2} that $x_{j,k}^{t+1,\mathbf{N}}$'s are independent conditional on $\mathbf{x}_t^{\mathbf{N}},\mathbf{u}_t^{\mathbf{N}}$. Also,
\begin{align*}
&\mathbb{E}\left[\delta\left(x_{j,k}^{t+1,\mathbf{N}}=x\right)\big|\mathbf{x}_t^{\mathbf{N}},\mathbf{u}_t^{\mathbf{N}}\right]=P_k(x_{j,k}^{t,\mathbf{N}},u_{j,k}^{t,\mathbf{N}},\boldsymbol{\mu}_t^{\mathbf{N}},\boldsymbol{\nu}_t^{\mathbf{N}})(x),~\forall x\in\mathcal{X}\\
\text{and~}&\sum_{x\in\mathcal{X}}\mathbb{E}\left[\delta\left(x_{j,k}^{t+1,\mathbf{N}}=x\right)\big|\mathbf{x}_t^{\mathbf{N}},\mathbf{u}_t^{\mathbf{N}}\right]=1
\end{align*}

Applying  Lemma \ref{simple_lemma}, we can conclude that,
\begin{align*}
J_1&\leq \dfrac{1}{N_{\mathrm{pop}}}\left(\sum_{k\in[K]}\sqrt{N_k}\right)\sqrt{|\mathcal{X}|}\leq \dfrac{1}{N_{\mathrm{pop}}}\left(\sum_{k\in[K]}\sqrt{N_k}\right)\left[\sqrt{|\mathcal{X}|}+\sqrt{|\mathcal{U}|}\right]
\end{align*}

The second term, $J_2$ is defined as follows,
\begin{align*}
&J_2\triangleq\sum_{k\in[K]}\sum_{x\in\mathcal{X}}\dfrac{1}{N_{\mathrm{pop}}}\mathbb{E}\left|\sum_{j=1}^{N_k}P_k(x_{j,k}^{t,\mathbf{N}},u_{j,k}^{t,\mathbf{N}},\boldsymbol{\mu}_t^{\mathbf{N}},\boldsymbol{\nu}_t^{\mathbf{N}})(x)- P_k(x_{j,k}^{t,\mathbf{N}},u_{j,k}^{t,\mathbf{N}},\boldsymbol{\mu}_t^{\mathbf{N}},\nu^{\mathrm{MF}}(\boldsymbol{\mu}_t^{\mathbf{N}},\boldsymbol{\pi}_t))(x)\right|\\
&\overset{}{\leq} \dfrac{1}{N_{\mathrm{pop}}}\sum_{k\in[K]}\sum_{j=1}^{N_k}\sum_{x\in\mathcal{X}}\mathbb{E}\left|P_k(x_{j,k}^{t,\mathbf{N}},u_{j,k}^{t,\mathbf{N}},\boldsymbol{\mu}_t^{\mathbf{N}},\boldsymbol{\nu}_t^{\mathbf{N}})(x)- P_k(x_{j,k}^{t,\mathbf{N}},u_{j,k}^{t,\mathbf{N}},\boldsymbol{\mu}_t^{\mathbf{N}},\nu^{\mathrm{MF}}(\boldsymbol{\mu}_t^{\mathbf{N}},\boldsymbol{\pi}_t))(x)\right|\\
&\overset{(a)}{\leq} L_P\left|\boldsymbol{\nu}_t^{\mathbf{N}}-\nu^{\mathrm{MF}}(\boldsymbol{\mu}_t^{\mathbf{N}},\boldsymbol{\pi}_t)\right|_{1}\\
&\overset{(b)}{\leq} \dfrac{L_P}{N_{\mathrm{pop}}}\left(\sum_{k\in[K]}\sqrt{N_k}\right)\sqrt{|\mathcal{U}|}\leq \dfrac{L_P}{N_{\mathrm{pop}}}\left(\sum_{k\in[K]}\sqrt{N_k}\right)\left[\sqrt{|\mathcal{X}|}+\sqrt{|\mathcal{U}|}\right]
\end{align*}

Relation (a) is a consequence of Assumption \ref{ass_1}(d) and the inequality (b) follows from Lemma \ref{lemma_nu}. Finally, 
\begin{align*}
&J_3=\sum_{k\in[K]}\sum_{x\in\mathcal{X}}\mathbb{E}\left|\dfrac{1}{N_{\mathrm{pop}}}\sum_{j=1}^{N_k}\left[P_k(x_{j,k}^{t,\mathbf{N}},u_{j,k}^{t,\mathbf{N}},\boldsymbol{\mu}_t^{\mathbf{N}},\nu^{\mathrm{MF}}(\boldsymbol{\mu}_t^{\mathbf{N}},\boldsymbol{\pi}_t))(x)\right. \right.\\&\hspace{3cm}-\left.\left.\sum_{u\in\mathcal{U}}\pi_k^t(x_{j,k}^{t,\mathbf{N}},\boldsymbol{\mu}_t^{\mathbf{N}})(u)P_k(x_{j,k}^{t,\mathbf{N}},u,\boldsymbol{\mu}_t^{\mathbf{N}},\nu^{\mathrm{MF}}(\boldsymbol{\mu}_t^{\mathbf{N}},\boldsymbol{\pi}_t))(x)\right] \right|\\
%&\leq\sum_{k\in[K]}\sum_{x\in\mathcal{X}}\sum_{u\in\mathcal{U}}\mathbb{E}\left|\dfrac{1}{N_{\mathrm{pop}}}\sum_{j=1}^{N_k}\left[\delta(u_{j,k}^{t,\mathbf{N}}=u)-\pi_k^t(x_{j,k}^{t,\mathbf{N}},\boldsymbol{\mu}_t^{\mathbf{N}})(u)\right] P_k(x_{j,k}^{t,\mathbf{N}},u,\boldsymbol{\mu}_t^{\mathbf{N}},\nu^{\mathrm{MF}}(\boldsymbol{\mu}_t^{\mathbf{N}},\boldsymbol{\pi}_t))(x)\right|\\
&\overset{(a)}{\leq}\dfrac{1}{N_{\mathrm{pop}}}\left(\sum_{k\in[K]}\sqrt{N_k}\right)\sqrt{|\mathcal{X}|}\leq \dfrac{1}{N_{\mathrm{pop}}}\left(\sum_{k\in[K]}\sqrt{N_k}\right)\left[\sqrt{|\mathcal{X}|}+\sqrt{|\mathcal{U}|}\right]
\end{align*}

Inequality (a) is a result of Lemma \ref{simple_lemma} and the facts that $\{\mathbf{u}_{j,k}^{t,\mathbf{N}}\}_{j,k}$'s are independent conditioned on $\mathbf{x}_t^{\mathbf{N}}$ and
\begin{align*}
&\mathbb{E}\left[P_k(x_{j,k}^{t,\mathbf{N}},u_{j,k}^{t,\mathbf{N}},\boldsymbol{\mu}_t^{\mathbf{N}},\nu^{\mathrm{MF}}(\boldsymbol{\mu}_t^{\mathbf{N}},\boldsymbol{\pi}_t))(x)\Big|\mathbf{x}_t^{\mathbf{N}}\right]\\
&\hspace{2cm}=\sum_{u\in\mathcal{U}}\pi_k^t(x_{j,k}^{t,\mathbf{N}},\boldsymbol{\mu}_t^{\mathbf{N}})(u)P_k(x_{j,k}^{t,\mathbf{N}},u,\boldsymbol{\mu}_t^{\mathbf{N}},\nu^{\mathrm{MF}}(\boldsymbol{\mu}_t^{\mathbf{N}},\boldsymbol{\pi}_t))(x),~\forall x\in\mathcal{X},\\
&
\sum_{x\in\mathcal{X}}\mathbb{E}\left[P_k(x_{j,k}^{t,\mathbf{N}},u_{j,k}^{t,\mathbf{N}},\boldsymbol{\mu}_t^{\mathbf{N}},\nu^{\mathrm{MF}}(\boldsymbol{\mu}_t^{\mathbf{N}},\boldsymbol{\pi}_t))(x)\Big|\mathbf{x}_t^{\mathbf{N}}\right]=1%\\
%\text{and~}& |P_k(x,u,\boldsymbol{\mu}_t^{\mathbf{N}},\nu^{\mathrm{MF}}(\boldsymbol{\mu}_t^{\mathbf{N}},\boldsymbol{\pi}_{t}))|_1 =1, ~\forall x\in\mathcal{X},\forall u\in\mathcal{U}
\end{align*}

This concludes the result.

\section{Proof of Lemma \ref{lemma_continuity_2}}
\label{app_lemma_9}
\subsection{Proof of Proposition (a)}

Following similar line of argument as used in the proof of Lemma \ref{lemma1}, we obtain,
\begin{align*}
\begin{split}
|\bar{\nu}^{\mathrm{MF}}(\boldsymbol{\bar{\mu}}&,\boldsymbol{\bar{\pi}})-\bar{\nu}^{\mathrm{MF}}(\boldsymbol{\bar{\mu}}',\boldsymbol{\bar{\pi}})|_{1}\\
&=\sum_{k\in[K]}|\bar{\nu}_k^{\mathrm{MF}}(\boldsymbol{\bar{\mu}},\boldsymbol{\bar{\pi}})-\bar{\nu}_k^{\mathrm{MF}}(\boldsymbol{\bar{\mu}}',\boldsymbol{\bar{\pi}})|_{1}\\
&= \sum_{k\in[K]}\sum_{u\in\mathcal{U}}\left|\sum_{x\in\mathcal{X}}\boldsymbol{\bar{\mu}}(x,k)\bar{\pi}_k(x,\boldsymbol{\bar{\mu}})(u)-\sum_{x\in\mathcal{X}}\boldsymbol{\bar{\mu}}'(x,k)\bar{\pi}_k(x,\boldsymbol{\bar{\mu}}')(u)\right|\\
&\leq \sum_{k\in[K]}\sum_{u\in\mathcal{U}}\sum_{x\in\mathcal{X}}\left|\boldsymbol{\bar{\mu}}(x,k)\bar{\pi}_k(x,\boldsymbol{\bar{\mu}})(u)-\boldsymbol{\bar{\mu}}'(x,k)\bar{\pi}_k(x,\boldsymbol{\bar{\mu}}')(u)\right|\\
&\overset{}{\leq} \sum_{k\in[K]}\sum_{x\in\mathcal{X}}|\boldsymbol{\bar{\mu}}(x,k)-\boldsymbol{\bar{\mu}}'(x,k)|\sum_{u\in\mathcal{U}}\bar{\pi}_k(x,\bar{\boldsymbol{\mu}})(u)\\
&\hspace{2cm}+\sum_{k\in[K]}\sum_{x\in\mathcal{X}}\boldsymbol{\bar{\mu}}'(x,k)\sum_{u\in\mathcal{U}}\Big|\bar{\pi}_k(x,\bar{\boldsymbol{\mu}})(u)-\bar{\pi}_k(x,\bar{\boldsymbol{\mu}}')(u)\Big|\\
&\overset{(a)}{\leq} \sum_{k\in[K]}\sum_{x\in\mathcal{X}}|\boldsymbol{\bar{\mu}}(x,k)-\boldsymbol{\bar{\mu}}'(x,k)|+\bar{L}_Q|\boldsymbol{\bar{\mu}}-\boldsymbol{\bar{\mu}}'|_{1}\sum_{k\in[K]}\sum_{x\in\mathcal{X}}\boldsymbol{\bar{\mu}}'(x,k)\\
&\overset{(b)}{=} \left(1+K\bar{L}_Q\right)|\boldsymbol{\bar{\mu}}-\boldsymbol{\bar{\mu}}'|_1
\end{split}
\end{align*}

Inequality (a) follows from Assumption \ref{ass_4} and the fact that $\bar{\pi}_k(x,\boldsymbol{\bar{\mu}})$ is a distribution $\forall x\in \mathcal{X}$, $\forall k\in [K]$. Equality $(b)$ uses the fact that $\boldsymbol{\bar{\mu}}'(.,k)$ is a distribution $\forall k\in [K]$.

\subsection{Proof of Proposition (b)}

Note that,
\begin{align*}
\sum_{k\in[K]}&\theta_k|	\bar{r}_k^{\mathrm{MF}}(\boldsymbol{\bar{\mu}},\boldsymbol{\bar{\pi}})-	\bar{r}_k^{\mathrm{MF}}(\boldsymbol{\bar{\mu}}',\boldsymbol{\bar{\pi}})|\\
&\leq  \sum_{k\in[K]}\sum_{x\in\mathcal{X}} \sum_{u\in\mathcal{U}} |\bar{r}_k(x,u,\boldsymbol{\bar{\mu}},\bar{\nu}^{\mathrm{MF}}(\boldsymbol{\bar{\mu}},\boldsymbol{\bar{\pi}}))-\bar{r}_k(x,u,\boldsymbol{\bar{\mu}}',\bar{\nu}^{\mathrm{MF}}(\boldsymbol{\bar{\mu}}',\boldsymbol{\bar{\pi}}))|\times \theta_k\boldsymbol{\bar{\mu}}(x,k)\bar{\pi}_k(x,\boldsymbol{\bar{\mu}})(u)\\
&+\sum_{k\in[K]} \sum_{x\in\mathcal{X}} \sum_{u\in\mathcal{U}} |\bar{r}_k(x,u,\boldsymbol{\bar{\mu}}',\bar{\nu}^{\mathrm{MF}}(\boldsymbol{\bar{\mu}}',\boldsymbol{\bar{\pi}}))|\times\theta_k|\boldsymbol{\bar{\mu}}(x,k)\bar{\pi}_k(x,\boldsymbol{\bar{\mu}})(u) - \boldsymbol{\bar{\mu}}'(x,k)\bar{\pi}_k(x,\boldsymbol{\bar{\mu}}')(u)|
\end{align*}

Utilising Assumption \ref{ass_2}(c), and the facts that $\boldsymbol{\theta}$, $\boldsymbol{\bar{\mu}}(\cdot,k)$, $\pi_k(x,\boldsymbol{\bar{\mu}})$ are probability distributions $\forall k\in[K]$, $\forall x\in \mathcal{X}$, the first term can be upper bounded by the following expression,
\begin{align*} &\bar{L}_R\left(|\boldsymbol{\bar{\mu}}-\boldsymbol{\bar{\mu}}'|_{1} + |\bar{\nu}^{\mathrm{MF}}(\boldsymbol{\bar{\mu}},\boldsymbol{\bar{\pi}})-\bar{\nu}^{\mathrm{MF}}(\boldsymbol{\bar{\mu}}',\boldsymbol{\bar{\pi}})|_{1}\right)
\leq \bar{L}_R \left[1+(1+K\bar{L}_Q)\right]|\boldsymbol{\bar{\mu}}-\boldsymbol{\bar{\mu}}'|_{1}
\end{align*}

Proposition (a) is used to derive the above inequality. Applying assumption \ref{ass_2}(b), the second term can be upper bounded by the following quantity.
\begin{align*}
&\bar{M}_R\sum_{k\in[K]} \sum_{x\in\mathcal{X}} \sum_{u\in\mathcal{U}} \theta_k|\boldsymbol{\bar{\mu}}(x,k)\bar{\pi}_k(x,\boldsymbol{\bar{\mu}})(u) - \boldsymbol{\bar{\mu}}'(x,k)\bar{\pi}_k(x,\boldsymbol{\bar{\mu}}')(u)|\\
&\overset{}{\leq} \bar{M}_R\sum_{k\in[K]}\sum_{x\in\mathcal{X}}\theta_k|\boldsymbol{\bar{\mu}}(x,k)-\boldsymbol{\bar{\mu}}'(x,k)|\sum_{u\in\mathcal{U}}\bar{\pi}_k(x,\bar{\boldsymbol{\mu}})(u)\\
&\hspace{2cm}+\bar{M}_R\sum_{k\in[K]}\sum_{x\in\mathcal{X}}\theta_k\boldsymbol{\bar{\mu}}'(x,k)\sum_{u\in\mathcal{U}}\Big|\bar{\pi}_k(x,\bar{\boldsymbol{\mu}})(u)-\bar{\pi}_k(x,\bar{\boldsymbol{\mu}}')(u)\Big|\\
&\overset{(a)}{\leq} \bar{M}_R\sum_{k\in[K]}\sum_{x\in\mathcal{X}}\theta_k|\boldsymbol{\bar{\mu}}(x,k)-\boldsymbol{\bar{\mu}}'(x,k)|+\bar{M}_R\bar{L}_Q|\boldsymbol{\bar{\mu}}-\boldsymbol{\bar{\mu}}'|_{1}\sum_{k\in[K]}\sum_{x\in\mathcal{X}}\theta_k\boldsymbol{\bar{\mu}}'(x,k)\\
&\overset{(b)}{\leq} \bar{M}_R\sum_{k\in[K]}\sum_{x\in\mathcal{X}}|\boldsymbol{\bar{\mu}}(x,k)-\boldsymbol{\bar{\mu}}'(x,k)|+\bar{M}_R\bar{L}_Q|\boldsymbol{\bar{\mu}}-\boldsymbol{\bar{\mu}}'|_{1}\sum_{k\in[K]}\theta_k\\
&\overset{(c)}{=}\bar{M}_R(1+\bar{L}_Q)|\boldsymbol{\bar{\mu}}-\boldsymbol{\bar{\mu}}'|_1
\end{align*}

Inequality (a) follows from Assumption \ref{ass_4} and the fact that $\bar{\pi}_k(x,\boldsymbol{\bar{\mu}})$ is a distribution $\forall x\in \mathcal{X}$, $\forall k\in [K]$ while result (b) is derived from the fact that $\boldsymbol{\bar{\mu}}'(\cdot,k)$ is a distribution and $\theta_k\leq 1$, $\forall k\in [K]$. Finally, equality (c) holds because $\boldsymbol{\theta}$ is a distribution. This proves the proposition. 

\subsection{Proof of Proposition (c)}
Note that,
\begin{align*}
|	\bar{P}^{\mathrm{MF}}&(\boldsymbol{\bar{\mu}},\boldsymbol{\bar{\pi}})-	\bar{P}^{\mathrm{MF}}(\boldsymbol{\bar{\mu}}',\boldsymbol{\bar{\pi}})|_1=\sum_{k\in[K]}|	\bar{P}_k^{\mathrm{MF}}(\boldsymbol{\bar{\mu}},\boldsymbol{\bar{\pi}})-	\bar{P}_k^{\mathrm{MF}}(\boldsymbol{\bar{\mu}}',\boldsymbol{\bar{\pi}})|_1\\
&\leq  \sum_{k\in[K]}\sum_{x\in\mathcal{X}} \sum_{u\in\mathcal{U}} |\bar{P}_k(x,u,\boldsymbol{\bar{\mu}},\bar{\nu}^{\mathrm{MF}}(\boldsymbol{\bar{\mu}},\boldsymbol{\bar{\pi}}))-\bar{P}_k(x,u,\boldsymbol{\bar{\mu}}',\bar{\nu}^{\mathrm{MF}}(\boldsymbol{\bar{\mu}}',\boldsymbol{\bar{\pi}}))|_1\times\boldsymbol{\bar{\mu}}(x,k)\bar{\pi}_k(x,\boldsymbol{\bar{\mu}})(u)\\
&+\sum_{k\in[K]} \sum_{x\in\mathcal{X}} \sum_{u\in\mathcal{U}} |\bar{P}_k(x,u,\boldsymbol{\bar{\mu}}',\bar{\nu}^{\mathrm{MF}}(\boldsymbol{\bar{\mu}}',\boldsymbol{\bar{\pi}}))|_1\times|\boldsymbol{\bar{\mu}}(x,k)\bar{\pi}_k(x,\boldsymbol{\bar{\mu}})(u) - \boldsymbol{\bar{\mu}}'(x,k)\bar{\pi}_k(x,\boldsymbol{\bar{\mu}}')(u)|
\end{align*}

Using Assumption \ref{ass_2}(d) and the facts that $\boldsymbol{\bar{\mu}}(\cdot,k)$, $\bar{\pi}_k(x,\boldsymbol{\bar{\mu}})$ are probability distributions $\forall x\in \mathcal{X}$, $\forall k\in [K]$, the first term can be upper bounded by the following expression,
\begin{align*} &K\bar{L}_P\left(|\boldsymbol{\bar{\mu}}-\boldsymbol{\bar{\mu}}'|_{1} + |\bar{\nu}^{\mathrm{MF}}(\boldsymbol{\bar{\mu}},\boldsymbol{\bar{\pi}})-\bar{\nu}^{\mathrm{MF}}(\boldsymbol{\bar{\mu}}',\boldsymbol{\bar{\pi}})|_{1}\right)
\leq K\bar{L}_P \left[1+(1+K\bar{L}_Q)\right]|\boldsymbol{\bar{\mu}}-\boldsymbol{\bar{\mu}}'|_{1}
\end{align*}

Proposition (a) is applied to derive the above inequality. Note that, $|\bar{P}_k(x,u,\boldsymbol{\bar{\mu}}',\bar{\nu}(\boldsymbol{\bar{\mu}}',\boldsymbol{\bar{\pi}}))|_1=1$. Therefore, the second term can be upper bounded by the following quantity.
\begin{align*}
\sum_{k\in[K]} \sum_{x\in\mathcal{X}} \sum_{u\in\mathcal{U}} |\boldsymbol{\bar{\mu}}(x,k)\bar{\pi}_k(x,\boldsymbol{\bar{\mu}})(u) - \boldsymbol{\bar{\mu}}'(x,k)\bar{\pi}_k(x,\boldsymbol{\bar{\mu}}')(u)|\overset{(a)}{\leq}(1+K\bar{L}_Q)|\boldsymbol{\bar{\mu}}-\boldsymbol{\bar{\mu}}'|_{1}
\end{align*}

Inequality (a) can be established by following identical arguments as used in Proposition (a). This concludes the result.

\section{Proof of Lemma \ref{lemma_approx_2}}
\label{app_lemma_10}

\subsection{Proof of Proposition (a)}

Using the definition of $L_{1}$-norm, we get:
\begin{align*}
\mathbb{E}\left|\boldsymbol{\bar{\nu}}_t^{\mathbf{N}}-\bar{\nu}^{\mathrm{MF}}(\boldsymbol{\bar{\mu}}_t^{\mathbf{N}},\boldsymbol{\bar{\pi}}_t)\right|_1 &=  \sum_{k\in[K]}\sum_{u\in\mathcal{U}}\mathbb{E}\left|\boldsymbol{\bar{\nu}}_t^{\mathbf{N}}(u,k)-\bar{\nu}^{\mathrm{MF}}(\boldsymbol{\bar{\mu}}_t^{\mathbf{N}},\boldsymbol{\bar{\pi}}_t)(u,k)\right|\\
&=\sum_{k\in[K]} \sum_{u\in\mathcal{U}}\dfrac{1}{N_k}\mathbb{E}\left|\sum_{j=1}^{N_k}\delta(u_{j,k}^{t,\mathbf{N}}=u)-\sum_{j=1}^{N_k}\pi_k^t(x_{j,k}^{t,\mathbf{N}},\boldsymbol{\bar{\mu}}_t^{\mathbf{N}})(u)\right|\\
&\overset{(a)}{\leq}\left(\sum_{k\in[K]}\dfrac{1}{\sqrt{N_k}}\right)\sqrt{|\mathcal{U}|}
\end{align*}

Inequality (a) follows from Lemma \ref{simple_lemma}. This concludes the proposition.

\subsection{Proof of Proposition (b)}

Note that,
\begin{align*}
&\theta_k\bar{r}_k^{\mathrm{MF}}(\boldsymbol{\bar{\mu}}_t^{\mathbf{N}},\boldsymbol{\bar{\pi}}_t)\\
&=\left(\dfrac{N_k}{N_{\mathrm{pop}}}\right)\dfrac{1}{N_k}\sum_{j=1}^{N_k}\sum_{x\in\mathcal{X}}\sum_{u\in\mathcal{U}}\bar{r}_k(x,u,\boldsymbol{\bar{\mu}}_t^{\mathbf{N}},\bar{\nu}^{\mathrm{MF}}(\boldsymbol{\bar{\mu}}_t^{\mathbf{N}},\boldsymbol{\bar{\pi}}_{t}))\times\bar{\pi}_k^t(x,\boldsymbol{\bar{\mu}}_t^{\mathbf{N}})(u)\delta(x_{j,k}^{t,\mathbf{N}}=x)\\
&=\dfrac{1}{N_{\mathrm{pop}}}\sum_{j=1}^{N_k}\sum_{u\in\mathcal{U}}\bar{r}_k(x_{j,k}^{t,\mathbf{N}},u,\boldsymbol{\bar{\mu}}_t^{\mathbf{N}},\bar{\nu}^{\mathrm{MF}}(\boldsymbol{\bar{\mu}}_t^{\mathbf{N}},\boldsymbol{\bar{\pi}}_{t}))\times\bar{\pi}_k^t(x_{j,k}^{t,\mathbf{N}},\boldsymbol{\bar{\mu}}_t^{\mathbf{N}})(u)
\end{align*}

We can upper bound the $\mathrm{LHS}$ of $(\ref{eq_34_2})$ by $ J_1+J_2$ where $J_1$ is defined as follows.
\begin{align*}
J_1&\triangleq\dfrac{1}{N_{\mathrm{pop}}}\mathbb{E}\left|\sum_{k\in[K]}\sum_{j=1}^{N_k}
\bar{r}_k(x_{j,k}^{t,\mathbf{N}},u_{j,k}^{t,\mathbf{N}},\boldsymbol{\bar{\mu}}_t^{\mathbf{N}},\boldsymbol{\bar{\nu}}_t^{\mathbf{N}})-\bar{r}_k(x_{j,k}^{t,\mathbf{N}},u_{j,k}^{t,\mathbf{N}},\boldsymbol{\bar{\mu}}_t^{\mathbf{N}},\bar{\nu}^{\mathrm{MF}}(\boldsymbol{\bar{\mu}}_t^{\mathbf{N}},\boldsymbol{\bar{\pi}}_{t}))
\right|\\
&\leq \dfrac{1}{N_{\mathrm{pop}}}\sum_{k\in[K]}\sum_{j=1}^{N_k}\mathbb{E}\left|
\bar{r}_k(x_{j,k}^{t,\mathbf{N}},u_{j,k}^{t,\mathbf{N}},\boldsymbol{\bar{\mu}}_t^{\mathbf{N}},\boldsymbol{\bar{\nu}}_t^{\mathbf{N}})-\bar{r}_k(x_{j,k}^{t,\mathbf{N}},u_{j,k}^{t,\mathbf{N}},\boldsymbol{\bar{\mu}}_t^{\mathbf{N}},\bar{\nu}^{\mathrm{MF}}(\boldsymbol{\bar{\mu}}_t^{\mathbf{N}},\boldsymbol{\bar{\pi}}_{t}))
\right|\\
&\overset{(a)}{\leq}\dfrac{1}{N_{\mathrm{pop}}}\sum_{k\in[K]}\sum_{j=1}^{N_k}\bar{L}_R \mathbb{E}\left|\boldsymbol{\bar{\nu}}_t^{\mathbf{N}}-\bar{\nu}^{\mathrm{MF}}(\boldsymbol{\bar{\mu}}_t^{\mathbf{N}},\boldsymbol{\bar{\pi}}_t)\right|_{1}\\
&\overset{(b)}{\leq} \bar{L}_R\Big(\sum_{k\in[K]}\dfrac{1}{\sqrt{N_k}}\Big)\sqrt{|\mathcal{U}|}
\end{align*}

Inequality (a) follows from Assumption \ref{ass_2}(c) while inequality (b) follows from Proposition (a). The term $J_2$ is defined as
\begin{align*}
J_2&\triangleq\mathbb{E}\left|\dfrac{1}{N_{\mathrm{pop}}}\sum_{k\in[K]} \sum_{j=1}^{N_k}\sum_{u\in\mathcal{U}}\bar{r}_k(x_{j,k}^{t,\mathbf{N}},u,\boldsymbol{\bar{\mu}}_t^{\mathbf{N}},\bar{\nu}^{\mathrm{MF}}(\boldsymbol{\bar{\mu}}_t^{\mathbf{N}},\boldsymbol{\bar{\pi}}_{t}))\times\left[\delta(u_{j,k}^{t,\mathbf{N}}=u)-\bar{\pi}_k^t(x_{j,k}^{t,\mathbf{N}},\boldsymbol{\bar{\mu}}_t^{\mathbf{N}})(u)\right]\right|\\
&\leq\dfrac{1}{N_{\mathrm{pop}}}\sum_{u\in\mathcal{U}}\mathbb{E}\left|\sum_{k\in[K]} \sum_{j=1}^{N_k}\bar{r}_k(x_{j,k}^{t,\mathbf{N}},u,\boldsymbol{\bar{\mu}}_t^{\mathbf{N}},\bar{\nu}^{\mathrm{MF}}(\boldsymbol{\bar{\mu}}_t^{\mathbf{N}},\boldsymbol{\bar{\pi}}_{t}))\times\left[\delta(u_{j,k}^{t,\mathbf{N}}=u)-\bar{\pi}_k^t(x_{j,k}^{t,\mathbf{N}},\boldsymbol{\bar{\mu}}_t^{\mathbf{N}})(u)\right]\right| 
\end{align*}

Using similar argument as used in Lemma \ref{lemma_r}, we therefore get,
\begin{align*}
J_2&\leq \dfrac{\bar{M}_R}{\sqrt{N}}\sqrt{|\mathcal{U}|}\leq\bar{M}_R\left(\sum_{k\in[K]}\dfrac{1}{\sqrt{N_k}}\right)\sqrt{|\mathcal{U}|}
\end{align*}

This concludes the result.

\subsection{Proof of Proposition (c)}

Note that the $\mathrm{LHS}$ of $(\ref{eq35_2})$ can be upper bounded by the following quantity..
\begin{align*}
&\sum_{k\in[K]}\sum_{x\in\mathcal{X}}\mathbb{E}\left|\boldsymbol{\bar{\mu}}_{t+1}^{\mathbf{N}}(x,k)-\bar{P}_k^{\mathrm{MF}}\left(\boldsymbol{\bar{\mu}}_t^{\mathbf{N}},\boldsymbol{\bar{\pi}}_t\right)(x)\right|\\
&= \sum_{k\in[K]}  \sum_{x\in\mathcal{X}}\dfrac{1}{N_k}\mathbb{E}\Bigg|\sum_{j=1}^{N_k}\delta(x_{j,k}^{t+1,\mathbf{N}}=x)\\
&\hspace{2cm}-\sum_{j=1}^{N_k}\sum_{u\in\mathcal{U}}\bar{\pi}_k^t(x_{j,k}^{t,\mathbf{N}},\boldsymbol{\bar{\mu}}_t^{\mathbf{N}})(u) \bar{P}_k(x_{j,k}^{t,\mathbf{N}},u,\boldsymbol{\bar{\mu}}_t^{\mathbf{N}},\bar{\nu}^{\mathrm{MF}}(\boldsymbol{\bar{\mu}}_t^{\mathbf{N}},\boldsymbol{\bar{\pi}}_t))(x)\Bigg|\leq  J_1 + J_2 + J_3
\end{align*}

The first term, $J_1$ is defined as follows:
\begin{align*}
J_1&\triangleq \sum_{k\in[K]} \sum_{x\in\mathcal{X}}\dfrac{1}{N_k} \mathbb{E}\left|\sum_{j=1}^{N_k}\delta(x_{j,k}^{t+1,\mathbf{N}}=x)-\sum_{j=1}^{N_k}\bar{P}_k(x_{j,k}^{t,\mathbf{N}},u_{j,k}^{t,\mathbf{N}},\boldsymbol{\bar{\mu}}_t^{\mathbf{N}},\boldsymbol{\bar{\nu}}_t^{\mathbf{N}})(x)\right|
\end{align*}

Using similar argument as used in Lemma \ref{lemma_mu} to bound $J_1$, we get,
\begin{align*}
J_1&\leq \left(\sum_{k\in[K]}\dfrac{1}{\sqrt{N_k}}\right)\sqrt{|\mathcal{X}|}\leq \left(\sum_{k\in[K]}\dfrac{1}{\sqrt{N_k}}\right)\left[\sqrt{|\mathcal{X}|}+\sqrt{|\mathcal{U}|}\right]
\end{align*}

The second term, $J_2$ is defined as follows,
\begin{align*}
J_2&\triangleq\sum_{k\in[K]}\sum_{x\in\mathcal{X}}\dfrac{1}{N_k}\mathbb{E}\left|\sum_{j=1}^{N_k}\bar{P}_k(x_{j,k}^{t,\mathbf{N}},u_{j,k}^{t,\mathbf{N}},\boldsymbol{\bar{\mu}}_t^{\mathbf{N}},\boldsymbol{\bar{\nu}}_t^{\mathbf{N}})(x)-\sum_{j=1}^{N_k} \bar{P}_k(x_{j,k}^{t,\mathbf{N}},u_{j,k}^{t,\mathbf{N}},\boldsymbol{\bar{\mu}}_t^{\mathbf{N}},\bar{\nu}^{\mathrm{MF}}(\boldsymbol{\bar{\mu}}_t^{\mathbf{N}},\boldsymbol{\bar{\pi}}_t))(x)\right|\\
&\overset{(a)}{\leq} K\bar{L}_P\left|\boldsymbol{\bar{\nu}}_t^{\mathbf{N}}-\bar{\nu}^{\mathrm{MF}}(\boldsymbol{\bar{\mu}}_t^{\mathbf{N}},\boldsymbol{\bar{\pi}}_t)\right|_{1}\\
&\overset{(b)}{\leq} K\bar{L}_P\left(\sum_{k\in[K]}\dfrac{1}{\sqrt{N_k}}\right)\sqrt{|\mathcal{U}|}\leq K\bar{L}_P\left(\sum_{k\in[K]}\dfrac{1}{\sqrt{N_k}}\right)\left[\sqrt{|\mathcal{X}|}+\sqrt{|\mathcal{U}|}\right]
\end{align*}

Relation (a) is a result of Assumption \ref{ass_2}(d) and the inequality (b) follows from Proposition (a). Finally, 
\begin{align*}
J_3
&=\sum_{k\in[K]}\sum_{x\in\mathcal{X}}\mathbb{E}\Bigg|\dfrac{1}{N_k}\sum_{j=1}^{N_k}\Bigg[\bar{P}_k(x_{j,k}^{t,\mathbf{N}},u_{j,k}^{t,\mathbf{N}},\boldsymbol{\bar{\mu}}_t^{\mathbf{N}},\bar{\nu}^{\mathrm{MF}}(\boldsymbol{\bar{\mu}}_t^{\mathbf{N}},\boldsymbol{\bar{\pi}}_t))(x)\\
&\hspace{2cm}-\sum_{u\in\mathcal{U}}\pi_k^t(x_{j,k}^{t,\mathbf{N}},\boldsymbol{\bar{\mu}}_t^{\mathbf{N}})(u) \bar{P}_k(x_{j,k}^{t,\mathbf{N}},u,\boldsymbol{\bar{\mu}}_t^{\mathbf{N}},\bar{\nu}^{\mathrm{MF}}(\boldsymbol{\bar{\mu}}_t^{\mathbf{N}},\boldsymbol{\bar{\pi}}_t))(x)\Bigg]\Bigg|\\
%&\leq\sum_{k\in[K]}\sum_{x\in\mathcal{X}}\sum_{u\in\mathcal{U}}\mathbb{E}\left|\dfrac{1}{N_k}\sum_{j=1}^{N_k}\left[\delta(u_{j,k}^{t,\mathbf{N}}=u)-\bar{\pi}_k^t(x_{j,k}^{t,\mathbf{N}},\boldsymbol{\bar{\mu}}_t^{\mathbf{N}})(u)\right] \bar{P}_k(x_{j,k}^{t,\mathbf{N}},u,\boldsymbol{\bar{\mu}}_t^{\mathbf{N}},\bar{\nu}^{\mathrm{MF}}(\boldsymbol{\bar{\mu}}_t^{\mathbf{N}},\boldsymbol{\bar{\pi}}_t))(x)\right|\\
&\overset{(a)}{\leq}\left(\sum_{k\in[K]}\dfrac{1}{\sqrt{N_k}}\right)\sqrt{|\mathcal{X}|}\leq \left(\sum_{k\in[K]}\dfrac{1}{\sqrt{N_k}}\right)\left[\sqrt{|\mathcal{X}|}+\sqrt{|\mathcal{U}|}\right]
\end{align*}

Inequality (a) is a result of Lemma \ref{simple_lemma}. This concludes the result.

\section{Proof of Lemma \ref{lemma_continuity}}
\label{app_lemma_11}
\subsection{Proof of Proposition (a)}

The following chain of inequalities hold true.
\begin{align*}
\begin{split}
|\nu^{\mathrm{MF}}(\boldsymbol{\mu}&,\boldsymbol{\pi})[\mathcal{U}]-\nu^{\mathrm{MF}}(\boldsymbol{\mu}',\boldsymbol{\pi})[\mathcal{U}]|_{1}\\&= \left|\sum_{k\in[K]}\sum_{x\in\mathcal{X}}\boldsymbol{\mu}(x,k)\pi_k(x,\boldsymbol{\mu}[\mathcal{X}])-\sum_{k\in[K]}\sum_{x\in\mathcal{X}}\boldsymbol{\mu}'(x,k)\pi_k(x,\boldsymbol{\mu}'[\mathcal{X}])\right|_{1}\\
&\leq \sum_{k\in[K]}\left|\sum_{x\in\mathcal{X}}\boldsymbol{\mu}(x,k)\pi_k(x,\boldsymbol{\mu}[\mathcal{X}])-\sum_{x\in\mathcal{X}}\boldsymbol{\mu}'(x,k)\pi_k(x,\boldsymbol{\mu}'[\mathcal{X}])\right|_{1}=|\nu(\boldsymbol{\mu},\boldsymbol{\pi})-\nu(\boldsymbol{\mu}',\boldsymbol{\pi})|_{1}\\
&= \sum_{u\in\mathcal{U}}\sum_{k\in[K]}\left|\sum_{x\in\mathcal{X}}\boldsymbol{\mu}(x,k)\pi_k(x,\boldsymbol{\mu}[\mathcal{X}])(u)-\sum_{x\in\mathcal{X}}\boldsymbol{\mu}'(x,k)\pi_k(x,\boldsymbol{\mu}'[\mathcal{X}])(u)\right|\\
&\overset{}{\leq} \sum_{k\in[K]}\sum_{x\in\mathcal{X}}|\boldsymbol{\mu}(x,k)-\boldsymbol{\mu}'(x,k)|\sum_{u\in\mathcal{U}}\pi_k(x,\boldsymbol{\mu}[\mathcal{X}])(u)\\
&\hspace{2cm}+\sum_{k\in[K]}\sum_{x\in\mathcal{X}}\boldsymbol{\mu}'(x,k)\sum_{u\in\mathcal{U}}|\pi_k(x,\boldsymbol{\mu}[\mathcal{X}])(u)-\pi_k(x,\boldsymbol{\mu}'[\mathcal{X}])(u)|\\
&\overset{(a)}{\leq} \sum_{k\in[K]}\sum_{x\in\mathcal{X}}|\boldsymbol{\mu}(x,k)-\boldsymbol{\mu}'(x,k)|+L_Q|\boldsymbol{\mu}[\mathcal{X}]-\boldsymbol{\mu}'[\mathcal{X}]|_{1}\sum_{k\in[K]}\sum_{x\in\mathcal{X}}\boldsymbol{\mu}'(x,k)\\
&\overset{(b)}{=} |\boldsymbol{\mu}-\boldsymbol{\mu}'|_1 + L_Q|\boldsymbol{\mu}[\mathcal{X}]-\boldsymbol{\mu}'[\mathcal{X}]|_1
\end{split}
\end{align*}

Result (a) follows from Lipschitz continuity of $\pi_k^t$ and the fact that $\pi_k(x,\boldsymbol{\mu})$ is a probability distribution. Finally, inequality (b) uses  the fact that $\boldsymbol{\mu}'$ is a distribution. This concludes the result.

\subsection{Proof of Proposition (b)}

Note that,
\begin{align*}
&\sum_{k\in[K]}|	r_k^{\mathrm{MF}}(\boldsymbol{\mu},\boldsymbol{\pi})-	r_k^{\mathrm{MF}}(\boldsymbol{\mu}',\boldsymbol{\pi})|\\
&\leq  \sum_{k\in[K]}\sum_{x\in\mathcal{X}} \sum_{u\in\mathcal{U}} |{r}_k(x,u,\boldsymbol{\mu}[\mathcal{X}],\nu(\boldsymbol{\mu},\boldsymbol{\pi})[\mathcal{U}])-{r}_k(x,u,\boldsymbol{\mu}'[\mathcal{X}],\nu(\boldsymbol{\mu}',\boldsymbol{\pi})[\mathcal{U}])|\times\boldsymbol{\mu}(x,k)\pi_k(x,\boldsymbol{\mu}[\mathcal{X}])(u)\\
&+\sum_{k\in[K]} \sum_{x\in\mathcal{X}} \sum_{u\in\mathcal{U}} |{r}_k(x,u,\boldsymbol{\mu}'[\mathcal{X}],\nu(\boldsymbol{\mu}',\boldsymbol{\pi})[\mathcal{U}])|\times|\boldsymbol{\mu}(x,k)\pi_k(x,\boldsymbol{\mu}[\mathcal{X}])(u) - \boldsymbol{\mu}'(x,k)\pi_k(x,\boldsymbol{\mu}'[\mathcal{X}])(u)|
\end{align*}

Using the Lipschitz continuity of $r_k$, and the facts that $\boldsymbol{\mu}$, $\pi_k(x,\boldsymbol{\mu})$ are distributions, the first term can be upper bounded by the following expression,
\begin{align*} &L_R\left(|\boldsymbol{\mu}[\mathcal{X}]-\boldsymbol{\mu}'[\mathcal{X}]|_{1} + |\nu^{\mathrm{MF}}(\boldsymbol{\mu},\boldsymbol{\pi})[\mathcal{U}]-\nu^{\mathrm{MF}}(\boldsymbol{\mu}',\boldsymbol{\pi})[\mathcal{U}]|_{1}\right)
\\
&\leq L_R|\boldsymbol{\mu}-\boldsymbol{\mu}'|_{1}+  L_R(1+L_Q)|\boldsymbol{\mu}[\mathcal{X}]-\boldsymbol{\mu}'[\mathcal{X}]|_{1}
\end{align*}

Proposition (a) is used to derive the above inequality. Utilising similar logic as used in Proposition (a), we can upper bound the second term by the following quantity:
\begin{align*}
M_R|\boldsymbol{\mu}-\boldsymbol{\mu}'|_{1}+M_RL_Q|\boldsymbol{\mu}[\mathcal{X}]-\boldsymbol{\mu}'[\mathcal{X}]|_{1}
\end{align*}

\subsection{Proof of Proposition (c)}

The proof is similar to that of Proposition (b). Note that,
\begin{align*}
&|	P^{\mathrm{MF}}(\boldsymbol{\mu},\boldsymbol{\pi})-	P^{\mathrm{MF}}(\boldsymbol{\mu}',\boldsymbol{\pi})|_1=\sum_{k\in[K]}|	P_k^{\mathrm{MF}}(\boldsymbol{\mu},\boldsymbol{\pi})-	P_k^{\mathrm{MF}}(\boldsymbol{\mu}',\boldsymbol{\pi})|_1\\
&\leq  \sum_{k\in[K]}\sum_{x\in\mathcal{X}} \sum_{u\in\mathcal{U}} |{P}_k(x,u,\boldsymbol{\mu}[\mathcal{X}],\nu^{\mathrm{MF}}(\boldsymbol{\mu},\boldsymbol{\pi})[\mathcal{U}])-{P}_k(x,u,\boldsymbol{\mu}'[\mathcal{X}],\nu^{\mathrm{MF}}(\boldsymbol{\mu}',\boldsymbol{\pi})[\mathcal{U}])|_1\boldsymbol{\mu}(x,k)\pi_k(x,\boldsymbol{\mu}[\mathcal{X}])(u)\\
&+\sum_{k\in[K]} \sum_{x\in\mathcal{X}} \sum_{u\in\mathcal{U}} |{P}_k(x,u,\boldsymbol{\mu}'[\mathcal{X}],\nu^{\mathrm{MF}}(\boldsymbol{\mu}',\boldsymbol{\pi})[\mathcal{U}])|_1\times|\boldsymbol{\mu}(x,k)\pi_k(x,\boldsymbol{\mu}[\mathcal{X}])(u) - \boldsymbol{\mu}'(x,k)\pi_k(x,\boldsymbol{\mu}'[\mathcal{X}])(u)|
\end{align*}

Using the Lipschitz continuity of ${P}_k$ and the facts that $\boldsymbol{\mu}$, $\pi_k(x,\boldsymbol{\mu})$ are distributions, the first term can be upper bounded by the following expression,
\begin{align*} L_P\Big(|\boldsymbol{\mu}[\mathcal{X}]-\boldsymbol{\mu}'[\mathcal{X}]|_{1} &+ |\nu^{\mathrm{MF}}(\boldsymbol{\mu},\boldsymbol{\pi})[\mathcal{U}]-\nu^{\mathrm{MF}}(\boldsymbol{\mu}',\boldsymbol{\pi})[\mathcal{U}]|_{1}\Big)
\\
&\leq L_P|\boldsymbol{\mu}-\boldsymbol{\mu}'|_{1}+  L_P(1+L_Q)|\boldsymbol{\mu}[\mathcal{X}]-\boldsymbol{\mu}'[\mathcal{X}]|_{1}
\end{align*}

Proposition (a) is used to derive the above inequality. Utilising similar logic as used in Proposition (a), and the fact that $|{P}_k(x,u,\boldsymbol{\mu}'[\mathcal{X}],\nu^{\mathrm{MF}}(\boldsymbol{\mu}',\boldsymbol{\pi})[\mathcal{U}])|_1=1$, $\forall x\in\mathcal{X}, \forall u\in\mathcal{U}$, we can bound the second term by the following quantity:
\begin{align*}
|\boldsymbol{\mu}-\boldsymbol{\mu}'|_{1}+L_Q|\boldsymbol{\mu}[\mathcal{X}]-\boldsymbol{\mu}'[\mathcal{X}]|_{1}
\end{align*}
This concludes the result.

\section{Proof of Lemma \ref{lemma_approx}}
\label{app_lemma_12}

\subsection{Proof of Proposition (a)}

Using the definition of $L_{1}$-norm, we get:
\begin{align*}\mathbb{E}&\left|\boldsymbol{\nu}_t^{\mathbf{N}}[\mathcal{U}]-\nu^{\mathrm{MF}}(\boldsymbol{\mu}_t^{\mathbf{N}},\boldsymbol{\pi}_t)[\mathcal{U}]\right|_1 \\
&=  \sum_{u\in\mathcal{U}}\mathbb{E}\left|\sum_{k\in[K]}\boldsymbol{\nu}_t^{\mathbf{N}}(u,k)-\sum_{k\in[K]}\nu^{\mathrm{MF}}(\boldsymbol{\mu}_t^{\mathbf{N}},\boldsymbol{\pi}_t)(u,k)\right|\\
&=\dfrac{1}{N_{\mathrm{pop}}} \sum_{u\in\mathcal{U}}\mathbb{E}\left|\sum_{k\in[K]}\sum_{j=1}^{N_k}\delta(u_{j,k}^{t,\mathbf{N}}=u)-\sum_{k\in[K]}\sum_{j=1}^{N_k}\pi_k^t(x_{j,k}^{t,\mathbf{N}},\boldsymbol{\mu}_t^{\mathbf{N}})(u)\right|
\end{align*}

Using Lemma \ref{simple_lemma}, we conclude the proposition.

\subsection{Proof of Proposition (b)}

Using similar argument as used in Lemma \ref{lemma_r}, we can bound the $\mathrm{LHS}$ of $(\ref{eq_34})$  by $J_1+J_2$ where
\begin{align*}
J_1&\triangleq\dfrac{1}{N_{\mathrm{pop}}}\mathbb{E}\left|\sum_{k\in[K]}\sum_{j=1}^{N_k}
{r}_k(x_{j,k}^{t,\mathbf{N}},u_{j,k}^{t,\mathbf{N}},\boldsymbol{\mu}_t^{\mathbf{N}}[\mathcal{X}],\boldsymbol{\nu}_t^{\mathbf{N}}[\mathcal{U}])-{r}_k(x_{j,k}^{t,\mathbf{N}},u_{j,k}^{t,\mathbf{N}},\boldsymbol{\mu}_t^{\mathbf{N}}[\mathcal{X}],\nu^{\mathrm{MF}}(\boldsymbol{\mu}_t^{\mathbf{N}},\boldsymbol{\pi}_{t})[\mathcal{U}])
\right|\\
&\leq\dfrac{1}{N_{\mathrm{pop}}}\sum_{k\in[K]}\sum_{j=1}^{N_k}\mathbb{E}\left|
{r}_k(x_{j,k}^{t,\mathbf{N}},u_{j,k}^{t,\mathbf{N}},\boldsymbol{\mu}_t^{\mathbf{N}}[\mathcal{X}],\boldsymbol{\nu}_t^{\mathbf{N}}[\mathcal{U}])-{r}_k(x_{j,k}^{t,\mathbf{N}},u_{j,k}^{t,\mathbf{N}},\boldsymbol{\mu}_t^{\mathbf{N}}[\mathcal{X}],\nu^{\mathrm{MF}}(\boldsymbol{\mu}_t^{\mathbf{N}},\boldsymbol{\pi}_{t})[\mathcal{U}])
\right|\\
&\overset{(a)}{\leq}L_R \mathbb{E}\left|\boldsymbol{\nu}_t^{\mathbf{N}}[\mathcal{U}]-\nu^{\mathrm{MF}}(\boldsymbol{\mu}_t^{\mathbf{N}},\boldsymbol{\pi}_t)[\mathcal{U}]\right|_{1}\\
&\overset{(b)}{\leq} \dfrac{L_R}{\sqrt{N_{\mathrm{pop}}}}\sqrt{|\mathcal{U}|} 
\end{align*}

Inequality (a) follows from the Lipschitz continuity of ${r}_k$ while (b) is a consequence of proposition (a). The second term, $J_2$ is as follows,
\begin{align*}
J_2&\triangleq\dfrac{1}{N_{\mathrm{pop}}}\mathbb{E}\left|\sum_{k\in[K]} \sum_{j=1}^{N_k}\sum_{u\in\mathcal{U}}{r}_k(x_{j,k}^{t,\mathbf{N}},u,\boldsymbol{\mu}_t^{\mathbf{N}}[\mathcal{X}],\nu^{\mathrm{MF}}(\boldsymbol{\mu}_t^{\mathbf{N}},\boldsymbol{\pi}_{t})[\mathcal{U}])\right.\\
&\hspace{4cm}\times\left[\delta(u_{j,k}^{t,\mathbf{N}}=u)-\pi_k^t(x_{j,k}^{t,\mathbf{N}},\boldsymbol{\mu}_t^{\mathbf{N}}[\mathcal{X}])(u)\right]\Bigg|\\
&\overset{(a)}{\leq}\dfrac{M_R}{\sqrt{N_{\mathrm{pop}}}}\sqrt{|\mathcal{U}|}
\end{align*}

Inequality (a) can be proved using Lemma \ref{simple_lemma}. 

\subsection{Proof of Proposition (c)}

Note that the $\mathrm{LHS}$ of $(\ref{eq35})$ can be upper bounded as follows, 
\begin{align*}
\mathbb{E}&\left|\boldsymbol{\mu}_{t+1}^{\mathbf{N}}[\mathcal{X}]-P^{\mathrm{MF}}(\boldsymbol{\mu}_t^{\mathbf{N}},\boldsymbol{\pi}_t)[\mathcal{X}]\right|_{1}\\
&=\sum_{x\in\mathcal{X}}\mathbb{E}\left|\sum_{k\in[K]}\boldsymbol{\mu}_{t+1}^{\mathbf{N}}(x,k)-\sum_{k\in[K]}P_k^{\mathrm{MF}}\left(\boldsymbol{\mu}_t^{\mathbf{N}},\boldsymbol{\pi}_t\right)(x)\right|\\
&=  \dfrac{1}{N_{\mathrm{pop}}}\sum_{x\in\mathcal{X}}\mathbb{E}\Bigg|\sum_{k\in[K]} \sum_{j=1}^{N_k}\Bigg\lbrace\delta(x_{j,k}^{t+1,\mathbf{N}}=x)\\
&\hspace{3.5cm}- \sum_{u\in\mathcal{U}}\pi_k^t(x_{j,k}^{t,\mathbf{N}},\boldsymbol{\mu}_t^{\mathbf{N}}[\mathcal{X}])(u) {P}_k(x_{j,k}^{t,\mathbf{N}},u,\boldsymbol{\mu}_t^{\mathbf{N}}[\mathcal{X}],\nu^{\mathrm{MF}}(\boldsymbol{\mu}_t^{\mathbf{N}},\boldsymbol{\pi}_t)[\mathcal{U}])(x)\Bigg\rbrace\Bigg|\\
&\leq  J_1 + J_2 + J_3
\end{align*}

The first term is defined as:
\begin{align*}
J_1&\triangleq  \dfrac{1}{N_{\mathrm{pop}}}\sum_{x\in\mathcal{X}} \mathbb{E}\left|\sum_{k\in[K]}\sum_{j=1}^{N_k}\delta(x_{j,k}^{t+1,\mathbf{N}}=x)-\sum_{k\in[K]}\sum_{j=1}^{N_k}{P}_k(x_{j,k}^{t,\mathbf{N}},u_{j,k}^{t,\mathbf{N}},\boldsymbol{\mu}^{t,\mathbf{N}}[\mathcal{X}],\boldsymbol{\nu}^{t,\mathbf{N}}[\mathcal{U}])(x)\right|
\end{align*}

Applying  Lemma \ref{simple_lemma}, we can conclude that,
\begin{align*}
J_1&\leq \dfrac{1}{\sqrt{N_{\mathrm{pop}}}}\sqrt{|\mathcal{X}|}\leq \dfrac{1}{\sqrt{N_{\mathrm{pop}}}}\left[\sqrt{|\mathcal{X}|}+\sqrt{|\mathcal{U}|}\right]
\end{align*}

The second term, $J_2$ is as follows,
\begin{align*}
&J_2\triangleq\dfrac{1}{N_{\mathrm{pop}}}\sum_{x\in\mathcal{X}}\mathbb{E}\Bigg|\sum_{k\in[K]}\sum_{j=1}^{N_k}\Big\lbrace\tilde{P}_k(x_{j,k}^{t,\mathbf{N}},u_{j,k}^{t,\mathbf{N}},\boldsymbol{\mu}_t^{\mathbf{N}}[\mathcal{X}],\boldsymbol{\nu}_t^{\mathbf{N}}[\mathcal{U}])(x)\\
&\hspace{4.5cm}- \tilde{P}_k(x_{j,k}^{t,\mathbf{N}},u_{j,k}^{t,\mathbf{N}},\boldsymbol{\mu}_t^{\mathbf{N}}[\mathcal{X}],\nu^{\mathrm{MF}}(\boldsymbol{\mu}_t^{\mathbf{N}},\boldsymbol{\pi}_t)[\mathcal{U}])(x)\Big\rbrace\Bigg|\\
&\overset{}{\leq} \dfrac{1}{N_{\mathrm{pop}}}\sum_{k\in[K]}\sum_{j=1}^{N_k}\sum_{x\in\mathcal{X}}\mathbb{E}\Big|\tilde{P}_k(x_{j,k}^{t,\mathbf{N}},u_{j,k}^{t,\mathbf{N}},\boldsymbol{\mu}_t^{\mathbf{N}}[\mathcal{X}],\boldsymbol{\nu}_t^{\mathbf{N}}[\mathcal{U}])(x)\\
&\hspace{4cm}- \tilde{P}_k(x_{j,k}^{t,\mathbf{N}},u_{j,k}^{t,\mathbf{N}},\boldsymbol{\mu}_t^{\mathbf{N}}[\mathcal{X}],\nu^{\mathrm{MF}}(\boldsymbol{\mu}_t^{\mathbf{N}},\boldsymbol{\pi}_t)[\mathcal{U}])(x)\Big|\\
&\overset{(a)}{\leq} L_P\left|\boldsymbol{\nu}_t^{\mathbf{N}}[\mathcal{U}]-\nu^{\mathrm{MF}}(\boldsymbol{\mu}_t^{\mathbf{N}},\boldsymbol{\pi}_t)[\mathcal{U}]\right|_{1}\\
&\overset{(b)}{\leq} \dfrac{L_P}{\sqrt{N_{\mathrm{pop}}}}\sqrt{|\mathcal{U}|}\\
&\leq \dfrac{L_P}{\sqrt{N_{\mathrm{pop}}}}\left[\sqrt{|\mathcal{X}|}+\sqrt{|\mathcal{U}|}\right]
\end{align*}

Inequality (a) is due to Lipschitz continuity of $P_k$ and (b) follows from Proposition (a). Finally, 
\begin{align*}
&J_3=\sum_{x\in\mathcal{X}}\mathbb{E}\Bigg|\dfrac{1}{N_{\mathrm{pop}}}\sum_{k\in[K]}\sum_{j=1}^{N_k}\Bigg[P_k(x_{j,k}^{t,\mathbf{N}},u_{j,k}^{t,\mathbf{N}},\boldsymbol{\mu}_t^{\mathbf{N}}[\mathcal{X}],\nu(\boldsymbol{\mu}_t^{\mathbf{N}},\boldsymbol{\pi}_t)[\mathcal{U}])(x)-\\
&\hspace{4cm}-\sum_{u\in\mathcal{U}} \pi_k^t(x_{j,k}^{t,\mathbf{N}},\boldsymbol{\mu}_t^{\mathbf{N}}[\mathcal{X}])(u) P_k(x_{j,k}^{t,\mathbf{N}},u,\boldsymbol{\mu}_t^{\mathbf{N}}[\mathcal{X}],\nu(\boldsymbol{\mu}_t^{\mathbf{N}},\boldsymbol{\pi}_t)[\mathcal{U}])(x)\Bigg]\Bigg|%\\
%&\leq\dfrac{1}{N_{\mathrm{pop}}}\sum_{x\in\mathcal{X}}\sum_{u\in\mathcal{U}}\mathbb{E}\Bigg|\sum_{k\in[K]}\sum_{j=1}^{N_k}\left[\delta(u_{j,k}^{t,\mathbf{N}}=u)-\pi_k^t(x_{j,k}^{t,\mathbf{N}},\boldsymbol{\mu}_t^{\mathbf{N}}[\mathcal{X}])(u)\right]\\
%&\hspace{2cm}\times P_k(x_{j,k}^{t,\mathbf{N}},u,\boldsymbol{\mu}_t^{\mathbf{N}}[\mathcal{X}],\nu(\boldsymbol{\mu}_t^{\mathbf{N}},\boldsymbol{\pi}_t)[\mathcal{U}])(x)\Bigg|
\end{align*}

Applying Lemma \ref{simple_lemma}, we finally obtain,
$J_3\leq \frac{1}{\sqrt{N_{\mathrm{pop}}}}\sqrt{|\mathcal{X}|}\leq \frac{1}{\sqrt{N_{\mathrm{pop}}}}\left[\sqrt{|\mathcal{X}|}+\sqrt{|\mathcal{U}|}\right]$.

\section{Proof of Theorem \ref{corr_1}}

Note that the $\mathrm{LHS}$ of $(\ref{eq_thm4})$ can be upper bounded as,
\begin{align*}
\mathrm{LHS}\leq\left|\sup_{\Phi\in\mathbb{R}^{\mathrm{d}}}v^{\mathbf{N}}(\boldsymbol{\mu}_0,\pi_{\Phi})-v_{\mathrm{MF}}^*(\boldsymbol{\mu}_0)\right|+\left|v_{\mathrm{MF}}^*(\boldsymbol{\mu}_0)-\dfrac{1}{T}\sum_{j=1}^J v^{\mathrm{MF}}({\boldsymbol{\mu}_0},\boldsymbol{\pi}_{\Phi_j})\right| 
\end{align*}

Using Theorem $\ref{thm_1}$, the first term can be bounded by $C'\left[\sqrt{|\mathcal{X}|}+\sqrt{|\mathcal{U}|}\right]\frac{1}{N_{\mathrm{pop}}}\sum_{k\in[K]}\sqrt{N_k}$ for some constant $C'$. Using Lemma \ref{lemma_10}, the second term be bounded by $\sqrt{\epsilon_{\mathrm{bias}}}/(1-\gamma)+\epsilon$ with a sample complexity  $\mathcal{O}(\epsilon^{-3})$. Choosing $\epsilon=C'\left[\sqrt{|\mathcal{X}|}+\sqrt{|\mathcal{U}|}\right]\frac{1}{N_{\mathrm{pop}}}\sum_{k\in[K]}\sqrt{N_k}$, we obtain the result as in the statement of the theorem.

\section{Loose Bounds}
\label{app_loose_bounds}

In this section, we shall demonstrate that one can derive loose bounds for multi-agent systems satisfying Assumption \ref{ass_1}, \ref{assumptions} using Theorem \ref{thm_2}. Similarly, loose bounds for systems satisfying Assumption \ref{ass_2} and \ref{ass_4} can be derived using Theorem \ref{thm_1}.

\subsection{Loose Bound Using Theorem \ref{thm_1}}

Consider a multi-agent system satisfying Assumptions \ref{ass_2} and \ref{ass_4}. We shall use the notations of Theorem \ref{thm_2}. Let, $\boldsymbol{\theta}\triangleq\{\theta_k\}_{k\in[K]}$ be prior probabilities of different classes. If $\bar{r}_k$'s and $\bar{P}_k$'s are given reward and transition functions of the system, then one can define $r_k$'s and $P_k$'s such that, $\forall x\in\mathcal{X}$, $\forall u\in \mathcal{U}$, $\forall \bar{\boldsymbol{\mu}}\in \mathcal{P}^K(\mathcal{X})$, $\forall \bar{\boldsymbol{\nu}}\in \mathcal{P}^K(\mathcal{U})$ and $\forall k\in[K]$,
\begin{align*}
	&\bar{r}_k(x,u,\bar{\boldsymbol{\mu}},\bar{\boldsymbol{\nu}})=r_k(x,u,\boldsymbol{\mu},\boldsymbol{\nu}), \\
	&\bar{P}_k(x,u,\bar{\boldsymbol{\mu}},\bar{\boldsymbol{\nu}})=P_k(x,u,\boldsymbol{\mu},\boldsymbol{\nu})
\end{align*}
where $\boldsymbol{\mu}$, $\boldsymbol{\nu}$ are uniquely defined as, $\boldsymbol{\mu}\triangleq\{\theta_k\bar{\boldsymbol{\mu}}(.,k)\}_{k\in[K]}$ and $\boldsymbol{\nu}\triangleq\{\theta_k\bar{\boldsymbol{\nu}}(.,k)\}_{k\in[K]}$. Clearly, $\boldsymbol{\mu}\in \mathcal{P}_{\boldsymbol{\theta}}(\mathcal{X}\times [K])$ where $\mathcal{P}_{\boldsymbol{\theta}}(\mathcal{X}\times [K])$ is the collection of distributions over $\mathcal{X}\times [K]$ such that the marginal distribution over $[K]$ derived from each of its elements is $\boldsymbol{\theta}$. Similarly, $\boldsymbol{\nu}\in \mathcal{P}_{\boldsymbol{\theta}}(\mathcal{U}\times[K])$. Also, for every policy $\boldsymbol{\bar{\pi}}\triangleq\{(\bar{\pi}_k^t)_{k\in[K]}\}_{t\in\{0,1,\cdots\}}$, one can define $\boldsymbol{\pi}\triangleq\{(\pi_k^t)_{k\in [K]}\}_{t\in\{0,1,\cdots\}}$ such that, $\forall x\in \mathcal{X}$, $\forall \boldsymbol{\bar{\mu}}\in\mathcal{P}^K(\mathcal{X})$ and $\forall k\in [K]$, 
\begin{align*}
	\bar{\pi}_k^t(x,\bar{\boldsymbol{\mu}})=\pi_k^t(x,\boldsymbol{\mu})
\end{align*}

Note that, the following inequality holds $\forall \boldsymbol{\mu},\boldsymbol{\mu}'\in\mathcal{P}_{\boldsymbol{\theta}}(\mathcal{X}\times[K])$, $\forall \boldsymbol{\nu},\boldsymbol{\nu}'\in \mathcal{P}_{\boldsymbol{\theta}}(\mathcal{U}\times [K])$, $\forall x\in \mathcal{X}$,  $\forall u\in \mathcal{U}$, and $\forall k\in [K]$
\begin{align}
	\begin{split}
		|r_k(x,u,\boldsymbol{\mu},\boldsymbol{\nu})-r_k(x,u,\boldsymbol{\mu}',&\boldsymbol{\nu}')|=|\bar{r}_k(x,u,\bar{\boldsymbol{\mu}},\bar{\boldsymbol{\nu}})-\bar{r}_k(x,u,\bar{\boldsymbol{\mu}}',\bar{\boldsymbol{\nu}}')|\\
		&\leq \bar{L}_R\left[|\bar{\boldsymbol{\mu}}-\bar{\boldsymbol{\mu}}'|_1+|\bar{\boldsymbol{\nu}}-\bar{\boldsymbol{\nu}}'|_1\right]\\
		&=\bar{L}_R\sum_{k\in[K]}\theta_k^{-1}\left[|\boldsymbol{\mu}(.,k)-\boldsymbol{\mu}'(.,k)|_1+|\boldsymbol{\nu}(.,k)-\boldsymbol{\nu}'(.,k)|_1\right]\\
		&\leq\bar{L}_R\boldsymbol{\theta}^{-1}_M\left[|\boldsymbol{\mu}-\boldsymbol{\mu}'|_1+|\boldsymbol{\nu}-\boldsymbol{\nu}'|_1\right]
	\end{split}
\end{align}
where we have, $\boldsymbol{\theta}_M^{-1}\triangleq\max\{\theta_k^{-1}\}_{k\in[K]}$, $\bar{\boldsymbol{\mu}}\triangleq\{\theta_k^{-1}\boldsymbol{\mu}(.,k)\}_{k\in[K]}$, $\bar{\boldsymbol{\mu}}'\triangleq\{\theta_k^{-1}\boldsymbol{\mu}'(.,k)\}_{k\in[K]}$, $\bar{\boldsymbol{\nu}}\triangleq\{\theta_k^{-1}\boldsymbol{\nu}(.,k)\}_{k\in[K]}$, and $\bar{\boldsymbol{\nu}}'\triangleq\{\theta_k^{-1}\boldsymbol{\nu}'(.,k)\}_{k\in[K]}$. Similarly, $\forall \boldsymbol{\mu},\boldsymbol{\mu}'\in\mathcal{P}_{\boldsymbol{\theta}}(\mathcal{X}\times[K])$, $\forall \boldsymbol{\nu},\boldsymbol{\nu}'\in \mathcal{P}_{\boldsymbol{\theta}}(\mathcal{U}\times [K])$, $\forall x\in \mathcal{X}$,  $\forall u\in \mathcal{U}$,  $\forall k\in [K]$, $\forall t\in \{0,1,\cdots\}$,
\begin{align}
	|P_k(x,u,\boldsymbol{\mu},\boldsymbol{\nu})-P_k(x,u,\boldsymbol{\mu}',\boldsymbol{\nu}')|_1&\leq \bar{L}_P\boldsymbol{\theta}_M^{-1}\left[|\boldsymbol{\mu}-\boldsymbol{\mu}'|_1+|\boldsymbol{\nu}-\boldsymbol{\nu}'|_1\right]\\
	|\pi_k^t(x,\boldsymbol{\mu})-\pi_k^t(x,\boldsymbol{\mu}')|_1&\leq \bar{L}_Q\boldsymbol{\theta}_M^{-1}|\boldsymbol{\mu}-\boldsymbol{\mu}'|_1
\end{align}

Hence, the given system can equivalently be thought as a multi-agent system satisfying Assumptions \ref{ass_1} and \ref{assumptions} with parameters $\bar{M}_R$, $\bar{L}_R\boldsymbol{\theta}_M^{-1}$, $\bar{L}_P\boldsymbol{\theta}_M^{-1}$ and $\bar{L}_Q\boldsymbol{\theta}_M^{-1}$. Using Theorem \ref{thm_1}, the approximation error bound for this translated system can be expressed as follows.

\begin{theorem}
	\label{thm_1_app}
	Let $\mathbf{x}_0^{\mathbf{N}}$ be the initial states and $\boldsymbol{\bar{\mu}}_0\in\mathcal{P}^K(\mathcal{X})$ their corresponding distribution. If $\bar{v}^{\mathbf{N}}$ denotes the empirical value function and $\bar{v}^{\mathrm{MF}}$ is its mean-field limit, then for any policy, $\boldsymbol{\bar{\pi}}\in\bar{\Pi}$,  the following inequality holds 
	\begin{align}
		\label{eq_28_new_app}
		\begin{split}
			\Big|&\bar{v}^{\mathbf{N}}(\mathbf{x}_0^{\mathbf{N}},{\boldsymbol{\bar{\pi}}})-\bar{v}^{\mathrm{MF}}(\boldsymbol{\bar{\mu}}_0,\boldsymbol{\bar{\pi}})\Big|\leq
			\dfrac{\bar{C}_R(\boldsymbol{\theta})}{1-\gamma}\sqrt{|\mathcal{U}|} \dfrac{1}{N_{\mathrm{pop}}}\left(\sum_{k\in[K]}\sqrt{N_k}\right) \\&+ \bar{C}_P(\boldsymbol{\theta})\left(\dfrac{\bar{S}_R(\boldsymbol{\theta})}{\bar{S}_P(\boldsymbol{\theta})-1}\right)\left[\sqrt{|\mathcal{X}|}+\sqrt{|\mathcal{U}|}\right] \dfrac{1}{N_{\mathrm{pop}}}\left(\sum_{k\in[K]}\sqrt{N_k}\right)\times\left[\dfrac{1}{1-\gamma \bar{S}_P(\boldsymbol{\theta})}-\dfrac{1}{1-\gamma}\right]
		\end{split}
	\end{align}
	whenever $\gamma \bar{S}_P(\boldsymbol{\theta}) <1$ where the parameters are defined as follows,
	\begin{align*}
		&\bar{C}_R(\boldsymbol{\theta})\triangleq \bar{M}_R+\bar{L}_R\boldsymbol{\theta}_M^{-1}\\
		&\bar{C}_P(\boldsymbol{\theta})\triangleq 2+\bar{L}_P\boldsymbol{\theta}_M^{-1}\\
		&\bar{S}_R(\boldsymbol{\theta})\triangleq \bar{M}_R(1+\bar{L}_Q\boldsymbol{\theta}_M^{-1})+\bar{L}_R\boldsymbol{\theta}_M^{-1}(2+\bar{L}_Q\boldsymbol{\theta}_M^{-1})\\
		&\bar{S}_P(\boldsymbol{\theta})\triangleq (1+\bar{L}_Q\boldsymbol{\theta}_M^{-1})+\bar{L}_P\boldsymbol{\theta}_M^{-1}(2+\bar{L}_Q\boldsymbol{\theta}_M^{-1})
	\end{align*}
\end{theorem}

One can verify that the bound $(\ref{eq_28_new_app})$ is weaker than the bound provided by Theorem \ref{thm_2}.

\subsection{Loose Bound Using Theorem \ref{thm_2}}

Consider a multi-agent system satisfying Assumptions \ref{ass_1} and \ref{assumptions}. We shall use the notations of Theorem \ref{thm_1}. Let, $\boldsymbol{\theta}\triangleq\{\theta_k\}_{k\in[K]}$ be prior probabilities of different classes. If ${r}_k$'s and ${P}_k$'s are given reward and transition functions of the system, then one can define $\bar{r}_k$'s and $\bar{P}_k$'s such that, $\forall x\in\mathcal{X}$, $\forall u\in \mathcal{U}$, $\forall {\boldsymbol{\mu}}\in \mathcal{P}(\mathcal{X}\times[K])$, $\forall {\boldsymbol{\nu}}\in \mathcal{P}(\mathcal{U}\times [K])$ and $\forall k\in[K]$,
\begin{align*}
	&{r}_k(x,u,{\boldsymbol{\mu}},{\boldsymbol{\nu}})=\bar{r}_k(x,u,\boldsymbol{\bar{\mu}},\boldsymbol{\bar{\nu}}), \\
	&{P}_k(x,u,{\boldsymbol{\mu}},{\boldsymbol{\nu}})=\bar{P}_k(x,u,\boldsymbol{\bar{\mu}},\boldsymbol{\bar{\nu}})
\end{align*}
where $\boldsymbol{\bar{\mu}}$, $\boldsymbol{\bar{\nu}}$ are uniquely defined as, $\boldsymbol{\bar{\mu}}\triangleq\{\theta_k^{-1}{\boldsymbol{\mu}}(.,k)\}_{k\in[K]}$ and $\boldsymbol{\bar{\nu}}\triangleq\{\theta_k^{-1}{\boldsymbol{\nu}}(.,k)\}_{k\in[K]}$. Clearly, $\boldsymbol{\bar{\mu}}\in \mathcal{P}^K(\mathcal{X})$, $\boldsymbol{\bar{\nu}}\in \mathcal{P}^K(\mathcal{U})$. Also, for every policy $\boldsymbol{{\pi}}\triangleq\{({\pi}_k^t)_{k\in[K]}\}_{t\in\{0,1,\cdots\}}$, one can define $\boldsymbol{\bar{\pi}}\triangleq\{(\bar{\pi}_k^t)_{k\in [K]}\}_{t\in\{0,1,\cdots\}}$ such that, $\forall x\in \mathcal{X}$, $\forall k\in [K]$, and $\forall\boldsymbol{{\mu}}\in\mathcal{P}(\mathcal{X}\times [K])$, 
\begin{align*}
	{\pi}_k^t(x,{\boldsymbol{\mu}})=\bar{\pi}_k^t(x,\boldsymbol{\bar{\mu}})
\end{align*}

Note that, the following inequality holds $\forall \boldsymbol{\bar{\mu}},\boldsymbol{\bar{\mu}}'\in\mathcal{P}^K(\mathcal{X})$, $\forall \boldsymbol{\bar{\nu}},\boldsymbol{\bar{\nu}}'\in \mathcal{P}^K(\mathcal{U})$, $\forall x\in \mathcal{X}$,  $\forall u\in \mathcal{U}$, and $\forall k\in [K]$
\begin{align}
	\begin{split}
		|\bar{r}_k(x,u,\boldsymbol{\bar{\mu}},\boldsymbol{\bar{\nu}})-\bar{r}_k(x,u,\boldsymbol{\bar{\mu}}',\boldsymbol{\bar{\nu}}'&)|=|{r}_k(x,u,{\boldsymbol{\mu}},{\boldsymbol{\nu}})-{r}_k(x,u,{\boldsymbol{\mu}}',{\boldsymbol{\nu}}')|\\
		&\leq {L}_R\left[|{\boldsymbol{\mu}}-{\boldsymbol{\mu}}'|_1+|{\boldsymbol{\nu}}-{\boldsymbol{\nu}}'|_1\right]\\
		&={L}_R\sum_{k\in[K]}\theta_k\left[|\boldsymbol{\bar{\mu}}(.,k)-\boldsymbol{\bar{\mu}}'(.,k)|_1+|\boldsymbol{\bar{\nu}}(.,k)-\boldsymbol{\bar{\nu}}'(.,k)|_1\right]\\
		&\leq{L}_R\left[|\boldsymbol{\bar{\mu}}-\boldsymbol{\bar{\mu}}'|_1+|\boldsymbol{\bar{\nu}}-\boldsymbol{\bar{\nu}}'|_1\right]
	\end{split}
\end{align}
where  ${\boldsymbol{\mu}}\triangleq\{\theta_k\boldsymbol{\bar{\mu}}(.,k)\}_{k\in[K]}$, ${\boldsymbol{\mu}}'\triangleq\{\theta_k\boldsymbol{\bar{\mu}}'(.,k)\}_{k\in[K]}$, ${\boldsymbol{\nu}}\triangleq\{\theta_k\boldsymbol{\bar{\nu}}(.,k)\}_{k\in[K]}$,  ${\boldsymbol{\nu}}'\triangleq\{\theta_k\boldsymbol{\bar{\nu}}'(.,k)\}_{k\in[K]}$. Similarly, $\forall \boldsymbol{\bar{\mu}},\boldsymbol{\bar{\mu}}'\in\mathcal{P}^K(\mathcal{X})$, $\forall \boldsymbol{\bar{\nu}},\boldsymbol{\bar{\nu}}'\in \mathcal{P}^K(\mathcal{U})$, $\forall x\in \mathcal{X}$,  $\forall u\in \mathcal{U}$,  $\forall k\in [K]$, $\forall t\in \{0,1,\cdots\}$,
\begin{align}
	|\bar{P}_k(x,u,\boldsymbol{\bar{\mu}},\boldsymbol{\bar{\nu}})-\bar{P}_k(x,u,\boldsymbol{\bar{\mu}}',\boldsymbol{\bar{\nu}}')|_1&\leq {L}_P\left[|\boldsymbol{\bar{\mu}}-\boldsymbol{\bar{\mu}}'|_1+|\boldsymbol{\bar{\nu}}-\boldsymbol{\bar{\nu}}'|_1\right]\\
	|\bar{\pi}_k^t(x,\boldsymbol{\bar{\mu}})-\bar{\pi}_k^t(x,\boldsymbol{\bar{\mu}}')|_1&\leq {L}_Q|\boldsymbol{\bar{\mu}}-\boldsymbol{\bar{\mu}}'|_1
\end{align}

Hence, the given system can equivalently be thought as a multi-agent system satisfying Assumptions \ref{ass_2} and \ref{ass_4} with parameters ${M}_R$, ${L}_R$, ${L}_P$ and ${L}_Q$. Using Theorem \ref{thm_2}, the approximation error bound for this translated system can be expressed as follows.
\begin{theorem}
	\label{thm_2_app}
	If $\mathbf{x}_0^{\mathbf{N}}$ be initial states and $\boldsymbol{{\mu}}_0\in\mathcal{P}(\mathcal{X}\times [K])$ its resulting distribution, then $\forall \boldsymbol{{\pi}}\in {\Pi}$,
	\begin{align}
		\label{eq_within_thm2_app}
		\begin{split}
			\Big|{v}^{\mathbf{N}}&(\mathbf{x}_0^{\mathbf{N}},{\boldsymbol{{\pi}}})-{v}^{\mathrm{MF}}(\boldsymbol{{\mu}}_0,\boldsymbol{{\pi}})\Big|
			\leq
			\dfrac{{C}_R}{1-\gamma}\sqrt{|\mathcal{U}|} \left(\sum_{k\in[K]}\dfrac{1}{\sqrt{N_k}}\right) \\+ &{C}_P\left(\dfrac{{S}_R}{{S}_P-1}\right)\left[\sqrt{|\mathcal{X}|}+\sqrt{|\mathcal{U}|}\right] \left(\sum_{k\in[K]}\dfrac{1}{\sqrt{N_k}}\right)\times\left[\dfrac{1}{1-\gamma {S}_P}-\dfrac{1}{1-\gamma}\right]
		\end{split}
	\end{align}
	whenever $\gamma {S}_P<1$ where ${v}^{\mathbf{N}}(\cdot,\cdot)$ denotes the empirical value function and ${v}^{\mathrm{MF}}(\cdot,\cdot)$ is its mean-field limit. The other terms are given as follows: ${C}_R\triangleq {M}_R+{L}_R$, ${C}_P\triangleq 2+K{L}_P$, ${S}_R\triangleq {M}_R(1+{L}_Q)+{L}_R(2+K{L}_Q)$, and ${S}_P\triangleq (1+K{L}_Q)+K{L}_P(2+K{L}_Q)$.
\end{theorem}

Clearly, the bound provided by $(\ref{eq_within_thm2_app})$ is weaker than the bound suggested in Theorem \ref{thm_1}.

\bibliography{BibL.bib}

\end{document}